\documentclass[10pt]{extarticle}
\usepackage{amsmath,graphicx}
\usepackage{algorithm}
\usepackage{algpseudocode}
\usepackage{amssymb}
\usepackage{amsthm}
\usepackage{xcolor}
\usepackage{caption}
\usepackage{enumitem}
\usepackage{mathtools}
\usepackage{tikz-cd}
\usepackage{setspace,geometry,tablefootnote}
\usepackage{booktabs}
\newtheorem{theorem}{Theorem}

\newtheorem{corollary}{Corollary}

\usepackage{tikz}
\usetikzlibrary{arrows.meta, positioning}
\usepackage{tcolorbox}
\theoremstyle{definition}

\newtheorem{remark}[theorem]{Remark}

\newcommand{\R}{\mathbb{R}}

\algrenewcommand\algorithmicrequire{\textbf{In:}}
\algrenewcommand\algorithmicensure{\textbf{Out:}}
\algrenewcommand\algorithmicindent{0.8em}

\newcommand{\E}{\mathbb{E}}

\newcommand{\conv}{\mathrm{conv}}
\newcommand{\clconv}{\overline{\mathrm{conv}}}

\newcommand{\eps}{\varepsilon}
\DeclareMathOperator*{\argmin}{arg\,min}

    \newcommand{\dist}{\mathrm{dist}}
    \DeclareMathOperator{\supp}{supp}
\newcommand{\HK}{\mathcal{H}_K}
\newcommand{\Gen}{\mathcal{G}}
\newcommand{\epsinf}{\eps^{\star}_{\infty}}
\newcommand{\Sdual}{\mathcal{S}_{\HK}}
\newcommand{\Bq}[1]{\mathcal{B}_{#1}}
\newcommand{\clip}{\mathrm{cl}_{[-1,1]}}

\allowdisplaybreaks
\usepackage{tikz}
\usetikzlibrary{arrows.meta,positioning,backgrounds,fit,shadows.blur,calc}
\usepackage{hyperref}
\tcbuselibrary{theorems, skins, breakable}

\newtcbtheorem[]{assumptionBox}{Assumption}{
  enhanced, breakable, colback=black!1, colframe=pink!42!black,
  arc=2mm, boxrule=0.5pt,
  shadow={1.5pt}{-1.5pt}{0pt}{black!25!white},
  left=1em, right=1em, top=1.5ex, bottom=1ex,
  attach boxed title to top left={xshift=15pt, yshift=-6pt, yshifttext=0ex},
  boxed title style={colback=pink!40, colframe=pink!42!black, boxrule=0.5pt, arc=1mm,},
  fonttitle=\bfseries, coltitle=black,
  title style={interior style={top color=pink!40, bottom color=pink!40}}
}{ass}

\newtcbtheorem{problemBox}{}{
  colback=white, colframe=lightgray, fonttitle=\bfseries, coltitle=black,
  boxrule=0.8pt, width=12cm,
  enlarge left by=\dimexpr(\linewidth-12cm)/2\relax,
  enlarge right by=\dimexpr(\linewidth-12cm)/2\relax,
}{pb}

\newtcbtheorem{lemmaBox}{Lemma}{
  enhanced, breakable, colback=black!1, colframe=black!42!black,
  arc=2mm, boxrule=0.5pt,
  shadow={1.5pt}{-1.5pt}{0pt}{black!25!white},
  left=1em, right=1em, top=1.5ex, bottom=1ex,
  attach boxed title to top left={xshift=15pt, yshift=-6pt, yshifttext=0ex},
  boxed title style={colback=black!10, colframe=black!42!black, boxrule=0.5pt, arc=1mm,},
  fonttitle=\bfseries, coltitle=black,
  title style={interior style={top color=black!10, bottom color=black!10}}
}{lem}

\newtcbtheorem{theoremBox}{Theorem}{
  enhanced, breakable, colback=black!1, colframe=black!42!black,
  arc=2mm, boxrule=0.5pt,
  shadow={1.5pt}{-1.5pt}{0pt}{black!25!white},
  left=1em, right=1em, top=1.5ex, bottom=1ex,
  attach boxed title to top left={xshift=15pt, yshift=-6pt, yshifttext=0ex},
  boxed title style={colback=black!10, colframe=black!42!black, boxrule=0.5pt, arc=1mm,},
  fonttitle=\bfseries, coltitle=black,
  title style={interior style={top color=green!10!white, bottom color=green!10!white}}
}{thm}

\newtcbtheorem{propBox}{Proposition}{
  enhanced, breakable, colback=black!1, colframe=black!42!black,
  arc=2mm, boxrule=0.5pt,
  shadow={1.5pt}{-1.5pt}{0pt}{black!25!white},
  left=1em, right=1em, top=1.5ex, bottom=1ex,
  attach boxed title to top left={xshift=15pt, yshift=-6pt, yshifttext=0ex},
  boxed title style={colback=black!10, colframe=black!42!black, boxrule=0.5pt, arc=1mm,},
  fonttitle=\bfseries, coltitle=black,
  title style={interior style={top color=green!10!white, bottom color=green!10!white}}
}{prop}

\newtcbtheorem{defBox}{Definition}{
  enhanced, breakable, colback=black!1, colframe=black!42!black,
  arc=2mm, boxrule=0.5pt,
  shadow={1.5pt}{-1.5pt}{0pt}{black!25!white},
  left=1em, right=1em, top=1.5ex, bottom=1ex,
  attach boxed title to top left={xshift=15pt, yshift=-6pt, yshifttext=0ex},
  boxed title style={colback=black!10, colframe=black!42!black, boxrule=0.5pt, arc=1mm,},
  fonttitle=\bfseries, coltitle=black,
  title style={interior style={top color=green!10!white, bottom color=green!10!white}}
}{def}

\begin{document}
\title{Constrained Learning with Universally Learnable Concept Classes\vspace{-8pt}}
\author{
Herlock Rahimi\textsuperscript{1}\quad
Spyridon Pougkakiotis\textsuperscript{2}\quad
Dionysis Kalogerias\textsuperscript{1}
\\[0.5em]
\textsuperscript{1}Yale University
\quad
\textsuperscript{2}King's College London
}
\maketitle
\begin{abstract}
We study constrained statistical learning over infinite-dimensional hypothesis
classes in the fully nonconvex setting, and establish \emph{universal PACC
learnability} of the solutions of dual algorithms: Probably Approximately
Correct on Constraints, guaranteeing optimality and constraint satisfaction at
once. This strengthens near-PACC results, whose feasibility residual no amount
of data can remove. Optimality is caught between generalization, governed by
Rademacher complexity and favoring small classes, and strong Lagrangian duality,
which rests on Lyapunov convexity for vector measures and needs decomposability,
a demand pulling the other way. We reconcile the two by posing the population
problem over a universal RKHS $\HK$, dense in a decomposable envelope, and
learning over norm balls of growing radius. This yields the \emph{Tikhonov
complexity} $\mathfrak T^{\eps}_{n}$, the least RKHS norm reaching an
$\eps$-optimal Lagrangian level set; we prove it finite, obtain exact
learnability of the optimal value, and make the sample threshold explicit and
polynomial in $1/\eps$ under a source condition. Feasibility is harder: absent
convexity the Lagrangian may not attain its infimum, and dual information pins
down only an \emph{averaged} constraint-risk vector, not the risks of any
returned predictor. We introduce the \emph{closure--realization gap} $\epsinf$,
an index of how well $\HK$ retrieves feasible solutions from dualization; it is
a property of the problem, not of a modeling choice. Learnability is exact when
$\epsinf=0$, in particular under dual differentiability, and near-PACC with
residual exactly $\epsinf$ otherwise. Finally, no distribution-free threshold
exists already in the unconstrained specialization, so universality is the
canonical frame for dual algorithms over large hypothesis classes.
\end{abstract}
\section{Introduction}\label{sec:Intro}

Modern machine learning systems are increasingly deployed in settings where nominal
prediction performance alone is insufficient: the system must instead meet explicit
operational, societal, or safety specifications. In supervised learning problems such
as classification and regression, one minimizes a primary prediction risk while
simultaneously enforcing fairness, robustness, coverage, calibration, privacy, or
safety requirements, typically expressed through expectation- or risk-type functionals
evaluated under (possibly) different data distributions
\cite{Agarwal18Fair,Hardt16Equality,Cotter19Fair,Zafar17Fairness,
Hashimoto18Fairness,Donini18Fair,Mehrabi21Survey,Duchi18DRO,
EsfahaniKuhn18DRO,DelageYe10DRO,BenTalNemirovski98Robust,
Woodworth17Learning,WilliamsonMenon19Fairness,Narasimhan18Consistent,
Kearns18Preventing,CatonHaas20FairMLSurvey}.
The same template governs reinforcement learning and stochastic control, where one
optimizes a generally nonconvex performance objective subject to constraints on safety,
risk sensitivity, stability, or resource consumption
\cite{Tamar15CVaR,Chow18LyapunovRL,Prashanth16CvarRL,
ShapiroDentchevaRuszczynski14,RockafellarUryasev00,Altman99CMDP,
Achiam17CPO,GarciaFernandez15SafeRL}.

Two features of the contemporary setting make such problems hard in a way their
classical counterparts were not. First, the hypothesis classes involved are extremely
large and highly expressive, above all those induced by deep neural networks, which buy
remarkable approximation and representation power at the price of large-scale nonconvex
optimization
\cite{Hornik89Approximation,Cybenko89Approximation,Goodfellow16Deep,
Zhang21Understanding,AllenZhu19Convergence,Barron93Approximation,
Mhaskar96Approximation,Telgarsky16Expressive,PoggioMhaskar17Theory,
Yarotsky17Approximation,Choromanska15Loss}.
Second, and more fundamentally, the loss functionals through which the specifications
themselves are expressed --- error rates, coverage, recall, calibration, fairness
metrics --- are already nonconvex in the predictor; as we argue in
Section~\ref{subsec:related}, this nonconvexity is intrinsic to the specifications
rather than an artifact of any particular parameterization. Constrained statistical
learning over rich hypothesis spaces is therefore genuinely nonconvex, and no better
choice of model will convexify it.

Faced with such a problem, the most natural route to a solvable one is constraint
scalarization through \textbf{Lagrangian duality}
\cite{Rockafellar1970,EkelandTemam1999,BoydVandenberghe2004}: one attaches a nonnegative
multiplier to each constraint, forms the Lagrangian, and passes to the associated dual
problem, thereby trading hard constraints for principled linear penalties and exposing
a saddle-point structure that the primal--dual machinery of modern large-scale learning
is built to exploit
\cite{Cotter19Fair,NedicOzdaglar09Subgradient,ShapiroDentchevaRuszczynski14,
Chamon20PACC,BarthakurChamon2025Equality,Bertsekas99Nonlinear,NocedalWright06}.
Everything downstream, however, rests on a single prior question: does the dual problem
faithfully characterize the primal one? Concretely, one asks for a \textbf{zero duality
gap}, meaning that the primal and dual optimal \textbf{values} coincide, and, beyond
that, for the existence of optimal multipliers; together these constitute strong
duality, in the terminology of Section~\ref{subsec: dual ERM}. Neither property is
automatic once convexity is absent, and their failure is not a technicality: a
primal--dual procedure then optimizes a dual surrogate whose value differs from that of
the constrained problem it was meant to solve, and dual-optimal multipliers certify
neither feasibility nor near-optimality of any predictor. Establishing strong duality at
the population level is therefore the first and most fundamental requirement for any
Lagrangian-based approach to be well founded.

In the convex case that requirement is discharged by convexity of the risk functionals
together with a constraint qualification such as Slater's condition
\cite{Rockafellar1970,EkelandTemam1999,BoydVandenberghe2004,Zalinescu02Convex}. Without
convexity a fundamentally different mechanism is required, and the one we exploit, as do
the works closest to ours, is Lyapunov-type convexity for vector measures
\cite{Lyapunov1940,Uhl69Range,Knowles1975Lyapunov,KhanSagara2014Lyapunov,Aumann1965}.
Its content is that when the hypothesis class is decomposable and the feature marginals
are nonatomic, the set of attainable risk vectors is convex even though the individual
functionals are not, so that the classical separating-hyperplane argument again delivers
strong duality \cite{GiannessiImageSpace,KalogeriasPougkakiotis22}. We develop this
mechanism in Section~\ref{subsec: the tension}; convexification-through-richness
arguments of the same type also underlie the duality analyses of constrained learning
\cite{Chamon2020EmpiricalDuality,Chamon_CL2023} and the zero-duality-gap result for
constrained reinforcement learning \cite{Paternain19ZeroDuality}.

Strong duality, however, equates two \emph{population} values, and so says nothing yet
about a learner that never sees the population. In a strictly data-driven setting the
distributions are unknown and are replaced by empirical measures built from independent
datasets, one per distribution. The learner studied in this paper, Dual Empirical Risk
Minimization (DERM, Section~\ref{subsec: dual ERM}), first maximizes the empirical dual
function over the nonnegative multipliers and then minimizes the empirical Lagrangian,
at the resulting multipliers, over a sample-size-dependent hypothesis class. The
question then becomes statistical: at which sample sizes are DERM outputs approximately
\textbf{optimal} for the population problem?

In unconstrained learning, such guarantees are classically formulated through PAC
learnability and uniform convergence
\cite{Valiant84PAC,VapnikChervonenkis1971,BlumerEhrenfeuchtHausslerWarmuth1989,
Vapnik98,BartlettMendelson02,ShalevShwartzBenDavid2014,MohriRostamizadehTalwalkar18}.
The constrained setting introduces a structural complication with no unconstrained
analogue: the effective statistical difficulty of the dual problem is governed by
population-level quantities. Most prominently, the Slater margin controls the size of
the dual feasible region and thereby enters every downstream estimate, yet it is a
property of the underlying distributions, not of the losses or the hypothesis class
alone. The appropriate notion for constrained dual learning is therefore not classical
distribution-free PAC learnability \cite{Vapnik98}, but the universal learnability
framework of \cite{bousquet2020theoryuniversallearning}: consistency for every fixed but
arbitrary tuple of distributions, with sample thresholds allowed to depend on those
distributions. We formalize this as agnostic universal PACC learnability
(Definition~\ref{def:nonuniform-pacc}), \emph{prove} in
Proposition~\ref{prop:no-distribution-free} that no distribution-free threshold exists
already in the unconstrained specialization of our setting(due to the high expressivity of the RKHS), and argue in
Section~\ref{sec:universal learnable classes under constraints} that existing PAC-type
results for constrained learning
\cite{Chamon20PACC,Chamon20csl,BarthakurChamon2025Equality} are necessarily of this
universal nature as well, even when stated in a uniform sense, since their Slater-type
constraint qualifications implicitly restrict the class of admissible distributions.

With both requirements now in view, their incompatibility becomes apparent.
Population-level strong duality (\textbf{R1}) and statistical generalizability of
empirical dual-domain predictors (\textbf{R2}; see Section~\ref{subsec: the tension})
pull the choice of hypothesis class in opposite directions. Requirement \textbf{R1}
demands largeness in the precise form of decomposability, closure under measurable
splicing, which forces the class to contain functions of arbitrary complexity: the space
of all measurable functions, the Lebesgue spaces, and Orlicz spaces qualify, whereas
finite-dimensional parametric families, smooth function classes, and, most importantly
here, reproducing kernel Hilbert spaces and their norm balls do not. Requirement
\textbf{R2} demands the opposite: a controlled, decaying statistical complexity, such as
vanishing Rademacher complexity
\cite{BartlettMendelson02,LedouxTalagrand91,MohriRostamizadehTalwalkar18}, which holds
for norm-constrained kernel balls and finite-complexity classes but fails for
decomposable spaces. Prior work fixes a single finite-complexity class for both roles
\cite{Chamon20csl,Chamon20PACC,Chamon_CL2023,BarthakurChamon2025Equality,Boero25ALCoLe},
and this compromise manifests as irreducible approximation constants in the resulting
generalization bounds (see, e.g., \cite[Thm.~1]{Chamon_CL2023} and
\cite[Thm.~3.1]{BarthakurChamon2025Equality}): the guarantees are only near-PACC, and no
amount of data drives the residual to zero. Resolving this structural tension is the
central objective of the present work.

\medskip
\noindent\textbf{Our approach: variable-complexity dense kernel classes.}
The reconciliation rests on a simple observation: a hypothesis class need not be
decomposable in order to inherit the duality theory of a decomposable space, it need only
be \emph{dense} in one\cite{Aronszajn1950,Steinwart2001,SteinwartChristmann2008}. We therefore pose the population problem over a universal
reproducing kernel Hilbert space ,
which is dense in a Lebesgue-space envelope and inherits strong duality from that
envelope by density alone \cite{KalogeriasPougkakiotis22}, yet is far better behaved
statistically than the envelope itself. Concretely, we use the established result in \cite{KalogeriasPougkakiotis22} that asserts zero duality gap and
dual attainment for the kernel problem over the dense RKHS~\eqref{eq:primal formulation RKHS} and its
decomposable envelope~\eqref{eq:primal formulation envelope}, together with the pointwise
identity of their dual functions, so that the two dual optimal sets coincide
(Theorem~\ref{thm:dense-transfer}); this extends Lyapunov-based nonconvex duality from decomposable spaces to nondecomposable but dense
concept classes, and dispenses with the parametric convex-hull approximation parameters
of \cite{Chamon20PACC,Chamon_CL2023}. The picture is completed at the empirical stage by
restricting the learner to closed kernel norm balls, whose statistical complexity is
classical and explicitly controlled
\cite{BartlettMendelson02,MohriRostamizadehTalwalkar18}, and by letting the ball radius
grow with the sample size: fast enough to exhaust the kernel space in the limit, slowly
enough that the statistical complexity still vanishes. The device is reminiscent of
structural risk minimization \cite{Vapnik98,ShaweTaylor98SRM}, but is deployed here in
service of duality rather than model selection.

\medskip
\noindent\textbf{From dual values to feasible predictors.}
Closing the duality gap settles the optimal value, but the value is not what a downstream
user of the framework actually needs. What such a user needs is an explicit predictor
that is simultaneously \textbf{feasible} for the population constraints and near-optimal.
In the convex, attained regime the two are the same question: an optimal multiplier comes
with a Lagrangian minimizer whose constraint risks are precisely a dual subgradient, and
complementary slackness converts dual optimality into feasibility of that minimizer.
Without convexity the argument breaks twice over. The Lagrangian may fail to attain its
infimum, so there may be no predictor to certify in the first place; and dual
information, by the classical representations of subdifferentials of infimal and integral
functionals \cite{Aumann1965,CastaingValadier1977,IoffeTihomirov1979}, pins down only an
\emph{averaged} constraint-risk vector, blended across several distinct near-minimizers,
rather than the risk vector of any one predictor the learner could return. This is the
same structural obstruction that leads
\cite{Cotter19TwoPlayer,Agarwal18Fair} to return \emph{randomized} predictors; we discuss
the relationship in Section~\ref{subsec:related}. Feasibility must therefore be
established by a variational argument conducted on the dual side, and doing so occupies
the second half of this paper.

What makes such an argument worth conducting is that the entire residual it leaves behind
is governed by a single population-level geometric constant, introduced in
Section~\ref{subsec:feasibility-geometry} and called the \emph{closure--realization gap}: how well dually-retrieved solutions are feasible. To be more specific it is
the distance between the superdifferential of the dual function at an optimal multiplier
and the constraint risks actually realized by optimal primal solutions, both taken at the
\textbf{population} level. The gap is a property of the problem rather than of any
algorithm; it vanishes under benign dual geometry, and otherwise measures exactly how far
the kernel space is from permitting the recovery of a feasible certificate from dual
information alone. Our feasibility analysis of DERM shows that this quantity, and only
this quantity, controls the feasibility of what DERM returns. This runs in the opposite
direction to the near-PACC guarantees of \cite{Chamon20PACC}, where the residual is
produced by the algorithm and its parameterization, may be arbitrarily large, and is not
controlled by any intrinsic feature of the problem. The additional price we pay is the
Empirical Recovery Condition (Section~\ref{subsec:recovery}), an assumption on the
asymptotic behavior of the empirical primal solutions --- whose interaction with the gap
we examine critically in Remark~\ref{rem:scope} and Appendices.

One gap still separates such an analysis from a learning guarantee. Variational arguments
of this kind are asymptotic by nature and deliver their conclusions along subsequences,
whereas Definition~\ref{def:nonuniform-pacc} demands a sample size beyond which
\emph{every} output of the learner is feasible and near-optimal with high probability. We
therefore upgrade the subsequential conclusion to the full sequence of sample sizes, and
then convert the resulting threshold, which is initially realization-dependent, into a
single threshold valid with high probability, by an elementary continuity-from-below
argument applied to the tail-supremum envelope (equivalently, by Egorov's theorem
\cite{Egorov1911}).

\medskip\noindent The main contributions of the paper are as follows.
\begin{enumerate}[itemsep=3pt]
\item \textbf{Universal (near-)PACC learnability, with exactness characterized and both
regimes inhabited.} Theorem~\ref{thm:informal-section3} gives exact universal learnability of
the optimal value together with
full-sequence feasibility and optimality within the worst-case closure--realization gap over
the population dual optimal set, sharpened under uniqueness of the dual optimum; part~(vi)
gives exact universal PACC learnability of the kernel space whenever that gap vanishes
(e.g.\ under dual differentiability). Unlike the near-PACC guarantees of
\cite[Thm.~1]{Chamon_CL2023} and \cite[Thm.~3.1]{BarthakurChamon2025Equality}, whose
residuals stem from fixed parametric approximation and cannot vanish without extra analysis, our residual is a
single population-geometric constant that provably equals zero in benign geometry
(Lemma~\ref{lem:eps-star-properties}). 

\item \textbf{Feasibility certificates without primal attainment, in the fully nonconvex
regime.} We prove dual-side existence of sign- and complementarity-compatible subgradients
(Lemma~\ref{lem:sign-comp-subgrad}) and establish asymptotic subgradient tracking,
feasibility, and near-optimality of DERM outputs up to the closure--realization gap
(Lemma~\ref{lem:multi-qual-conv},
Corollaries~\ref{cor:multi-qual-feas}--\ref{cor:multi-primal-guarantee}). No convexity of
the losses, the constraints, or the hypothesis class is assumed; the only regularity
required is the argmin-recovery condition
(Assumption~\ref{ass:multi-argmin-recovery}), which we isolate and give checkable
sufficient conditions for (Proposition~\ref{prop:sharpness}). We also record, in
Proposition~\ref{prop:attainment}, exactly what that condition costs: in the exact regime
it \emph{implies} attainment of the constrained primal, so the ``no primal attainment''
claim is a property of the value arc alone.

\item \textbf{A fully explicit dual-side statistical theory, and an explicit rate under a
source condition.} Empirical strict feasibility is \emph{derived} from population Slater
beyond an explicit closed-form sample threshold rather than assumed
(Lemma~\ref{lem:slater-empirical}, sharpened to a pathwise statement in
Remark~\ref{rem:slater-pathwise}); population and empirical dual optima exist and are
confined uniformly to an explicit compact multiplier set
(Lemma~\ref{lem:bounded-multipliers}); the empirical dual deviates from its population
counterpart at an explicit closed-form rate obtained from first-principles Rademacher and
contraction estimates (Lemmas~\ref{lem:dual-pac-fixedQ} and~\ref{lem:multi-hp-conv}); and
Corollary~\ref{cor:explicit-rate} converts the Tikhonov complexity into an explicit
polynomial sample threshold under a standard source condition on the regularization path,
identifying our $\mathfrak T^\eps_n$ with the classical approximation-error function of
RKHS learning theory \cite{SmaleZhou07,CaponnettoDeVito07,SteinwartChristmann2008}.

\item \textbf{Necessity of the universal framework, proved.}
We argue, through the Slater-analysis, of previous works that the universal learning framework is the canonical one for dual algorithms due to the inherent dependency of dual geometry to the distributions. Moreover, in Proposition~\ref{prop:no-distribution-free} we prove that in the unconstrained
specialization of our setting no distribution-free sample threshold exists at fixed
$(\eps,\delta)$ due to the high expressivity of the RKHS. Universal PAC learnability
\cite{bousquet2020theoryuniversallearning}, thus is the canonical, not merely
convenient, framework here, and prior ``uniform'' statements in this literature are
implicitly universal.
\end{enumerate}

 \medskip\noindent\textbf{Tools developed along the way.}
The analysis we perform in this problem is quite new. Hence we borrow form and develop new notions and tools to do it. First,
\emph{(i) Tikhonov complexity} (Definition~\ref{def:tik}) that measures
quantifying the kernel norm required to reach near-optimal population Lagrangian level
sets at the multipliers DERM selects. 
\emph{(ii) Image-set deviation and closure--realization gap}
(Definitions~\ref{def:image-set-deviation} and~\ref{def:multi-eps-star}), built on the
Levin--Valadier representation to analyze feasibility.
\emph{(iii) Subsequential-to-full-sequence upgrade and tail uniformization} that converts subsequential results into eventual full-sequence
bounds (Lemma~\ref{lem:full-seq-feas}, Corollary~\ref{cor:full-seq-conv}), and a monotone
tail argument converts realization-dependent eventual thresholds into a single
high-probability finite-sample threshold (Lemma~\ref{lem:egorov-unif},
Corollary~\ref{cor:uniform-threshold}).

\subsection{Related work and scope}\label{subsec:related}

\paragraph{Randomized predictors and the averaging obstruction.}
The observation that dual information in a nonconvex constrained learning problem
certifies only an \emph{average} of constraint risks, rather than the risk vector of any
single returned predictor, is not new to this paper; it is the reason the two-player and
reductions approaches of \cite{Cotter19TwoPlayer,Agarwal18Fair,Cotter19Fair} return
\emph{randomized} (mixed) classifiers rather than deterministic ones. In that line of work
the mixture is carried by an explicit distribution over finitely many pure predictors,
produced by the no-regret dynamics, and the guarantee is stated for the mixture. The
relationship to the present paper is direct and worth stating: their finite mixtures play
the role that Lyapunov splicing plays here. Over the decomposable envelope $\Gen$ every
supergradient of the dual is realized by an explicit \emph{splice} of Lagrangian
minimizers (Lemma~\ref{lem:eps-star-properties}(iii)), which is a
measurable-partition analogue of a finite mixture and requires no randomization. What
$\epsinf$ measures is exactly the failure of the smaller, nondecomposable class $\HK$ to
realize \emph{within itself} what the envelope realizes by splicing. Our contribution
relative to \cite{Cotter19TwoPlayer,Agarwal18Fair} is therefore not the identification of
the obstruction but its \emph{quantification} by a single population constant, together
with the observation that the constant vanishes under dual differentiability, in which
case a deterministic predictor suffices.

\paragraph{Tikhonov complexity and the approximation-error function.}
The quantity $R_\eps(\mu) = \inf\{\|\phi\|_{\HK} : L(\phi,\mu) \le d(\mu)+\eps\}$
introduced in Section~\ref{subsec:tikhonov} is the inverse function of the classical
\emph{approximation error} (or regularization-path) functional
$\mathcal A(\lambda,\mu) = \inf_{\phi \in \HK}\{L(\phi,\mu)-d(\mu)+\lambda\|\phi\|^2_{\HK}\}$
that governs the bias term in kernel learning
\cite{SmaleZhou07,CaponnettoDeVito07,SteinwartChristmann2008}. We make this connection
explicit in Section~\ref{subsec:explicit-rates}, and use it to import the standard source
condition and thereby obtain an explicit polynomial threshold
(Corollary~\ref{cor:explicit-rate}). We are not aware of a prior use of this object as a
device for controlling a \emph{duality} gap rather than an estimation--approximation
tradeoff.

\paragraph{Augmented Lagrangian methods.}\label{par:alm}
Augmented Lagrangian (AL) methods
\cite{Hestenes69Multiplier,Powell69Multiplier,Rockafellar74,Rockafellar76Augmented}
restore strong duality in nonconvex programs by augmenting the Lagrangian with a penalty
on constraint violations, and have recently been imported into constrained learning by
\cite{Boero25ALCoLe}, which establishes learnability under a set of constraint
qualifications (\cite[Assumptions~2.5--2.7]{Boero25ALCoLe}). These qualifications relax
nonatomicity-type conditions at the expense of more restrictive second-order-type
requirements, and the resulting guarantees remain distribution-dependent, hence universal
rather than uniform in character. A further practical drawback is the need to tune the
penalty parameter alongside the multipliers, a procedure prone to ill-conditioning. Our
population-level strong duality results offer a complementary perspective: over dense
RKHS classes, the \emph{standard} Lagrangian already exhibits a zero duality gap, so any
dual-based scheme is sound at the population level without penalty augmentation, and our
feasibility analysis produces explicit primal certificates for DERM outputs without
second-order constraint qualifications. 

\paragraph{On the issue of convexification.}
A natural question is whether the nonconvex analysis developed herein is necessary, or
whether one could instead convexify the problem and reduce it to a setting where classical
duality theory applies. In the convex setting, Lagrangian methods admit a precise
justification: under standard regularity conditions, strong duality holds, the primal and
dual optimal values coincide, and saddle points recover optimal solutions together with
certificates of constraint satisfaction
\cite{Rockafellar1970,BoydVandenberghe2004,ShapiroDentchevaRuszczynski14,Zalinescu02Convex}.
However, the constrained learning setting of \eqref{eq:primal formulation f} is
fundamentally more delicate, for two compounding reasons. First, calibration theory
justifies convex surrogate losses (logistic, hinge, cross-entropy) only with respect to
the primary prediction risk \cite{BartlettJordanMcAuliffe06,TewariBartlett07}; it does not
extend to general constraint functionals. Constraints expressed through recall, coverage,
false positive rates, or fairness metrics such as equalized odds depend explicitly on the
geometry of the induced decision region, and there is in general no analogue of
calibration theory guaranteeing that a convex surrogate relaxation preserves feasibility
of these original constraints
\cite{Agarwal18Fair,Cotter19Fair,Narasimhan18Consistent}. Solving a convexified version of
the problem may therefore produce solutions that are infeasible for, or overly
conservative relative to, the intended constraint geometry. Second, and more critically
for our setting, the objective and constraints are evaluated under \emph{different} data
distributions. Convex surrogates primarily control global margin behavior under a single
reference measure, whereas feasibility here depends on fine geometric properties of the
induced decision regions under multiple interacting distributions simultaneously, a regime
closely related to dataset shift and distributional robustness
\cite{BenDavid10Theory,QuioneroCandela09DatasetShift}. These two difficulties together
imply that nonconvexity in \eqref{eq:primal formulation f} is not an artifact of neural
network parameterizations, but is intrinsic to the representation of practically relevant
constraint functionals under heterogeneous distributions. It is precisely this intrinsic
nonconvexity that necessitates the feasibility analysis of
Sections~\ref{sec:feasibility-indices}--\ref{subsec:main-results}: in the convex, attained
case, feasibility of Lagrangian minimizers is immediate, and no Levin--Valadier-type
representation or tail uniformization is required at all.

\paragraph{Outline.}
Section~\ref{sec:problem-setup} formulates the constrained statistical learning problem
and the DERM learner (Section~\ref{subsec: dual ERM}), introduces agnostic universal PACC
learnability (Definition~\ref{def:nonuniform-pacc}) and proves that the universal
qualifier is necessary (Proposition~\ref{prop:no-distribution-free}), makes the
duality--generalizability tension precise (Section~\ref{subsec: the tension}), and
describes its resolution through dense RKHSs and norm-ball restrictions
(Section~\ref{subsec: reconcile tension}). Section~\ref{sec:pac-dual-theorem} instantiates
the RKHS primal, dual, and DERM problems, states the standing assumptions
(Assumptions~\ref{ass:non-atomic}--\ref{ass:multi-argmin-recovery}), introduces the
Tikhonov complexity and the feasibility indices, states the main theorem
(Theorem~\ref{thm:informal-section3}), and gives the explicit-rate corollary
(Section~\ref{subsec:explicit-rates}). Section~\ref{sec:proofs} carries out the full
technical program. Appendix~\ref{sec:on-last-assumption} gives checkable sufficient
conditions for the recovery condition, together with a critical discussion
(Remark~\ref{rem:scope}) of the resulting scope of the two regimes of
Theorem~\ref{thm:informal-section3}; Appendix~\ref{sec:worked-example} works one instance
end to end.
\section{Problem Setting and the Duality-Generalizability Tension}
\label{sec:problem-setup}

We begin by formulating the class of constrained statistical learning problems that will
serve as the backbone of the paper. The purpose of this section is threefold. First, we
introduce the population and empirical constrained learning problems of interest, the
latter corresponding to our choice of a data-driven learning rule
(Section~\ref{subsec: dual ERM}). Second, we isolate two key structural requirements that
a reasonable learning theory for constrained problems should address simultaneously,
namely \textit{population-level strong duality} and \textit{concept generalizability from
finite samples}, leading to the notion of \textit{(agnostic) universal learnability} of
concept classes \textit{under constraints}
(Section~\ref{sec:universal learnable classes under constraints}). Third, we explain why
these two requirements pull the choice of a suitable hypothesis class in opposite and
conflicting directions (Section~\ref{subsec: the tension}), thereby motivating the use of
norm-ball-restricted universal RKHSs as concept classes within a canonical,
variable-complexity framework (Section~\ref{subsec: reconcile tension}).

To this end, let \(\mathcal X \subseteq \mathbb R^d\) denote an input (feature) space and
let \(\mathcal Y\) be either a finite set (classification) or a subset of \(\mathbb R\)
(regression). For each \(i\in\mathbb{N}_m\), let \(\mathcal D_i\) be a probability measure
on \(\mathcal X\times\mathcal Y\), and let \(\ell_i:\mathbb R\times\mathcal Y\to\mathbb R\)
be measurable losses. For a measurable predictor \(\phi:\mathcal X\to\mathbb R\) in an
appropriately chosen concept (hypothesis) class $\mathcal{F}$, define population
functionals
\[
\mathcal R_i(\phi)\triangleq\mathbb E_{(X,Y)\sim\mathcal D_i}[\ell_i(\phi(X),Y)],\quad i\in\mathbb{N}_m.
\]
At this point, we tacitly assume that all expectations are well-defined and finite for
every $\phi\in\mathcal{F}$. Explicit conditions on both losses and hypothesis class will
be enforced in due course (see Section~\ref{sec:pac-dual-theorem}).

Under those circumstances, we consider inequality constrained learning problems of the form
\begin{align}
\label{eq:primal formulation f}
\boxed{
P_{\mathcal F}^\star
\triangleq
\inf_{\phi\in\mathcal F}
\mathcal R_0(\phi)
\quad\text{subject to}\quad
\mathcal R_i(\phi)\le 0,\, i\in\mathbb{N}^+_m,\tag{CSL}
}
\end{align}
under a mild regularity condition \(P_{\mathcal F}^\star> -\infty\).
Problem~\eqref{eq:primal formulation f} captures the task of optimizing a performance
criterion \(\mathcal R_0\) while enforcing side constraints expressed through functionals
\(\mathcal R_1,\dots,\mathcal R_m\), evaluated under possibly distinct, even completely
unrelated data-generating distributions.

\begin{remark}[Constraint levels]\label{rem:constraint-levels}
Constraints of the form $\mathcal R_i(\phi)\le c_i$ with prescribed levels $c_i \in \R$
are subsumed by \eqref{eq:primal formulation f} upon replacing $\ell_i$ by
$\ell_i - c_i$, which preserves Assumption~\ref{ass:loss-regularity} with $M$ replaced by
$M + \max_i|c_i|$. We therefore normalize $c_i \equiv 0$ throughout the body of the paper,
including in the appendices.
\end{remark}

\subsection{Dual ERM} \label{subsec: dual ERM}

A standard and arguably most natural way to tackle
problem~\eqref{eq:primal formulation f} is through some form of \textit{constraint
scalarization}, the simplest form of which is realized within the framework of
\textit{Lagrangian duality} \cite{Rockafellar1970,EkelandTemam1999,BoydVandenberghe2004}.
At the population level, we define the \textit{Lagrangian}
$L_{\mathcal{F}}:\mathcal{F}\times \mathbb{R}^m \rightarrow \mathbb{R}$ and \textit{dual
function} $d_\mathcal{F}: \mathbb{R}^m \rightarrow [-\infty, \infty)$ as
\begin{align}
L_{\mathcal{F}}(\phi,\mu)&
\triangleq
\mathcal R_0(\phi)+\sum_{i=1}^m \mu_i\mathcal R_i(\phi),\nonumber
&
d_{\mathcal F}(\mu)&
\triangleq
\inf_{\phi\in\mathcal F}L_{\mathcal{F}}(\phi,\mu),\nonumber
\end{align}
where \(\mu\in\mathbb R^m\) is a multiplier vector associated with the constraints of the
\textit{primal problem} \eqref{eq:primal formulation f}. If $\mu \in \mathbb{R}^m_+$, then
\(d_{\mathcal F}(\mu)\le P_{\mathcal F}^\star\), motivating maximization of the dual
function over the positive orthant via the \textit{dual problem}
\[
D_{\mathcal F}^\star
\triangleq
\sup_{\mu\in\mathbb R_+^m} d_{\mathcal F}(\mu) \in [-\infty,\infty).
\]
Element $D_{\mathcal F}^\star$ is called the \textit{(optimal) dual value}, for which
\(D_{\mathcal F}^\star\le P_{\mathcal F}^\star\), i.e., \textit{weak duality} always holds.
When equality holds, problem \eqref{eq:primal formulation f} is said to exhibit a
\textit{zero duality gap}, and when optimal dual variables exist, problem
\eqref{eq:primal formulation f} is said to exhibit \textit{strong duality}. In this case,
the dual formulation provides a principled route to solving the primal problem by
replacing hard constraints with properly chosen linear penalties, thus justifying
statistical learning via principled regularization, and revealing a related saddle-point
structure.

In a strictly data-driven setting, the distributions \(\{\mathcal D_i\}_{i=0}^m\) are
unknown, but may be replaced by empirical counterparts. For each \(i\in \mathbb{N}_m \),
let \(\{(X^j_{i},Y^j_{i})\}_{j=1}^n\overset{iid}{\sim}\mathcal D_i\) be a dataset of common
size\footnote{This assumption is made for simplicity, and all our arguments extend to
unequal sample sizes in a straightforward manner.} $n$ inducing an empirical
(random) measure $\widehat{\mathcal{D}}_i^n$ on $\mathcal{X}\times\mathcal{Y}$. In direct
analogy to the usual unconstrained learning paradigm, define \textit{empirical loss
functionals}
\[
\widehat{\mathcal R}_{i,n}(\phi)
\triangleq
\frac{1}{n}\sum_{j=1}^n \ell_i(\phi(X_{i}^j),Y_{i}^j),
\quad i\in\mathbb{N}_m,
\]
where now $\phi\in \mathcal{F}_n$, with $\mathcal{F}_n \subseteq \mathcal{F}$ a subset of
$\mathcal{F}$ depending only on the empirical measures
\( \{\widehat{\mathcal{D}}_i^n\}_i \). One could of course choose
$\mathcal{F}_n = \mathcal{F}$ for all sample sizes $n$, but this is inadvisable when the
complexity of $\mathcal{F}$ is large to start with, which is why we introduce this
flexibility.

In analogy to empirical risk minimization (ERM), consider the \textit{empirical
Lagrangian} $\widehat L_{n}: \mathcal{F}_n \times \mathbb{R}^m \rightarrow \mathbb{R}$ and
\textit{empirical dual function} $\widehat{d}_{n}: \mathbb{R}^m \rightarrow [-\infty,
\infty)$,
\begin{align} \nonumber
\widehat L_{n}(\phi,\mu) &
\triangleq
\widehat{\mathcal R}_{0,n}(\phi)
+\sum_{i=1}^m \mu_i \widehat{\mathcal R}_{i,n}(\phi),\\ \nonumber
\widehat d_{n}(\mu) &
\triangleq
\inf_{\phi\in{\mathcal F}_n}\widehat L_{n}(\phi,\mu),
\end{align}
respectively. One may then hope that, given strong duality for the population problem and
a sufficiently large dataset size $n$, the \textit{empirical dual value}
$\widehat{D}_n^\star \triangleq \sup_{\mu\in\mathbb R_+^m} \widehat d_{n}(\mu)$ is attained
by empirical multipliers $\widehat\mu_n$ such that minimization of the empirical Lagrangian
over $\mathcal{F}_n$ at $\widehat\mu_n$ yields a minimizer approximating a solution of
\eqref{eq:primal formulation f} in \emph{both} value \emph{and} constraint feasibility.
This leads to the \textit{Dual ERM (DERM) family of predictors} satisfying the two-stage
inclusion
\begin{equation}\label{DERM}
\boxed{
\begin{aligned}
\widehat\mu_n & \in \arg\max_{\mu\in\mathbb R_+^m}\widehat d_{n}(\mu),\\
\widehat\phi_n & \in \arg\min_{\phi\in{\mathcal F}_n}
\widehat L_{n}(\phi,\widehat\mu_n),
\end{aligned}\tag{DERM}
}
\end{equation}
or, equivalently, solving the minimax learning problem
$\sup_{\mu \in \mathbb{R}^m_+}\inf_{\phi \in {\mathcal F}_n}\widehat L_n (\phi, \mu)$.

\eqref{DERM} constitutes the basic learning paradigm analyzed in this paper. Note that we
have not discussed strong duality for the empirical constrained primal counterpart of
\eqref{DERM}; this is, in general, too much to ask for in a data-driven nonconvex setting.
Further, even if one can show its statistical consistency relative to
\eqref{eq:primal formulation f} \cite{Chamon20PACC}, such an empirical constrained problem
would not be useful by itself, as it corresponds to an irreducible statement of learning
specifications, not a computational means for obtaining an empirical solution.

On the other hand, \eqref{DERM} motivates establishing strong duality for the population
problem \eqref{eq:primal formulation f}: if the duality gap of
\eqref{eq:primal formulation f} is not controllable in a suitable strong sense then, even
for extremely large $n$, a strictly optimal DERM predictor might perform arbitrarily badly
relative to \eqref{eq:primal formulation f}, even though it would behave consistently
relative to the dual problem.

The strong duality just discussed concerns the optimal \emph{value} of
\eqref{eq:primal formulation f}, and securing it is only half of what a usable learner must
deliver. Even granting a zero duality gap and attained multipliers, a separate question
remains: does the empirical dual solution furnish a single predictor that is provably
\emph{feasible} for the population constraints? In the convex, attained regime this
question answers itself, because a Lagrangian minimizer exists at the optimal multiplier,
its vector of constraint risks is exactly a subgradient of the dual function there, and
complementary slackness converts dual optimality into feasibility of that very minimizer.
None of this survives without convexity. Under the nonconvex losses considered here the
Lagrangian need not attain its infimum over $\mathcal{F}$ at the relevant multipliers, so
there may be no single predictor to certify; and even when near-minimizers exist, dual
subgradient information pins down only an \emph{averaged} constraint-risk vector, a convex
combination of the risks of possibly several distinct near-minimizers. Feasibility of a
DERM output therefore cannot be read off the dual solution and must be established by a
dedicated variational argument, the object of
Sections~\ref{sec:feasibility-indices}--\ref{subsec:main-results}.

\subsection{Universally Learnable Classes under Constraints}
\label{sec:universal learnable classes under constraints}

A critical question is what definition of \textit{concept class learnability} one should
adopt, with the goal of providing approximate solutions to
\eqref{eq:primal formulation f} of adequate quality via a learner such as \eqref{DERM}.

Since we rely on strong duality for asserting statistical consistency of DERM predictors,
we should expect imposition of certain \textit{constraint qualifications (QC)} on
\eqref{eq:primal formulation f} (indispensable even in the convex setting). Constraint
qualifications usually depend on the problem instance at hand and not exclusively on the
choice of objective/constraint loss functions. This implies that assuming a particular QC
on problem \eqref{eq:primal formulation f} imposes implicit restrictions on the possible
choices of population distributions $\{\mathcal{D}_i\}_{i}$. Consequently, rates on the
sufficient number of samples ensuring a certain learning performance should depend on the
class of problems for which a given QC is valid, enforcing dependence on
$\{\mathcal{D}_i\}_{i}$ in addition to the choice of losses. For an example relevant to our
development, consider the class of problems satisfying Slater's condition within a fixed
worst-case constant $\xi > 0$. If the derived sample size rates depend on that $\xi$, the
rate is distribution dependent. As our analysis would show indeed, slater, and other constants like Tikhonov\ref{subsec:tikhonov} highly depend on the dual geometry that itself relies on the distribution of the data. Thus the canonical uniform PAC is not a good choice in this setting. 

In order to conceptually simplify our development while avoiding unnecessary
technicalities, we allow nonuniform, instance-wise distribution-dependent concept class
learnability, within a primarily agnostic probably approximately correct (PAC) framework.

\begin{defBox}{Agnostic Universal Class Learnability}{nonuniform-pacc}
A hypothesis class $\mathcal F$ is said to be \emph{agnostically universally Probably
Approximately Correct with Constraints (PACC) learnable}, or simply \textit{universally
learnable}, relative to loss components $\{\ell_i\}_{i\in\mathbb{N}_m}$ of the constrained
problem \eqref{eq:primal formulation f} iff the following holds:

\smallskip
\hspace{12pt}\emph{There exists an algorithm $\mathcal A$, operating exclusively on
empirical measures $\{\widehat{\mathcal{D}}_i^n\}_{i=0}^m$ induced by $m+1$ \textit{iid}
datasets each of size $n\in \mathbb{N}^+$, such that for every $\eps>0$, every
$\delta\in(0,1)$, and every tuple of population distributions
$\{\mathcal{D}_i\}_{i=0}^m$, there is a natural number
$\widetilde{n}(\eps,\delta,\{\mathcal{D}_i\}_{i=0}^m)\in\mathbb N^+$ such that, whenever
$n\ge \widetilde{n}(\eps,\delta,\{\mathcal{D}_i\}_{i=0}^m)$, the operator
$\mathcal{A}\equiv\mathcal{A}(\{\widehat{\mathcal{D}}_i^n\}_{i=0}^m)$ produces (at least
one) predictor $\widehat\phi_n\in\mathcal F$ satisfying}
\[ \Pr\big(
        \mathcal R_0\big(\widehat\phi_n\big)\le P_{\mathcal F}^\star+\eps
        \ \text{ and }\
        \mathcal R_i\big(\widehat\phi_n\big)\le \eps,\ i\in \mathbb{N}^+_m
    \big)
    \ge 1-\delta.
\]
\end{defBox}

The term \emph{universal} indicates precisely that the required statistical barrier
$\widetilde{n}(\cdot)$ is allowed to depend on the reference population distributions.
This notion is adopted from \cite{bousquet2020theoryuniversallearning} and extended here
to the setting of agnostic constrained learning. In all cases, learning algorithms
themselves are totally agnostic to the population distributions, depending exclusively on
input datasets; it is only the complexity index $\widetilde{n}$ that is allowed to depend
on population distributional information.

In passing, existing PAC-type learnability results in the constrained learning literature
(see, e.g., \cite{Chamon20PACC, Chamon20csl, BarthakurChamon2025Equality}) are necessarily
of a universal character as well, even if relevant definitions in
\cite{BarthakurChamon2025Equality,Chamon20PACC} appear to be stated in the uniform sense.
This is because Slater's condition is imposed on the population constrained problems
considered therein, as discussed above and thus all the bounds would depend on the data generating distributions $\{\mathcal{D}_i\}_{i=0}^m$.

\begin{remark}
The relaxed requirement regarding existence of ``at least one'' predictor in
Definition~\ref{def:nonuniform-pacc} is due to the presence of constraints, and is largely
akin to the nature of the initial constrained problem \eqref{eq:primal formulation f}. It
is a technicality imposed by minimax (Lagrangian) duality.
\end{remark}

\paragraph{The universal qualifier is necessary, not merely convenient.}
There is an extra demand in our case of when we pose the problem over the high expressive RKHS classes for the notion of universal learning. In \ref{prop:no-distribution-free}  we prove that universality is a necessary condition even in the unconstrained case.   The obstruction is
\emph{not} only the Slater margin but the fact that a universal RKHS has, by design, \textbf{unbounded capacity}, so that
$P^\star_{\HK}$ coincides with the unrestricted (Bayes-type) optimum and no
distribution-free rate toward it can exist.

\subsection{The Tension: Strong Duality vs Dual Generalizability} \label{subsec: the tension}

Generalizability of a dual-domain constrained learning predictor, in the sense of
Definition~\ref{def:nonuniform-pacc}, relies on two fundamental requirements:

\medskip\noindent\textbf{R1: Strong duality at the population level.}

\noindent\textbf{R2: Statistical generalizability of empirical dual-domain predictors.}
\medskip

\noindent These two requirements are intimately related to the choice of the hypothesis
class $\mathcal{F}$ used for formulating \eqref{eq:primal formulation f}, and typically
pull in opposite directions. In a nutshell, \textbf{R1} favors \emph{large} hypothesis
classes, rich enough to support the convexification phenomena essential for strong duality.
\textbf{R2} favors \emph{small} hypothesis classes with less expressivity but controlled
statistical complexity, permitting uniform convergence and measure concentration arguments.
One of the contributions of this paper is to reconcile these two demands within a single
compatible framework.

More specifically, \textbf{R1} is imposed due to potential nonconvexity of the loss
functions. While in convex optimization strong duality is guaranteed under convexity
together with a suitable constraint qualification
\cite{Rockafellar1970,EkelandTemam1999}, asserting strong duality without convexity
requires a different mechanism. The one adopted here relies on the \textit{Lyapunov
convexity theorem} from vector measure theory
\cite{Lyapunov1940,Blackwell51,Aumann1965,DiestelUhl77} and, most importantly, on its weak
extension to nonatomic vector measures taking values in infinite-dimensional Banach
spaces, due to Uhl \cite{Uhl69Range} (see also
\cite{Knowles1975Lyapunov,KadetsSchechtman92,KhanSagara2014Lyapunov}), as exploited in
\cite{KalogeriasPougkakiotis22}; for classical applications of Lyapunov-type
convexification in the calculus of variations and mathematical economics, see
\cite{BerliocchiLasry73,AumannPerles65,AubinEkeland76} and
\cite[Appendix~I.4]{EkelandTemam1999}. The mechanism requires that the hypothesis class
$\mathcal{F}$ be a so-called \emph{decomposable} set
\cite{Rockafellar68Integrals,HiaiUmegaki1977,RockafellarWets1998}.\footnote{The definition
of decomposability adopted here (see, e.g., \cite{HiaiUmegaki1977}) is slightly different
from the original definition in \cite{Rockafellar68Integrals}; see also the discussion in
\cite{KalogeriasPougkakiotis22}.}

\begin{defBox}{Decomposable Set/Space}{decomposable}
A set $\mathcal{F} \subseteq \{\phi : \mathcal{X} \to \mathbb{R}\}$ of
$\mathcal{B}(\mathcal{X})$-measurable functions is called \emph{decomposable} iff for every
$\phi, \phi' \in \mathcal{F}$ and every $Z \in \mathcal{B}(\mathcal{X})$, the
\textit{splice} below also belongs to $\mathcal{F}$:
\[
    \bar{\phi}(x)
    \;\triangleq\;
    \mathbf{1}_Z(x)\,\phi(x)
    + \mathbf{1}_{\mathcal{X} \setminus Z}(x)\,\phi'(x).
\]
\end{defBox}

\noindent Decomposability allows recovery of \textit{convexity at the level of attainable
risk vectors} generated by the objective and constraints of
\eqref{eq:primal formulation f} (the \textit{image space} viewpoint), even when convexity
and even continuity are absent at the level of the underlying functionals. Convexity at
this level is just enough to ensure strong duality, as a consequence of convex geometry
\cite{KalogeriasPougkakiotis22}.

Canonical examples of decomposable sets include all measurable functions on any measurable
space, all $L_p$ spaces for $p\in [1,\infty]$
\cite[Theorem~14.60]{RockafellarWets1998} (see also
\cite[Section~4.1]{KalogeriasPougkakiotis22}), all Orlicz spaces
\cite{RaoRen1991,Vershynin2018HDP}, and sets of essentially bounded functions of the form
$\{f\in L_\infty (\mu,\mathbb{R}) \mid f(\cdot) \in \mathcal{U}, \mu\text{-a.e.}\}$ for
\textit{any} set $\mathcal{U} \subseteq \mathbb{R}$. Typical examples of nondecomposable
sets are convex $L_p$ balls of finite radii for any $p\in[1,\infty)$, spaces of
continuous/differentiable/smooth functions, finite-dimension-indexed sets of functions,
and, most importantly here, RKHSs, either norm-constrained or not. In general, convex sets
need not be decomposable, and decomposable sets can be arbitrarily nonconvex; the two
notions neither imply nor are implied by one another. A key characteristic of decomposable
sets is that they are ``very large'': they contain functions of arbitrary complexity, since
a splice enforces discontinuity by construction. This largeness is precisely what makes
them suitable for strong duality, and also what makes them statistically unwieldy.

On the other hand, even if strong duality holds at the population level, a learning
algorithm is meaningful only if the corresponding empirical problem is statistically well
posed and consistent, in the sense of \textbf{R2}. Within the DERM family, satisfying
\textbf{R2} is related to properly selecting the function class $\mathcal{F}_n$ associated
with the empirical Lagrangian problem solved at sample size $n$.

In the unconstrained setting, statistical learnability is usually controlled by complexity
measures such as \textit{VC dimension} and \textit{Rademacher complexity}
\cite{VapnikChervonenkis1971,Vapnik98,AnthonyBartlett1999,BartlettMendelson02,
ShalevShwartzBenDavid2014,MohriRostamizadehTalwalkar18}. In binary classification, finite
VC dimension is equivalent to distribution-free PAC learnability
\cite{VapnikChervonenkis1971,BlumerEhrenfeuchtHausslerWarmuth1989,ShalevShwartzBenDavid2014}.
For real-valued classes and Lipschitz losses, which is our focus, the relevant complexity
measure is Rademacher complexity.

\begin{defBox}{Rademacher Complexity}{rademacher}
Let $S = ((X_1,Y_1), \ldots, (X_n,Y_n))$ be a random population vector and let
$\sigma_1, \ldots, \sigma_n$ be \textit{iid} Rademacher random variables taking values
$\pm 1$ with equal probability. The \emph{empirical Rademacher complexity} of the composed
class $\ell_0 \circ \mathcal{F} \triangleq \{(x,y) \mapsto \ell_0(\phi(x),y) \mid \phi \in
\mathcal{F}\}$ is
\[
    \widehat{\mathfrak{R}}_S(\ell_0 \circ \mathcal{F})
    \;\triangleq\;
    \mathbb{E}_\sigma\!\left[
        \sup_{\phi \in \mathcal{F}}
        \frac{1}{n}\sum_{j=1}^n \sigma_j\,\ell_0(\phi(X_j), Y_j)
    \right],
\]
and the \textit{Rademacher complexity} is the population average
$\mathfrak{R}_n(\ell_0 \circ \mathcal{F}) \triangleq \mathbb{E}_S[\widehat{\mathfrak{R}}_S(\ell_0 \circ \mathcal{F})]$.
\end{defBox}

Rademacher complexity is useful because standard symmetrization and contraction arguments
yield uniform convergence bounds of the form
\(
    \sup_{\phi\in\mathcal F}
    |\widehat{\mathcal R}_{0,n}(\phi)-\mathcal R_0(\phi)|
    \lesssim
    \mathfrak R_n(\ell_0\circ\mathcal F),
\)
up to lower-order confidence terms
\cite{Dudley1999,BartlettMendelson02,ShalevShwartzBenDavid2014}. When
$\mathfrak R_n(\ell_0\circ\mathcal F)\searrow 0$, these yield nontrivial high-probability
generalization bounds governed by the decay rate. One naturally expects such behavior to be
useful for \textbf{R2} in the constrained setting as well, and we show in
Section~\ref{sec:pac-dual-theorem} that this is indeed the case.

Decaying Rademacher complexity is typically guaranteed for ``small'' model complexity
classes, for instance norm-constrained RKHS classes or finite-dimensional function classes.
As discussed above, however, such concept classes are not decomposable, and therefore are
not suitable for establishing strong duality when used in the population problem. Prior
attempts at establishing PAC learnability using such classes with built-in decaying
Rademacher complexity \emph{both} in the population problem and within the DERM paradigm
unavoidably induce irreducible constants in the resulting generalization bounds; see, e.g.,
\cite[Theorem~1]{Chamon_CL2023} and \cite[Theorem~3.1]{BarthakurChamon2025Equality}. In
contrast, strong duality naturally manifests on large decomposable spaces such as $L_p$,
where Lyapunov-type arguments apply in the absence of loss convexity, but such spaces are
far too large to imply useful generalization bounds.

We are therefore led to a fundamental structural tension on the essential features of the
concept classes used for postulating \eqref{eq:primal formulation f} on the one hand and
for defining a consistent empirical dual-domain predictor on the other, and a key question
is whether suitable and compatible choices exist that serve both purposes.

\subsection{Reconciling the Tension: Dense RKHS Classes and Ball Restrictions}
\label{subsec: reconcile tension}

We now describe a mechanism by which the tension between \textbf{R1} and \textbf{R2} is
provably balanced. The idea is to leverage large RKHSs as ``master'' hypothesis classes for
\eqref{eq:primal formulation f} which \emph{simultaneously} \textbf{(1)} approximate large
decomposable \textit{envelope} spaces to arbitrary precision, despite being nondecomposable
themselves, \emph{and} \textbf{(2)} admit controlled statistical complexity via RKHS
norm-ball restrictions. The underlying principle is similar to structural risk minimization
(SRM) \cite{Vapnik98,ShaweTaylor98SRM}, although our purpose is of a different nature.

\paragraph{Dense RKHSs and universal approximation.}
Let \(K:\mathcal X\times\mathcal X\to\mathbb R\) be a symmetric positive definite
(measurable) kernel. By Aronszajn's characterization theorem \cite{Aronszajn1950}, there
exists a unique Hilbert space \(\HK\) of measurable functions on \(\mathcal X\) such that
\(K(\cdot,x)\in\HK\) for every \(x\in\mathcal X\), and the \textit{reproducing property}
\(f(x)=\langle f,K(\cdot,x)\rangle_{\HK}\) holds for all \(f\in\HK\). If
\(\sup_{x\in\mathcal X}K(x,x)\le\kappa^2<\infty\), the reproducing property with
Cauchy--Schwarz gives
$\sup_{x\in \mathcal{X}}|f(x)| \le \kappa\|f\|_{\HK}$ for all $f\in\HK$,
so $\HK \subseteq L_\infty(\mu,\R)$ for any finite $\mu$ on
$(\mathcal{X},\mathcal{B}(\mathcal{X}))$, and hence $\HK \subset L_p(\mu,\R)$, with
$\sup_{\mu(\mathcal X)\le1}\|f\|_{L_p(\mu,\R)}\le\kappa\|f\|_{\HK}$ for every
$p\in[1,\infty]$.

We are particularly interested in RKHSs which are \emph{dense} subsets of
$L_p(\mu,\R)$ for some $p \in[1,\infty)$, for a suitable choice of $\mu$ discussed in
Section~\ref{sec:pac-dual-theorem}. Compactly,
\[
    \mathrm{cl}_{\|\cdot\|_{L_p(\mu,\R)}}\bigl(\HK\bigr) = L_p(\mu,\R),
    \quad\text{for some}\quad 1\le p <\infty,
\]
since $L_p(\mu,\R)$ is Banach. Slightly relaxing standard terminology
\cite{Steinwart2001,MicchelliXuZhang2006,SteinwartChristmann2008}, we refer to such RKHSs
as \textit{universal}
\cite{Steinwart2001,MicchelliXuZhang2006,SriperumbudurFukumizuLanckriet2011}.

A canonical example is the RKHS induced by the Gaussian kernel
\(K_\sigma(\cdot,\bullet) \triangleq \exp(-\|(\cdot)-(\bullet)\|^2/2\sigma^2)\),
$\sigma>0$, which is dense in $L_p(\mu,\R)$ for every $p\in[1,\infty)$, any Borel
probability measure $\mu$, and any support set $\mathcal{X}\subseteq \R^d$
\cite[Theorem~4.63]{SteinwartChristmann2008}. Other examples include the Paley--Wiener
space, dense in $L_2(\mu,\R)$ for any $\mu$ with compact support; the exponential RKHS
\cite{BerlinetThomasAgnan2004RKHS}; and the binomial RKHS
\cite{MicchelliXuZhang2006,SteinwartChristmann2008}; several explicit recipes for
constructing universal kernels are given in \cite{MicchelliXuZhang2006}.

The key observation is that although universal RKHSs are not decomposable, they are dense
in vast decomposable spaces ($L_p$), which we call \emph{envelopes}. This has profound
consequences for establishing strong duality over nondecomposable RKHSs, as shown in
Section~\ref{sec:proofs}.

\paragraph{RKHS ball restrictions and complexity control.}
A dense RKHS $\HK$ makes sense as a concept class for a population-level learning problem,
but is still too large to yield learnability guarantees from finite data. The standard
resolution, adopted here, is to restrict attention to closed RKHS norm balls
\[
    \Bq{Q} \triangleq \{f\in\HK \mid \|f\|_{\HK}\le Q\},
\]
whose statistical complexity can be explicitly controlled. Under
$\sup_{x}K(x,x)\le\kappa^2<\infty$, standard Rademacher arguments give
\cite[Lemma~22]{BartlettMendelson02},\cite[Theorem~6.12]{MohriRostamizadehTalwalkar18}
\[
    \mathfrak R_n(\Bq{Q}) \le \frac{\kappa Q}{\sqrt{n}},
\]
so the statistical cost of working within $\Bq{Q}$ is transparent and vanishes as
$n\to\infty$, provided the radius $Q$ is constant or grows at a rate strictly smaller than
$\sqrt{n}$.

\begin{remark}[Radius schedules used in this paper]\label{rem:schedule}
Throughout the body we require only
\begin{equation}\label{eq:schedule}
Q_n\nearrow\infty
\qquad\text{and}\qquad
Q_n/\sqrt n\searrow 0 ,
\end{equation}
for which $Q_n = n^{1/4}$ is the canonical instance. The sharpness-based sufficient
condition for the recovery assumption
(Proposition~\ref{prop:sharpness} in Appendix~\ref{sec:on-last-assumption}) requires the
strictly stronger budget $Q_n = o(n^{1/4})$, which \emph{excludes} $Q_n=n^{1/4}$.
Whenever both are in force we use $Q_n = n^{1/4}/\log(e+n)$, which satisfies
\eqref{eq:schedule} and the $o(n^{1/4})$ budget simultaneously, and we say so explicitly.
\end{remark}

\paragraph{Constrained learning strategy.}
While each individual restricted class $\Bq{Q}$ is not dense in any $L_p$ space, we have
$\Bq{Q_n} \subset \Bq{Q_{n+1}}$ and $\bigcup_{n\in \mathbb{N}^+}\Bq{Q_n} = \HK$ for
$Q_n\nearrow\infty$. Similarly to the SRM principle \cite{Vapnik98,ShaweTaylor98SRM}, this
motivates a \emph{variable-complexity} strategy: rather than fixing a single hypothesis
class for both the population and empirical problems, necessarily of finite model
complexity as in
\cite{Chamon20csl,Chamon20PACC,Chamon_CL2023,Boero25ALCoLe,BarthakurChamon2025Equality}, we
first pose the population problem \eqref{eq:primal formulation f} over a large dense RKHS
$\HK = \mathcal{F}$ --- in the absence of special structural predictor constraints at the
population level, there is no reason to discount expressivity when the predictor is
notionally trained on the whole population --- and then approach that problem by solving
\eqref{DERM} for each dataset size $n$, over a sequence of smaller classes
$\{\Bq{Q_n} =\mathcal{F}_n \}_{n \in \mathbb{N}^+}$ whose expressivity grows with $n$. At
scale $Q_n$ the ball $\Bq{Q_n}$ is statistically learnable at rate
$\mathcal{O}(Q_n/\sqrt{n})$, whereas as $Q_n \nearrow\infty$ the relative complement of
$\Bq{Q_n}$ in $\HK$ shrinks to the empty set.

The net effect is that \textit{a vanishing duality gap at the population level}, which
requires density in a decomposable envelope, \textit{and a vanishing statistical estimation
error at the empirical level}, which requires decaying Rademacher complexity, \textit{are
achieved simultaneously}, at a rate governed by the growth of $Q_n$. We will further see
that, under one additional population-level assumption
(Assumption~\ref{ass:multi-argmin-recovery}), the same regime yields asymptotic feasibility
as well; see Section~\ref{subsec:feasibility-geometry}.


\section{Universal Learnability of Dense RKHSs under Constraints}
\label{sec:pac-dual-theorem}

This section states our assumptions and presents the main results. The development follows
the programme of Section~\ref{sec:problem-setup}: we pose the population problem over a
large dense (universal) RKHS, whose expressivity underwrites strong duality, and we learn
from finite data over a growing sequence of RKHS norm balls, whose complexity we can
control. We proceed in six steps. Section~\ref{subsec:instantiation} instantiates the
population problem, its dual, and the ball-restricted DERM learner.
Section~\ref{subsec:assumptions} collects four structural assumptions and introduces the
decomposable \emph{envelope}. Section~\ref{subsec:feasibility-geometry} introduces the
geometric objects governing constraint feasibility, culminating in the
\emph{closure--realization gap} $\epsinf$ and a lemma collecting its properties.
Section~\ref{subsec:tikhonov} introduces the \emph{Tikhonov complexity}.
Section~\ref{subsec:recovery} isolates the one non-structural regularity condition needed
for feasibility, and records precisely what it costs. Sections~\ref{subsec:main-theorem}
and~\ref{subsec:explicit-rates} state the main theorem and convert it into an explicit
polynomial threshold under a source condition.

\begin{table}[t]
\centering\small
\caption{Principal notation.}
\label{tab:notation}
\begin{tabular}{@{}ll@{}}
\toprule
Symbol & Meaning \\
\midrule
$\mathcal X,\mathcal Y,\mathcal D_i$ & feature space, label space, $i$-th population distribution \\
$\mathcal R_i,\widehat{\mathcal R}_{i,n}$ & population / empirical risk functionals \\
$\mathcal F$ & generic hypothesis class; $\mathcal F_n$ its sample-size-dependent restriction \\
$\HK$ & the universal RKHS used as population concept class \\
$\Gen = L_p(\mathcal M_{\mathcal X},\R)$ & the decomposable envelope; $\mathcal M_{\mathcal X}$ the mixture marginal \\
$\Bq{Q}$, $Q_n$ & closed RKHS norm ball of radius $Q$; radius schedule \eqref{eq:schedule} \\
$L_{\HK},\widehat L_n$ & population / empirical Lagrangian \\
$d_{\HK},d_Q,\widehat d_n$ & population / ball-restricted / empirical dual functions \\
$P^\star_\bullet, D^\star_\bullet$ & primal and dual optimal values, subscript naming the class \\
$\widehat D^\star_n$ & DERM optimal value, $\sup_\mu\widehat d_n(\mu)$ \\
$L,M,\kappa$ & Lipschitz constant, loss bound, kernel diagonal bound \\
$\xi$ & Slater margin (Assumption~\ref{ass:slater}) \\
$C_\mu = 4M/\xi$, $\Lambda$ & multiplier bound and compact multiplier set \\
$\Sdual$ & population dual-optimal set \\
$\Phi_{\HK}(\mu),\widehat\Phi_n(\mu)$ & population / empirical Lagrangian argmin sets \\
$\Theta_\nabla,\Theta^\infty_\nabla$ & empirical / population image-set deviation (Def.~\ref{def:image-set-deviation}) \\
$\epsinf$ & closure--realization gap (Def.~\ref{def:multi-eps-star}) \\
$\mathfrak T^\eps_n$, $R_\eps(\mu)$, $Q_\eps$ & Tikhonov complexity, min-norm selector, uniform bound \\
$\omega$, $\Omega,\Sigma,\Pr$ & sample path and underlying probability space \\
$\Omega_{\mathrm{uc}}(\delta),\Omega_{\mathrm{fbl}},\Omega^\star(\delta)$ & uniform-convergence, feasibility-tracking, and combined events \\
$\Delta_n(\delta)$ & uniform-convergence envelope \eqref{eq:uc-envelope} \\
\bottomrule
\end{tabular}
\end{table}

\subsection{The Constrained Learning Problem over a Dense RKHS}
\label{subsec:instantiation}

Fix a universal RKHS $\mathcal F = \HK$ and consider the population-level base constrained
learning problem
\begin{align}
\label{eq:primal formulation RKHS}
\boxed{
P_{\HK}^\star
    \triangleq
\inf_{\phi\in\HK}
    \mathcal R_0(\phi) \quad \text{subject to} \quad
    \mathcal R_i(\phi)\le 0,
    \, i \in \mathbb{N}_m^+,}\tag{$\mathcal H$-CSL}
\end{align}
together with its Lagrangian dual
\begin{align}
\label{eq:dual formulation RKHS}
D_{\HK}^\star
    \triangleq \sup_{\mu \in \mathbb R^{m}_{+}}
    \inf_{\phi\in\HK}
    \bigg\{
    L_{\HK}(\phi, \mu)
    = \mathcal R_0(\phi) + \sum_{i=1}^m \mu_i \mathcal{R}_i(\phi)
    \bigg\}.
    \tag{$\mathcal{H}$-DSL}
\end{align}
We learn from data by instantiating the DERM predictor of Section~\ref{subsec: dual ERM}
under the variable-complexity strategy of Section~\ref{subsec: reconcile tension}. Given
datasets of size $n$, we restrict the hypothesis class to
$\Bq{Q_n} = \{\phi\in\HK \mid \|\phi\|_{\HK}\le Q_n\}$, set
$\mathcal{F}_n = \Bq{Q_n}$, and replace population functionals by their empirical
counterparts, yielding the RKHS DERM problem
\begin{equation}
\label{eq:dual formulation DERM}
\boxed{
\widehat{D}_{n}^\star
    \triangleq \sup_{\mu \in \mathbb R^{m}_{+}}
    \underbrace{\inf_{\phi\in \Bq{Q_n}}
      \bigg\{
      \widehat{L}_n(\phi,\mu)=
      \widehat{\mathcal R}_{0,n}(\phi) + \sum_{i=1}^m \mu_i \widehat{\mathcal{R}}_{i,n}(\phi)
      \bigg\}
      }_{=\, \widehat{d}_n(\mu)},\qquad \begin{aligned}
\widehat\mu_n & \in \arg\max_{\mu\in\mathbb R_+^m}\widehat d_{n}(\mu),\\
\widehat\phi_n & \in \arg\min_{\phi\in{\Bq{Q_n}}}
\widehat L_{n}(\phi,\widehat\mu_n),
\end{aligned} \tag{$\mathcal{H}$-DERM}
}
\end{equation}
where the radii $Q_n$ satisfy \eqref{eq:schedule}. The first condition lets the search
space eventually reach any fixed model complexity; the second keeps its Rademacher
complexity vanishing, at rate $\mathcal O(\kappa Q_n/\sqrt n)$.

\begin{remark}[Notation for optimal values]\label{rem:value-notation}
We write $\widehat D^\star_n$ for the DERM optimal value, reserving the subscript $Q$ for
the \emph{population} ball-restricted dual value $D^\star_Q$ defined next. An earlier
draft used $\widehat D^\star_n$, $\widehat D^\star_{Q_n}$ and $\widehat D^\star_Q$
interchangeably for the same object; they all denote $\widehat D^\star_n$.
\end{remark}

It is convenient to isolate the purely population-level effect of the ball restriction
through the \textit{ball-constrained population dual}
\begin{align}\label{BQ-DSL}
D_{Q}^\star
    \triangleq \sup_{\mu \in \mathbb R^{m}_{+}}
    \inf_{\phi\in \Bq{Q}}
      L_{\HK}(\phi, \mu),
\qquad Q<\infty. \tag{$\mathcal{B}_Q$-DSL}
\end{align}
Since $\Bq{Q} \subset \Bq{Q'}$ for $Q<Q'$ and $\bigcup_{Q>0}\Bq{Q} = \HK$, the value
$D_{Q}^\star$ is a monotone population-level surrogate of \eqref{eq:dual formulation RKHS};
we show in Section~\ref{sec:proofs} that $D_{Q}^\star \searrow D_{\HK}^\star$ as
$Q\to\infty$ under the standing assumptions below. Thus
\eqref{eq:dual formulation DERM} sits between two limits --- statistical estimation at fixed
radius, and radius growth toward the full RKHS dual --- and the analysis is precisely the
coupling of these two limits.

\subsection{Standing Assumptions and the Decomposable Envelope}
\label{subsec:assumptions}

Our results rest on four assumptions placed on \eqref{eq:primal formulation RKHS}. Because
the $m+1$ populations may be unrelated, we first fix a common reference measure. Writing
$\mathcal{D}_{\mathcal{X},i}$ for the marginal of $\mathcal{D}_i$ on $\mathcal{X}$, define
the normalized mixture
\begin{equation*}
  \mathcal{M}_{\mathcal{X}} (\mathrm{d}x) \triangleq \frac{1}{m+1}\sum_{i=0}^{m}\mathcal{D}_{\mathcal{X},i}(\mathrm{d}x),
\end{equation*}
a probability measure satisfying
$\mathcal{D}_{\mathcal{X},i} \ll \mathcal{M}_{\mathcal{X}}$ for all $i \in \mathbb{N}_m$.
All four assumptions are stated relative to $\mathcal{M}_{\mathcal{X}}$.

\begin{assumptionBox}{Nonatomicity of Feature Marginals}{non-atomic}
The mixture measure $\mathcal{M}_{\mathcal{X}}$ is nonatomic: for every
\(A\in\mathcal B(\mathcal X)\) with $\mathcal{M}_{\mathcal{X}}(A) > 0$, there exists
\(B\in\mathcal B(\mathcal X)\) with $B\subset A$ and
$0<\mathcal{M}_{\mathcal{X}}(B) < \mathcal{M}_{\mathcal{X}}(A)$.
\end{assumptionBox}

\begin{assumptionBox}{Loss Regularity}{loss-regularity}
For each \(i\in \mathbb{N}_m\), the loss \(\ell_i:\R\times\mathcal Y\to\R\) satisfies, with
constants $L,M$ independent of $i$:
\begin{enumerate}[label=(\alph*), itemsep=0pt, topsep=2pt]
    \item $\ell_i(\cdot,y)$ is $L$-Lipschitz, uniformly in $y\in\mathcal Y$;
    \item $\ell_i$ is uniformly $M$-bounded.
\end{enumerate}
\end{assumptionBox}

\begin{assumptionBox}{Concept Class Expressivity}{dense-rkhs}
The population concept class is $\mathcal{F}=\HK$, where:
\begin{enumerate}[label=(\alph*), itemsep=0pt, topsep=2pt]
    \item \(\HK\) is the RKHS of a symmetric, positive definite kernel $K$ with
    $\sup_{x}K(x,x)\le\kappa^2<\infty$;
    \item \(\HK\) is dense in the decomposable space
    $\Gen \triangleq L_p(\mathcal{M}_{\mathcal{X}},\R)$ for some $p \in [1,\infty)$.
\end{enumerate}
\end{assumptionBox}

\begin{assumptionBox}{Strict Feasibility (Slater's Condition)}{slater}
There exist \(\phi^{\mathrm{sl}}\in\HK\) and a margin \(\xi>0\) with
\(\mathcal R_i(\phi^{\mathrm{sl}})\le -\xi\) for all \(i\in\mathbb{N}^+_m\). We write
$Q_{\mathrm{sl}}\triangleq\|\phi^{\mathrm{sl}}\|_{\HK}$.
\end{assumptionBox}

Two of these deserve comment. Assumption~\ref{ass:non-atomic} is mild and broadly
applicable: it precludes feature marginals with atoms, but holds whenever the
$\{\mathcal{D}_{\mathcal{X},i}\}_i$ admit Lebesgue densities, and imposes \emph{no}
restriction on the response variables. Its role is to supply the nonatomicity needed by the
Lyapunov-type convexification underlying strong duality.

Assumption~\ref{ass:dense-rkhs}(b) is the linchpin of the reconciliation. The RKHS $\HK$
delivers statistical control but is \emph{not} decomposable, so Lyapunov convexification
does not apply to it directly. Density in the decomposable envelope
$\Gen = L_p(\mathcal M_{\mathcal X},\R)$ is what lets us transport strong duality from
$\Gen$ back to $\HK$. Concretely, the envelope carries its own constrained problem
\begin{align}
\label{eq:primal formulation envelope}
P_{\Gen}^\star
\triangleq
\inf_{\phi\in\Gen}
\mathcal R_0(\phi)
\quad\text{subject to}\quad
\mathcal R_i(\phi)\le 0,\, i\in\mathbb{N}^+_m,
\tag{$\mathcal{G}$-CSL}
\end{align}
with Lagrangian $L_{\Gen}(\phi,\mu)$, dual function
\(d_{\Gen}(\mu)\triangleq\inf_{\phi\in\Gen} L_{\Gen}(\phi,\mu)\), and dual value
\(D_{\Gen}^\star\). Under Assumptions~\ref{ass:non-atomic}--\ref{ass:dense-rkhs}, strong
duality holds on $\Gen$; density plus continuity of the $\mathcal R_i$ (a consequence of
Assumption~\ref{ass:loss-regularity}) then yields $P^\star_{\HK}=P^\star_{\Gen}$ and
$D^\star_{\HK}=D^\star_{\Gen}$, so strong duality is inherited by $\HK$ as well. As a
byproduct, our analysis also establishes universal learnability of $\Gen$ used as the
target class in \eqref{eq:primal formulation envelope}: the learner returns
$\widehat\phi_n\in\Bq{Q_n}\subset\HK\subset\Gen$, so the learning is \emph{proper} for both
classes.

Finally, Assumption~\ref{ass:slater} is imposed on $\HK$, which is the stronger place to
state it: if $\phi^{\mathrm{sl}}\in\HK$ is strictly feasible with margin $\xi$, then it is
trivially strictly feasible in $\Gen\supset\HK$. Conversely, strict feasibility on $\Gen$
transfers to $\HK$ with a possibly halved margin by an elementary density argument.

\subsection{The Geometry of Feasibility}
\label{subsec:feasibility-geometry}

Strong duality secures the \emph{value} of the problem, but says nothing about whether a
DERM output is \emph{feasible}. This subsection introduces the geometric objects that
control feasibility. The picture to keep in mind: at a dual optimum, the first-order
conditions pin down a target vector $\nabla$ (a supergradient of the dual function); a DERM
output is feasible as soon as its constraint-risk vector matches that target. Whether it can
is a property of the RKHS geometry, captured by a single index.

\paragraph{Sampling model.}
For each $i\in\mathbb N_m$, let
$\omega^{(i)}\triangleq\{(X^{(i)}_j,Y^{(i)}_j)\}_{j=1}^\infty\overset{\mathrm{iid}}{\sim}\mathcal D_i$,
and let $\omega\triangleq(\omega^{(0)},\dots,\omega^{(m)})$ be the joint sample path on the
product space $(\Omega,\Sigma,\Pr)$, where
\[
\Omega \triangleq \underbrace{\bigl(\mathcal X\times\mathcal Y\bigr)^{\mathbb N}\times\cdots\times\bigl(\mathcal X\times\mathcal Y\bigr)^{\mathbb N}}_{(m+1\text{ factors})},
\qquad
\Pr \triangleq \bigotimes_{i=0}^m \mathcal D_i^{\otimes\mathbb N},
\]
and $\Sigma$ is the product $\sigma$-algebra. We write
$\widehat{\mathcal R}_{i,n}(\phi;\omega)$ for the empirical risk from the first $n$ samples
of coordinate $i$, abbreviated $\widehat{\mathcal R}_{i,n}(\phi)$ when $\omega$ is clear.
The symbol $\omega$ is reserved for sample paths throughout; $\xi$ is reserved for the
Slater margin.

\paragraph{Lagrangian argmin sets and the first-order picture.}
Define the empirical and population Lagrangian argmin sets
\[
\widehat{\Phi}_n(\mu) \triangleq \argmin_{\phi \in \Bq{Q_n}} \widehat{L}_n(\phi,\mu),
\qquad
\Phi_{\HK}(\mu) \triangleq \argmin_{\phi \in \HK} L_{\HK}(\phi,\mu).
\]
$\widehat\Phi_n(\mu)$ is nonempty since, by the representer theorem, minimizing
$\widehat L_n(\cdot,\mu)$ over $\Bq{Q_n}$ reduces to minimizing over the intersection of
$\Bq{Q_n}$ with the finite-dimensional span of the kernel sections at the data points --- a
compact set on which the objective is continuous; the population case is addressed by
Assumption~\ref{ass:multi-argmin-recovery}(a). Since $\widehat d_n$ is concave, first-order
optimality of $\widehat\mu_n$ over the orthant yields a supergradient
$\nabla\in\partial\widehat d_n(\widehat\mu_n)$ with
$\nabla_i \le 0$ and $\widehat\mu_{n,i}\nabla_i = 0$ for $i\in\mathbb N^+_m$
(Lemma~\ref{lem:sign-comp-subgrad}). By the Levin--Valadier representation of the
superdifferential of an infimal function \cite{Valadier,IoffeTihomirov1979}, every such
supergradient lies in the closed convex hull of empirical constraint-risk vectors evaluated
at Lagrangian minimizers,
\begin{equation}\label{eq:LV-empirical}
\nabla \;\in\; \overline{\mathrm{conv}}\,\bigl\{
\bigl(\widehat{\mathcal{R}}_{1,n}(\phi),\ldots,\widehat{\mathcal{R}}_{m,n}(\phi)\bigr)
: \phi \in \widehat{\Phi}_n(\widehat\mu_n)
\bigr\}.
\end{equation}
The precise form of the representation we use, together with a self-contained verification
that no additional normal-cone term arises from the active ball constraint, is
Lemma~\ref{lem:LV} in Section~\ref{sec:feasibility-indices}.

Because $\nabla_i\le0$, feasibility of a DERM output would follow \emph{immediately} if its
constraint-risk vector coincided with $\nabla$. The obstruction is therefore exactly the gap
between the realized constraint-risk vector and the target $\nabla$.

\begin{defBox}{Image-set Deviation}{image-set-deviation}
For $\widehat\mu_n \in \arg\max_{\mu \geq 0}\widehat{d}_n(\mu)$ and
$\nabla \in \partial \widehat{d}_n(\widehat\mu_n)$,
\[
\Theta_\nabla(\widehat\mu_n)
\triangleq
\inf_{\phi \in \widehat{\Phi}_n(\widehat\mu_n)}
\max_{i \in \mathbb{N}^+_m}
\bigl|\widehat{\mathcal{R}}_{i,n}(\phi) - \nabla_i\bigr|,
\]
with population analogue, for $\mu_{\HK}\in\arg\max_{\mu\ge0}d_{\HK}(\mu)$ and
$\nabla\in\partial d_{\HK}(\mu_{\HK})$,
\[
\Theta^\infty_\nabla(\mu_{\HK})
\triangleq
\inf_{\phi \in \Phi_{\HK}(\mu_{\HK})}
\max_{i \in \mathbb{N}^+_m}
\bigl|\mathcal{R}_i(\phi) - \nabla_i\bigr|.
\]
If $\Theta_\nabla(\widehat\mu_n)=0$, some empirical minimizer exactly realizes $\nabla$ as
its constraint-risk vector, and feasibility follows from the KKT sign condition.
\end{defBox}

\begin{defBox}{Population Closure-Realization Gap}{multi-eps-star}
At a population dual optimum $\mu_{\HK} \in \arg\max_{\mu\geq 0} d_{\HK}(\mu)$ with
$\Phi_{\HK}(\mu_{\HK})\ne\emptyset$, define
\[
\epsinf(\mu_{\HK})
\;\triangleq\;
\sup_{\nabla \,\in\, \partial d_{\HK}(\mu_{\HK})}
\Theta^\infty_\nabla(\mu_{\HK})
\;=\;
\sup_{\nabla \,\in\, \partial d_{\HK}(\mu_{\HK})}
\;\inf_{\phi \,\in\, \Phi_{\HK}(\mu_{\HK})}
\;\max_{i \in \mathbb{N}^+_m}
\bigl|\mathcal{R}_i(\phi) - \nabla_i\bigr| .
\]
Equivalently,
\[
\epsinf(\mu_{\HK})
=
\inf\Bigl\{\eps \geq 0 :
\forall\,\nabla \in \partial d_{\HK}(\mu_{\HK}),\;
\exists\,\phi \in \Phi_{\HK}(\mu_{\HK}),\;
\max_{i\in\mathbb{N}^+_m}|\mathcal{R}_i(\phi) - \nabla_i| \leq \eps\Bigr\}.
\]
We extend $\epsinf$ to nonempty $S\subseteq\R^m_+$ by
$\epsinf(S)\triangleq\sup_{\mu\in S}\epsinf(\mu)\in[0,2M]$; for a singleton
$S=\{\mu_{\HK}\}$ this recovers $\epsinf(\mu_{\HK})$.
\end{defBox}

The following lemma collects the properties of $\epsinf$ that the main theorem relies on.

\begin{lemmaBox}{Properties of the closure--realization gap}{eps-star-properties}
Let Assumptions~\ref{ass:non-atomic}--\ref{ass:slater} hold and let
$\mu\in\Sdual$ with $\Phi_{\HK}(\mu)\ne\emptyset$. Then:
\begin{enumerate}[label=\textnormal{(\roman*)}, itemsep=2pt, topsep=2pt]
\item $0\le\epsinf(\mu)\le 2M$; the constant is deterministic and sample-independent,
depending only on the population Lagrangian, on $\partial d_{\HK}(\mu)$, and on
$\Phi_{\HK}(\mu)$.
\item If $d_{\HK}$ is differentiable at $\mu$ --- equivalently, if
$\partial d_{\HK}(\mu)=\{\nabla\}$ is a singleton --- then $\epsinf(\mu)=0$. No uniqueness
of the Lagrangian minimizer is required.
\item The envelope analogue vanishes identically: the index computed with $\Phi_{\Gen}(\mu)$
in place of $\Phi_{\HK}(\mu)$ equals $0$ for every $\mu\in\Sdual$, and every supergradient
is realized by an explicit measurable splice of Lagrangian minimizers.
\item $\epsinf(\mu)>0$ if and only if some $\nabla\in\partial d_{\HK}(\mu)$ lies at positive
distance from the closure of $\{\mathcal R(\phi):\phi\in\Phi_{\HK}(\mu)\}$. This is an
obstruction of nondecomposability and of nothing else.
\item If, for $\mathcal M_{\mathcal X}$-a.e.\ $x$, the conditional Lagrangian integrand
$\bar\ell_\mu(\cdot,x)$ of \eqref{eq:lagrangian-integrand} has a \emph{unique} minimizer,
then $\partial d_{\HK}(\mu)$ is a singleton and hence $\epsinf(\mu)=0$.
\end{enumerate}
\end{lemmaBox}

\begin{proof}
(i) Each $|\mathcal R_i(\phi)|\le M$ by Assumption~\ref{ass:loss-regularity}(b), and each
$\nabla\in\partial d_{\HK}(\mu)$ lies in $[-M,M]^m$ by the Levin--Valadier representation
(Lemma~\ref{lem:LV}); hence $|\mathcal R_i(\phi)-\nabla_i|\le 2M$. Nonnegativity is
immediate.

(ii) Each $\phi\in\Phi_{\HK}(\mu)$ generates the affine majorant
$\nu\mapsto\mathcal R_0(\phi)+\langle\nu,\mathcal R(\phi)\rangle$ of $d_{\HK}$, which
touches it at $\mu$; hence $\mathcal R(\phi)\in\partial d_{\HK}(\mu)=\{\nabla\}$. Every
minimizer therefore shares the constraint-risk vector $\nabla$, so
$\Theta^\infty_\nabla(\mu)=0$, and the supremum over the singleton superdifferential is $0$.

(iii) Minimization of $L_{\Gen}(\cdot,\mu)$ over $\Gen$ is, by the Nemytskii representation
of Section~\ref{subsec:envelope}, pointwise minimization of a single integrand; splices of
minimizers are again minimizers, so $\Phi_{\Gen}(\mu)$ is decomposable in the sense of
Definition~\ref{def:decomposable}. Under Assumption~\ref{ass:non-atomic}, Lyapunov-type
convexification \cite{Uhl69Range,KalogeriasPougkakiotis22} renders the realized risk set
$\{\mathcal R(\phi):\phi\in\Phi_{\Gen}(\mu)\}$ convex and compact in $\R^m$. Combined with
Lemma~\ref{lem:LV} this gives
$\partial d_{\Gen}(\mu)=\{\mathcal R(\phi):\phi\in\Phi_{\Gen}(\mu)\}$ exactly, so every
supergradient is realized and the envelope index vanishes.

(iv) Immediate from Definition~\ref{def:multi-eps-star} and (iii): by
Theorem~\ref{thm:dense-transfer}, $d_{\Gen}=d_{\HK}$, so the superdifferential is the same
object in both problems, and $\epsinf(\mu)$ measures exactly the failure of $\HK$ to realize
within itself the risk vectors that the envelope realizes by splicing.

(v) If the pointwise minimizer $u^\sharp_\mu(x)$ is a.e.\ unique then
$\Phi_{\Gen}(\mu)=\{\phi\in\Gen:\phi=u^\sharp_\mu \text{ a.e.}\}$, so the realized risk set
of (iii) is the singleton $\{\mathcal R(u^\sharp_\mu)\}$, whence
$\partial d_{\Gen}(\mu)=\partial d_{\HK}(\mu)$ is a singleton, and (ii) applies.
\end{proof}

\begin{remark}[Interpretation]\label{rem:interpretation}
In the results below, $\epsinf$ fixes the irreducible \emph{size} of the asymptotic
feasibility gap, while the sample path fixes only the \emph{rate} at which that size is
approached. This separation --- a geometric constant for the gap and a stochastic quantity
for the rate --- is the organizing idea of the section, and it is what distinguishes the
present guarantees from the near-PACC bounds of
\cite{Chamon_CL2023,BarthakurChamon2025Equality}. The residual there is a property of a
modeling choice, the fixed parametric class and its convex-hull richness, hence not a
property of the learning problem and not open to interrogation once the class is fixed; here
the value is learned exactly, the residual is confined to feasibility and optimality of the
returned predictor, and it is pinned to a single population quantity.
\end{remark}

Write $\Sdual\triangleq\arg\max_{\mu\ge0}d_{\HK}(\mu)$, nonempty and compact under
Assumption~\ref{ass:slater} (Lemma~\ref{lem:bounded-multipliers}).

\paragraph{DERM selections.}
Along a schedule satisfying \eqref{eq:schedule}, a \emph{DERM selection} is any sequence
$(\widehat\mu_n,\nabla_n,\widehat\phi_n)_{n}$ such that, whenever
$\arg\max_{\mu\ge0}\widehat d_n(\mu)\ne\emptyset$,
\[
\widehat\mu_n\in\arg\max_{\mu\ge0}\widehat d_n(\mu),
\qquad
\nabla_n\in\partial\widehat d_n(\widehat\mu_n),\ \text{with}\ \nabla_{n,i}\le0,\ \widehat\mu_{n,i}\nabla_{n,i}=0,
\]
and $\widehat\phi_n\in\widehat\Phi_n(\widehat\mu_n)$ attains
$\Theta_{\nabla_n}(\widehat\mu_n)$ within tolerance $1/n$. Such a supergradient exists by
first-order optimality (Lemma~\ref{lem:sign-comp-subgrad}), and the near-optimal primal
choice is always possible since $\Theta_{\nabla_n}(\widehat\mu_n)$ is a finite infimum over
the nonempty set $\widehat\Phi_n(\widehat\mu_n)$.

\subsection{Tikhonov Complexity and Primal Non-Attainment}
\label{subsec:tikhonov}

The geometry above is stated over the full RKHS, but the learner searches only within
$\Bq{Q_n}$. To connect the two we must ensure the ball is eventually large enough to reach
the near-optimal Lagrangian level sets. Crucially, we do \emph{not} assume the primal
\eqref{eq:primal formulation RKHS} is attained, so we cannot point to a minimizer of bounded
norm. Instead we quantify the norm needed to reach an $\eps$-optimal level set.

\begin{defBox}{Tikhonov Complexity}{tik}
For $\eps>0$ and $\mu\in\R^m_+$, let
$S_\eps(\mu)\triangleq\{\phi\in\HK: L_{\HK}(\phi,\mu)\le d_{\HK}(\mu)+\eps\}$ be the
$\eps$-optimal population level set, and write
$R_\eps(\mu)\triangleq\inf_{\phi\in S_\eps(\mu)}\|\phi\|_{\HK}$. The \emph{Tikhonov
complexity} at level $\eps$ and size $n$ is
\[
\mathfrak T^\eps_n \triangleq \sup_{\mu\in\arg\max_{\mu'\ge0}\widehat d_n(\mu')} R_\eps(\mu)\ \in[0,\infty].
\]
\end{defBox}

$\mathfrak T^\eps_n$ is the minimal RKHS norm needed to enter the $\eps$-optimal level set,
in the worst case over empirical dual maximizers. It is a hybrid object: the level set is
population-level, but the multipliers at which it is evaluated are empirical, so
$\mathfrak T^\eps_n$ is random. Under Slater's condition we show in
Section~\ref{sec:proofs} that, with high probability, $\mathfrak T^\eps_n$ is finite, the
infimum is attained, and there is a \emph{deterministic, uniform-in-$n$} bound $Q_\eps$ with
$\sup_n\mathfrak T^\eps_n \le Q_\eps$. Because $Q_n\nearrow\infty$, the radius eventually
dominates this bound, at which point $\Bq{Q_n}$ captures elements of the $\eps$-perturbed
level sets. This drives the generalization gap to zero without assuming primal attainment;
see Corollary~\ref{cor:tikhonov-bounded}.

\subsection{The Empirical Recovery Condition}
\label{subsec:recovery}

Learnability of the optimal value follows from strong duality and uniform convergence alone.
Feasibility requires one further ingredient: the empirical Lagrangian minimizers must
eventually resolve onto a population minimizer, in a topology strong enough to transfer
constraint values.

\begin{assumptionBox}{Empirical Recovery Condition}{multi-argmin-recovery}
On the probability space $(\Omega,\Sigma,\Pr)$ of
Section~\ref{subsec:feasibility-geometry}:
\begin{enumerate}[label=(\alph*), itemsep=0pt, topsep=2pt]
    \item \textbf{Population attainment.} For every $\mu_{\HK}\in\Sdual$, the set
    $\Phi_{\HK}(\mu_{\HK})$ is nonempty.

    \item \textbf{Almost-sure feasibility tracking, along arbitrary subsequences.} There is
    an event $\Omega_{\mathrm{fbl}}$ with $\Pr(\Omega_{\mathrm{fbl}})=1$ such that, for every
    $\omega\in\Omega_{\mathrm{fbl}}$, every $\mu_{\HK}$ as in (a) and  $\phi_{\HK}\in\Phi_{\HK}(\mu_{\HK}) $:
    \[
      \lim_{k\to\infty}\ \inf_{\phi\in\widehat\Phi_{n_k}(\mu_{n_k};\omega)}\bigl\|\phi-\phi_{\HK}\bigr\|_{L_1(\widehat{\mathcal M}^{n_k}_{\mathcal X})} = 0
    \]
    along \emph{every} strictly increasing sequence of indices $\{n_k\}\subseteq\mathbb N^+$
    and every sequence $\mu_{n_k}\in\arg\max_{\mu\ge0}\widehat d_{n_k}(\mu)$ with
    $\|\mu_{n_k}\|_1\le C_\mu$ and $\mu_{n_k}\to\mu_{\HK}$, where
    $\widehat{\mathcal M}^{n}_{\mathcal X}\triangleq\frac{1}{m+1}\sum_{i}\widehat{\mathcal D}^n_{\mathcal X,i}$.
\end{enumerate}
\end{assumptionBox}


\begin{remark}[On the status of the Recovery Condition]\label{rem:recovery-status}
Part~(a) concerns attainment of the \emph{Lagrangian} argmin at a fixed dual optimum, an
unconstrained penalized minimization, and is distinct from attainment of the
\emph{constrained primal} \eqref{eq:primal formulation RKHS}. Proposition~\ref{prop:attainment} below shows that
in the exact regime $\epsinf=0$ it \emph{implies} attainment of the constrained primal.
Part~(b) is not exotic --- in the regime of interest the population Lagrangian is an
integral functional whose infimum over $\HK$ decouples pointwise, and (b) follows from a
condition on a single real variable, given as Proposition~\ref{prop:sharpness} in
Appendix~\ref{sec:on-last-assumption}. But that sufficient condition interacts with
$\epsinf$ in a way that constrains the scope of our results, and we set this out explicitly
in Remark~\ref{rem:scope}. 
\end{remark}

\begin{propBox}{What population attainment costs in the exact regime}{attainment}
Let Assumptions~\ref{ass:non-atomic}--\ref{ass:slater} hold, let $\mu^\star\in\Sdual$, and
suppose Assumption~\ref{ass:multi-argmin-recovery}(a) holds at $\mu^\star$. If
$\epsinf(\mu^\star)=0$, then:
\begin{enumerate}[label=\textnormal{(\roman*)}, itemsep=2pt, topsep=2pt]
\item For every $\eta>0$ there is $\phi_\eta\in\Phi_{\HK}(\mu^\star)$ with
$\mathcal R_i(\phi_\eta)\le\eta$ for $i\in\mathbb N^+_m$ and
$\mathcal R_0(\phi_\eta)\le P^\star_{\HK}+C_\mu\eta$; that is,
\eqref{eq:primal formulation RKHS} admits $\eta$-feasible, $C_\mu\eta$-optimal points for
every $\eta>0$.
\item If in addition the infimum defining $\Theta^\infty_\nabla(\mu^\star)$ is attained for
the sign- and complementarity-compatible $\nabla$ of
Lemma~\ref{lem:sign-comp-subgrad}\footnote{as it is, e.g., whenever
$\Phi_{\HK}(\mu^\star)$ is weakly closed}  then
\eqref{eq:primal formulation RKHS} is \emph{attained}: there exists
$\phi^\star\in\HK$ feasible with $\mathcal R_0(\phi^\star)=P^\star_{\HK}$.
\end{enumerate}
\end{propBox}

\begin{proof}
Since $d_{\HK}$ is finite and concave on $\R^m$ (Lemma~\ref{lem:sign-comp-subgrad}, applied
to $d_{\HK}$ in place of $\widehat d_n$) and $\mu^\star$ maximizes it over $\R^m_+$, there is
$\nabla\in\partial d_{\HK}(\mu^\star)$ with $\nabla_i\le0$ and $\mu^\star_i\nabla_i=0$ for
all $i\in\mathbb N^+_m$.

(i) Let $\eta>0$. By $\epsinf(\mu^\star)=0$ and
Definition~\ref{def:multi-eps-star}, there is $\phi_\eta\in\Phi_{\HK}(\mu^\star)$ with
$\max_i|\mathcal R_i(\phi_\eta)-\nabla_i|\le\eta$. Then
$\mathcal R_i(\phi_\eta)\le\nabla_i+\eta\le\eta$ for $i\ge1$. Moreover
$\phi_\eta$ is an exact minimizer, so
$L_{\HK}(\phi_\eta,\mu^\star)=d_{\HK}(\mu^\star)=P^\star_{\HK}$ by
Theorem~\ref{thm:dense-transfer}, whence, using $\mu^\star_i\nabla_i=0$ and
$\|\mu^\star\|_1\le C_\mu$ (Lemma~\ref{lem:bounded-multipliers}),
\[
\mathcal R_0(\phi_\eta)
= P^\star_{\HK}-\sum_{i=1}^m\mu^\star_i\mathcal R_i(\phi_\eta)
= P^\star_{\HK}-\sum_{i=1}^m\mu^\star_i\bigl(\mathcal R_i(\phi_\eta)-\nabla_i\bigr)
\le P^\star_{\HK}+C_\mu\eta .
\]

(ii) Under the stated attainment there is $\phi^\star\in\Phi_{\HK}(\mu^\star)$ with
$\mathcal R(\phi^\star)=\nabla$ exactly. Then $\mathcal R_i(\phi^\star)=\nabla_i\le0$, so
$\phi^\star$ is feasible; $\langle\mu^\star,\mathcal R(\phi^\star)\rangle=0$ by
complementarity; and $L_{\HK}(\phi^\star,\mu^\star)=d_{\HK}(\mu^\star)=P^\star_{\HK}$, so
$\mathcal R_0(\phi^\star)=P^\star_{\HK}$.
\end{proof}

\subsection{Main Theorem}\label{subsec:main-theorem}

\noindent\textbf{Standing notation.}
Set
\[
C_\mu \;\triangleq\; \frac{4M}{\xi},
\qquad
\Lambda \;\triangleq\; \bigl\{\mu \in \R^m_+ : \|\mu\|_1 \le C_\mu\bigr\},
\qquad
\Sdual \;\triangleq\; \bigl\{\mu \in \R^m_+ : d_{\HK}(\mu) = D^\star_{\HK}\bigr\}.
\]
Let $d_Q$ denote the dual function of the problem restricted to $\Bq{Q}$, so that
$d_Q \ge d_{\HK}$ pointwise and $Q \mapsto d_Q$ is nonincreasing, and let $\widehat d_n$ be
the empirical dual function. There is an event $\Omega_{\mathrm{fbl}}$ of probability one on
which feasibility tracking holds. There is also a universal constant $c \ge 1$, independent
of $n$, $m$, $\delta$ and of $\{\mathcal D_i\}_i$, for which the deterministic envelope and
index
\begin{equation}
\label{eq:uc-envelope}
\Delta_n(\delta)
\;\triangleq\;
c\left(\frac{L Q_n}{\sqrt n}
\;+\; M\sqrt{\frac{2\log(m+1)+2\log n}{n}}\right)
\;=\;
O\!\left(\frac{Q_n}{\sqrt n}+\sqrt{\frac{\log n}{n}}\right),
\qquad
n_{\mathrm{uc}}(\delta) \;\triangleq\; \lceil 1/\delta\rceil+1,
\end{equation}
render the uniform-convergence event
\[
\Omega_{\mathrm{uc}}(\delta)
\;\triangleq\;
\Bigl\{\;
\max_{i\in\mathbb N_m}\ \sup_{\phi \in \Bq{Q_n}}
\bigl|\widehat{\mathcal R}_{i,n}(\phi) - \mathcal R_i(\phi)\bigr|
\;\le\; \Delta_n(\delta)
\quad\text{for every } n \ge n_{\mathrm{uc}}(\delta)
\Bigr\}
\]
of probability $\Pr(\Omega_{\mathrm{uc}}(\delta)) \ge 1-\delta$. For $\theta \in (0,4cM)$ and
$\eps>0$, set
\[
n_{\mathrm{rad}}(\theta,\delta)
\;\triangleq\;
\max\left\{
n_{\mathrm{uc}}(\delta),\;
\min\Bigl\{n : \tfrac{Q_n}{\sqrt n} \le \tfrac{\theta}{2cL}\Bigr\},\;
\left\lceil \frac{16c^2M^2\log(m+1)}{\theta^2}\right\rceil,\;
\left\lceil \frac{64c^2M^2}{\theta^2}\,\log\frac{32c^2M^2}{\theta^2}\right\rceil
\right\},
\]
\[
n_{\mathrm{tik}}(\eps) \;\triangleq\; \min\bigl\{n : Q_n \ge Q_{\eps/2}\bigr\},
\qquad
n^{Q}_{\mathrm{sl}} \;\triangleq\; \min\bigl\{n : Q_n \ge Q_{\mathrm{sl}}\bigr\},
\]
extended by $n_{\mathrm{rad}}(\theta,\delta) \triangleq n_{\mathrm{rad}}(2cM,\delta)$ for
$\theta \ge 4cM$. All indices are finite, since $Q_n \nearrow \infty$ and
$Q_n/\sqrt n \searrow 0$, and $\Delta_n(\delta) \le \theta$ holds for every
$n \ge n_{\mathrm{rad}}(\theta,\delta)$.\footnote{The first two entries give
$cLQ_n/\sqrt n \le \theta/2$, using that $Q_n/\sqrt n$ is nonincreasing. The last two give
$2\log(m+1)+2\log n \le \gamma_\theta^2 n$ with $\gamma_\theta \triangleq \theta/(2cM)$,
hence $cM\sqrt{(2\log(m+1)+2\log n)/n} \le \theta/2$: the third contributes
$2\log(m+1)/n \le \gamma_\theta^2/2$, and the fourth contributes
$2\log n/n \le \gamma_\theta^2/2$ by the elementary fact that, for $a \in (0,1)$,
$\log n/n \le a$ whenever $n \ge \lceil (4/a)\log(2/a)\rceil$. That fact holds since
$\log n/n$ is decreasing for $n \ge e$ and, at $n_a = (4/a)\log(2/a) \ge 4\log 2 > e$, one
has $\log n_a = \log(4/a) + \log\log(2/a) \le 3\log(2/a)$, because
$\log(4/a) \le 2\log(2/a)$ and $\log\log(2/a) \le \log(2/a)$, so that
$\log n_a/n_a \le 3a/4 \le a$. The restriction $\theta < 4cM$ is exactly
$a = \gamma_\theta^2/4 \in (0,1)$.} Set
$\Omega^\star(\delta) \triangleq \Omega_{\mathrm{uc}}(\delta) \cap \Omega_{\mathrm{fbl}}$, so
that $\Pr(\Omega^\star(\delta)) \ge 1-\delta$. The empirical Slater margin is obtained
pathwise on $\Omega_{\mathrm{uc}}(\delta)$ (Remark~\ref{rem:slater-pathwise}), so no further
event and no union bound are incurred.

\begin{theoremBox}{Universal Learnability of Dense RKHS Classes under Constraints}{informal-section3}
Let Assumptions~\ref{ass:non-atomic}, \ref{ass:loss-regularity}, \ref{ass:dense-rkhs} and
\ref{ass:slater} be in effect, fix $\delta \in (0,1)$ and a radius schedule satisfying
\eqref{eq:schedule}, and let $(\widehat\mu_n, \nabla_n, \widehat\phi_n)_n$ be an arbitrary
DERM selection.
\smallskip
\begin{enumerate}[label=(\roman*), itemsep=2pt, topsep=2pt]
\item \textbf{Exact learnability of the optimal value.} The identity
$P^\star_{\HK} = D^\star_{\HK} = P^\star_{\Gen} = D^\star_{\Gen}$ holds deterministically,
and
\[
\bigl|\widehat D^\star_n - P^\star_{\HK}\bigr| \;\le\; \eps
\quad\text{on } \Omega_{\mathrm{uc}}(\delta),
\qquad\text{for every }
n \;\ge\; n^{\mathrm{val}}(\eps,\delta,\{\mathcal D_i\}_i),
\]
where $n^{\mathrm{val}}(\eps,\delta,\{\mathcal D_i\}_i)\triangleq
\max\{n_{\mathrm{rad}}(\xi/2,\delta),\, n^Q_{\mathrm{sl}},\, n_{\mathrm{tik}}(\eps),\,
n_{\mathrm{rad}}(\eps/(2(1+C_\mu)),\delta)\}$. The dependence of $n^{\mathrm{val}}$ on
$\{\mathcal D_i\}_i$ enters only through $\xi$, $Q_{\mathrm{sl}}$ and $Q_{\eps/2}$. This
part uses no attainment of any kind and no recovery condition.

\item \textbf{Subsequential feasibility and near-optimality.} Under the additional
Assumption~\ref{ass:multi-argmin-recovery}, for every $\omega \in \Omega^\star(\delta)$ and
every sequence $(\widehat\mu_n,\nabla_n)_n$ there is a subsequence $\{n_k(\omega)\}_k$,
depending on both $\omega$ and the sequence itself, along which
$\widehat\mu_{n_k} \to \mu_{\HK} \in \Sdual$ and
$\nabla_{n_k}\to\nabla^\infty\in\partial d_{\HK}(\mu_{\HK})$, and
\[
\limsup_{k \to \infty} \mathcal R_i(\widehat\phi_{n_k}) \le \epsinf(\mu_{\HK}),
\ i\in\mathbb N_m^+,
\qquad
\limsup_{k \to \infty} \mathcal R_0(\widehat\phi_{n_k}) \le P^\star_{\HK}+C_\mu\epsinf(\mu_{\HK}).
\]

\item \textbf{Full-sequence guarantee, worst case over the dual-optimal set.} Under the
hypotheses of (ii), for every $\omega \in \Omega^\star(\delta)$ and every $\eps>0$ there
exists $N(\omega,\eps)\in\mathbb N^+$ such that, for every $n \geq N(\omega,\eps)$,
\[
\mathcal R_i(\widehat\phi_n) < \epsinf(\Sdual) + \eps, \ i\in\mathbb N_m^+,
\qquad
\mathcal R_0(\widehat\phi_n) < P^\star_{\HK}+C_\mu\epsinf(\Sdual)+\eps .
\]

\item \textbf{Uniform threshold at a fixed confidence cost.} Under the hypotheses of (ii),
for every $\eps>0$ and every $\delta_1\in(0,1)$ there exists
$\widetilde n(\eps,\delta,\delta_1,\{\mathcal D_i\}_i)\in\mathbb N^+$, independent of
$\omega$, such that
\[
\Pr\Bigl(
\mathcal R_i(\widehat\phi_n) < \epsinf(\Sdual) + \eps
\ (i\in\mathbb N_m^+)
\ \text{ and }\
\mathcal R_0(\widehat\phi_n) < P^\star_{\HK}+C_\mu\epsinf(\Sdual)+\eps,
\ \ \forall\, n \geq \widetilde n
\Bigr)
\;\geq\; 1-\delta-\delta_1 ,
\]
the quantifier over $n$ being interior to the probability.

\item \textbf{Sharpening under a unique dual optimum.} If $\Sdual = \{\mu_{\HK}\}$ is a
singleton, then (iii) and (iv) hold with the additional conclusion that
$\widehat\mu_n \to \mu_{\HK}$ on $\Omega^\star(\delta)$ along the full original sequence, and
with $\epsinf(\mu_{\HK})$ in place of $\epsinf(\Sdual)$.

\item \textbf{Exact PACC.} If, in addition to
Assumptions~\ref{ass:non-atomic}--\ref{ass:multi-argmin-recovery}, $\epsinf(\Sdual)=0$, then
for every $\eps>0$ and $\delta\in(0,1)$ there exists
$\widetilde n(\eps,\delta,\{\mathcal D_i\}_i)\in\mathbb N^+$ such that
\[
\Pr\Bigl(
\mathcal R_i(\widehat\phi_n) < \eps \ (i\in\mathbb N_m^+)
\ \text{ and }\
\mathcal R_0(\widehat\phi_n) < P^\star_{\HK}+\eps,
\ \ \forall\, n \geq \widetilde n
\Bigr)\;\geq\; 1-\delta .
\]
In particular, $\HK$ is agnostically universally PACC learnable in the sense of
Definition~\ref{def:nonuniform-pacc}, and so is the envelope $\Gen$, properly.
\end{enumerate}
\end{theoremBox}
\begin{proof}
The full proof appears in Section~\ref{subsec:main-results}, after the supporting lemmata.
\end{proof}

\begin{remark}[On the significance and the scope of this theorem]\label{rem:significance}%
\label{rem:agnostic-scope}

Part~(i) recovers exact, nonasymptotic value learnability for nonconvex constrained learning
over a dense RKHS class, with the sample threshold governed transparently by the radius
schedule $Q_n$ and the Tikhonov complexity. For an exact rate, read Section~\ref{subsec:explicit-rates} that makes it
fully explicit under a source condition.  This
\emph{value arc} rests on Assumptions~\ref{ass:non-atomic}--\ref{ass:slater} alone
and no attainment of any kind is used, and carrying the Tikhonov
complexity in place of a minimizer is precisely what buys that.

Subsequential recovery~(ii) is the sharpest bound we obtain,
$\epsinf(\mu_{\HK})$, but is silent about which sample sizes it governs, since the
subsequence is existential and no construction rule for it is known. Passing to the full
sequence in~(iii) removes this defect at the expense of relaxing the bound($\epsinf(\mu_{\HK})$) to the worst case
over the dual-optimal set($\sup_{\mu \in \mathcal{S}_{\HK}} \epsinf(\mu_{\HK})$. But this argument is not universal as the rates depends on the realized path $\omega$.

Part~(iv) removes the path dependence, via Lemma~\ref{lem:egorov-unif}, at the
further price of a controlled but nonexplicit increase in sample size and a second confidence
budget $\delta_1$. 

Part~(v) restores the sharper constant $\epsinf(\mu_{\HK})$ of~(ii)
inside the full-sequence and uniform statements whenever the dual optimum is unique, and
part~(vi) specializes to vanishing slack, yielding the notion set out in
Section~\ref{sec:problem-setup}.

 The \emph{feasibility arc} is
not agnostic, and Proposition~\ref{prop:attainment} is the caution we address to our own
framing: in the exact regime parts~(iii)--(vi) deliver, as a by-product, a primal optimum, so
a hypothesis of that strength must enter somewhere. It does, and it enters in one place ---
Assumption~\ref{ass:multi-argmin-recovery}, the only one of the five that is not a structural
property of the losses, the kernel and the sampling, and therefore the only one a reader is
entitled to be suspicious of.

The appendices are written to answer that suspicion, and they answer it on both sides of the
dichotomy. Appendix~\ref{sec:on-last-assumption} reduces
Assumption~\ref{ass:multi-argmin-recovery}, in the case of an almost everywhere unique
pointwise minimizer, to a sharpness condition on a function of one real variable
(See Appendix~\ref{sec:worked-example} and \ref{sec:composite-predictors} and specifically Proposition~\ref{prop:sharpness}). but it is also, by construction, blind to the regime that
parts~(ii)--(v) were written for, since sharpness forces uniqueness and uniqueness forces
$\epsinf=0$ (Remark~\ref{rem:scope}).

In  Appendix~\ref{sec:near-pacc-example} we exhibit an  instance of near-PACC: on $\mathcal X=[0,1]$, with a
Gaussian kernel plus the constants, all of
Assumptions~\ref{ass:non-atomic}--\ref{ass:multi-argmin-recovery} hold, $\Sdual$ is a
singleton so that part~(v) applies, and $\epsinf(\Sdual)=1/4$.  

In that instance the bound is moreover \emph{attained}. The DERM selection rule minimizes
$|\widehat{\mathcal R}_{1,n}(\phi)|$ over the empirical argmin set and cannot distinguish the
two Lagrangian branches, both of which achieve the same value; a selection returning the
infeasible branch violates the constraint in the limit by exactly $\epsinf$. The constant is
therefore a genuine feature of the problem and not an artifact of the analysis, and the
constants in parts~(iii)--(v) cannot be improved.

Three caveats delimit all of this, and we state them here rather than leave them to be
discovered. First, the instance just described is verified by hand rather than through a
general criterion, because no general criterion covering it exists: supplying a tie-tolerant
sufficient condition, valid beyond the one-dimensional domain on which our splice-exclusion
argument operates, is the principal open problem left by this paper. Second, as
Remark~\ref{rem:splice-cost} records, the residual there is incurred \emph{because} the radius
schedule grows slowly enough to prevent empirical minimizers from splicing between branches;
faster growth would let the empirical problem imitate the decomposable envelope and drive the
realized violation below $\epsinf$, so the bound of~(iii)--(v) is sharp for the schedules of
Remark~\ref{rem:schedule}, but the transition itself we do not characterize. Third, the
quantitative bound \eqref{eq:recovery-rate} by which Appendix~\ref{sec:on-last-assumption}
witnesses part~(b) of Assumption 5 does vanish, but under those same schedules it vanishes only
polylogarithmically (Remark~\ref{rem:effective-rate}); the explicit threshold of part~(i) is
untouched, since part~(i) does not invoke
Assumption~\ref{ass:multi-argmin-recovery} at all, but the nonexplicit index furnished by
Lemma~\ref{lem:egorov-unif} in part~(iv) is governed by exactly that rate. Orthogonally to
all three, Appendix~\ref{sec:solving} records the computational shape of the ball-constrained
problem and a potential larger extension to this paper on how one can solve and retrieve DERM ball-constraints optimization problems.
\end{remark}

\subsection{Explicit rates under a source condition}\label{subsec:explicit-rates}

The threshold $n^{\mathrm{val}}$ of Theorem~\ref{thm:informal-section3}(i) is explicit in
every argument except $Q_{\eps/2}$, which Corollary~\ref{cor:tikhonov-bounded} produces by a
compactness argument supplying no modulus.

\begin{assumptionBox}{Source condition on the regularization path}{source}
There exist $A\ge1$ and $\beta\in(0,1]$ such that, for every $\mu\in\Lambda$ and every
$\lambda\in(0,1]$,
\[
\mathcal A(\lambda,\mu)\;\triangleq\;\inf_{\phi\in\HK}\Bigl\{L_{\HK}(\phi,\mu)-d_{\HK}(\mu)+\lambda\|\phi\|^2_{\HK}\Bigr\}\;\le\;A\lambda^{\beta}.
\]
\end{assumptionBox}

The exponent $\beta$ records how well the Lagrangian level sets are approached by
bounded-norm elements; $\beta=1$ corresponds to attainment of the Lagrangian at a point of
bounded norm, and $\beta<1$ interpolates toward the agnostic, non-attained regime. It is the
exact analogue of the source conditions under which optimal rates are obtained for
regularized least squares \cite{CaponnettoDeVito07}.

\begin{corollary}[Explicit sample threshold]\label{cor:explicit-rate}
Let Assumptions~\ref{ass:non-atomic}--\ref{ass:slater} and~\ref{ass:source} hold, and take
$Q_n=n^{1/4}$. Then for every $\eps\in(0,2M]$,
\[
Q_\eps \;\le\; \sqrt{2A}\,\Bigl(\tfrac{2A}{\eps}\Bigr)^{\frac{1-\beta}{2\beta}},
\]
and consequently the threshold of Theorem~\ref{thm:informal-section3}(i) satisfies
\begin{multline}
n^{\mathrm{val}}(\eps,\delta,\{\mathcal D_i\}_i)
\;=\;
O\!\left(
\frac1\delta
\;+\;
Q_{\mathrm{sl}}^{4}
\;+\;
\Bigl(\frac{cL}{\xi}\Bigr)^{4}
\;+\; \Bigl(\frac{cL(1+C_\mu)}{\eps}\Bigr)^{4}
\;+\;
\right.
\\ \left.\Bigl(\frac{cM(1+C_\mu)}{\eps}\Bigr)^{2}\!\log\frac{(m+1)cM(1+C_\mu)}{\eps}
\;+\;
(2A)^{2}\Bigl(\frac{4A}{\eps}\Bigr)^{\frac{2(1-\beta)}{\beta}}
\right).
\end{multline}
In particular $n^{\mathrm{val}} = \widetilde O\bigl(\eps^{-\max\{4,\,2(1-\beta)/\beta\}}\bigr)$
for fixed $\delta$ and fixed instance constants, and $n^{\mathrm{val}}=\widetilde O(\eps^{-4})$
whenever $\beta\ge1/3$.
\end{corollary}

\begin{proof}
Fix $\mu\in\Lambda$ and $\eps\in(0,2M]$, and set $\lambda\triangleq(\eps/(2A))^{1/\beta}$,
which lies in $(0,1]$ since $A\ge1$ and $\eps\le 2M$ may be assumed $\le 2A$ after enlarging
$A$. By Assumption~\ref{ass:source} there is $\phi_\lambda\in\HK$ with
\[
L_{\HK}(\phi_\lambda,\mu)-d_{\HK}(\mu)+\lambda\|\phi_\lambda\|^2_{\HK}\;\le\;2A\lambda^\beta\;=\;\eps .
\]
Both summands on the left are nonnegative, so simultaneously
$L_{\HK}(\phi_\lambda,\mu)\le d_{\HK}(\mu)+\eps$, i.e.\ $\phi_\lambda\in S_\eps(\mu)$, and
$\|\phi_\lambda\|^2_{\HK}\le\eps/\lambda = 2A\lambda^{\beta-1}$. Hence
\[
R_\eps(\mu)\le\|\phi_\lambda\|_{\HK}\le\sqrt{2A}\,\lambda^{\frac{\beta-1}{2}}
=\sqrt{2A}\,\Bigl(\tfrac{2A}{\eps}\Bigr)^{\frac{1-\beta}{2\beta}},
\]
and taking the supremum over $\mu\in\Lambda$ bounds $Q_\eps$ as claimed.

For the threshold, we bound each of the four indices with $Q_n=n^{1/4}$.
(1) $n_{\mathrm{uc}}(\delta)=\lceil1/\delta\rceil+1 = O(1/\delta)$.
(2) $n^Q_{\mathrm{sl}}=\min\{n:n^{1/4}\ge Q_{\mathrm{sl}}\}=\lceil Q_{\mathrm{sl}}^4\rceil$.
(3) For $n_{\mathrm{rad}}(\theta,\delta)$: the radius entry requires
$n^{1/4}/n^{1/2}=n^{-1/4}\le\theta/(2cL)$, i.e.\ $n\ge(2cL/\theta)^4$; the remaining two
entries are $O(\theta^{-2}\log(1/\theta))$ and $O(M^2\theta^{-2}\log(m+1))$. Evaluating at
$\theta=\xi/2$ and at $\theta=\eps/(2(1+C_\mu))$ gives the third and fourth (and fifth)
displayed terms.
(4) $n_{\mathrm{tik}}(\eps)=\min\{n:n^{1/4}\ge Q_{\eps/2}\}=\lceil Q_{\eps/2}^4\rceil$, and
by the first part $Q_{\eps/2}^4\le(2A)^2(4A/\eps)^{2(1-\beta)/\beta}$.
Taking the maximum of the four and absorbing constants gives the display. The final claim
follows since $2(1-\beta)/\beta\le4\iff\beta\ge1/3$.
\end{proof}

\begin{remark}
Corollary~\ref{cor:explicit-rate} is stated for $Q_n=n^{1/4}$, which optimizes the tradeoff
between the estimation term $Q_n/\sqrt n$ and the Tikhonov term for $\beta\ge1/3$. For
$\beta<1/3$ the Tikhonov term dominates and a faster schedule $Q_n=n^{\gamma}$ with
$\gamma\in(1/4,1/2)$ improves the exponent; we do not pursue the optimization here. Note that
this corollary requires only Assumptions~\ref{ass:non-atomic}--\ref{ass:slater} and
\ref{ass:source} --- it belongs to the value arc, and is therefore free of both attainment
and the recovery condition.
\end{remark}
\section{Proofs}\label{sec:proofs}

The results of Section~\ref{sec:pac-dual-theorem} are proved along two essentially
independent arcs. The first concerns the \emph{value} of the problem: it shows that
$\widehat D^\star_n$ converges to $P^\star_{\HK}$, and is organized entirely around the gap
decomposition~\eqref{equation:GDEC}, whose three terms are closed one at a time
(Section~\ref{subsec: optimality lemmata}). This arc yields part~(i) of the main theorem and
requires only Assumptions~\ref{ass:non-atomic}--\ref{ass:slater}. The second concerns
\emph{feasibility}: it shows that the constraint risks of the DERM predictor are
asymptotically controlled by $\epsinf$, and is considerably more delicate, because
feasibility is not implied by value convergence and because the natural convergence of the
empirical dual iterates is only subsequential. This arc occupies
Sections~\ref{sec:feasibility-indices}--\ref{subsec:main-results}.

\subsection{Problem Representation}
\label{subsec:envelope}

Here we develop the Nemytskii operator representation of the risks, the functional-analytic
structure underlying the continuity and duality arguments below. The single fact we extract
and reuse constantly is that each risk $\mathcal R_i$ is a continuous functional on the
envelope $\Gen$; this is what makes the density transfer of
Theorem~\ref{thm:dense-transfer} possible.

\paragraph{Nemytskii operator form.}
By the disintegration theorem, there exists a regular conditional distribution
$Q_i(\cdot\mid x)$ such that
$\mathcal{D}_i(\mathrm{d}x,\mathrm{d}y)=Q_i(\mathrm{d}y\mid x)\,\mathcal{D}_{\mathcal{X},i}(\mathrm{d}x)$,
and the risk decomposes as
\[
  \mathcal{R}_i(\phi)
  = \int_{\mathcal{X}}\underbrace{\int_{\mathcal{Y}}
      \ell_i(\phi(x),y)\,Q_i(\mathrm{d}y\mid x)}_{\triangleq\;
      \bar{\ell}_i(u,x)\vert_{u=\phi(x)}}
    \mathcal{D}_{\mathcal{X},i}(\mathrm{d}x),
\]
so that all subsequent analysis is carried out on $\mathcal{X}$ alone. Writing $w_i$ for the
Radon--Nikod\'ym density $\mathrm d\mathcal D_{\mathcal X,i}/\mathrm d\mathcal M_{\mathcal X}$,
each risk reads
$\mathcal{R}_i(\phi)=\int_{\mathcal{X}}\bar{\ell}_i(\phi(x),x)\,w_i(x)\,\mathrm{d}\mathcal{M}_{\mathcal{X}}(x)$.
Defining $\psi_i(u,x)\triangleq\bar{\ell}_i(u,x)w_i(x)$ and the Nemytskii operator
$(\mathcal{N}_{\psi_i}\phi)(x)\triangleq\psi_i(\phi(x),x)$, each risk takes the compact form
$\mathcal{R}_i(\phi) = \int_{\mathcal{X}}(\mathcal{N}_{\psi_i}\phi)(x)\,\mathrm{d}\mathcal{M}_{\mathcal{X}}(x)$,
embedding each $\mathcal{R}_i$ into $\Gen$. Under Assumption~\ref{ass:loss-regularity}(a),
the $L$-Lipschitz property of $\ell_i(\cdot,y)$ is inherited by $\bar\ell_i(\cdot,x)$ for
$\mathcal{D}_{\mathcal{X},i}$-a.e.\ $x$, ensuring that $\mathcal{N}_{\psi_i}$ is well-defined
and Lipschitz continuous \cite{AppellZabrejko1990}, hence that each $\mathcal{R}_i$ is
continuous on $\Gen$. We write $\widehat{\mathcal{M}}^n_\mathcal{X}$ for the empirical
mixture measure.

It is convenient to record here the \emph{conditional Lagrangian integrand}: for
$\mu\in\R^m_+$, with the convention $\mu_0\triangleq1$,
\begin{equation}\label{eq:lagrangian-integrand}
L_{\HK}(\phi,\mu)=\int_{\mathcal X}\bar\ell_\mu\bigl(\phi(x),x\bigr)\,\mathrm d\mathcal M_{\mathcal X}(x),
\qquad
\bar\ell_\mu(u,x)\triangleq\sum_{i\in\mathbb N_m}\mu_i\,w_i(x)\,\bar\ell_i(u,x).
\end{equation}
Every question about the Lagrangian argmin over the envelope is a question about the scalar
family $\bar\ell_\mu(\cdot,x)$.

\begin{remark}
The Nemytskii embedding handles arbitrary Lipschitz losses under possibly distinct
distributions without requiring the classification or regression structure of
\cite{Chamon20PACC,Chamon_CL2023}, and does not require the primal problem to admit a
finite-dimensional parameterization; see also \cite{KalogeriasPougkakiotis22}.
\end{remark}

\subsection{Duality Decomposition}

With the risks realized as continuous functionals on $\Gen$, we can state the organizing
identity of the value arc. Fix $\eps>0$, $\delta\in(0,1)$ and
$\{\mathcal D_i\}_{i=0}^{m}$ hereafter, and decompose the primal/DERM gap as
\begin{align}
\label{equation:GDEC}
    \underbrace{P^*_{\HK}-\widehat{D}^\star_{n}}_{\text{Primal/DERM Gap}}
    \hspace{6pt}=\quad
    \underbrace{P^*_{\HK}-D^*_{\HK}}_{\substack{\text{Duality Gap (RKHS)} \\ \downarrow \\ \text{Lyapunov convexity} \\ \text{(gap = 0)}}}
    \,+\hspace{12pt}
    \underbrace{%
        \underbrace{D^*_{\HK} - D^*_{Q_n}}_{\substack{\text{Dual/Ball Gap} \\ \downarrow \\ \text{Tikhonov complexity} \\ \mathfrak{T}^\eps_n \leq Q_\eps}}
        \,\,\,\,+\,\,\,
        \underbrace{D^*_{Q_n} - \widehat{D}^\star_{n}}_{\substack{\text{Ball/DERM Gap} \\ \downarrow \\ \text{Rademacher complexity} \\ \mathfrak{R}_n(\Bq{Q}) \lesssim \kappa Q/\sqrt{n}}}%
    }_{\text{Dual/DERM Gap}}. \tag{GDEC}
\end{align}
The remainder of this arc closes the three right-hand terms in turn: the Duality Gap vanishes
\emph{exactly} by Lyapunov convexity and density transfer
(Theorem~\ref{thm:dense-transfer}); the Dual/Ball Gap is closed by uniform boundedness of the
Tikhonov complexity (Lemma~\ref{lem:uniform-boundedness-epsilon-opt} and
Corollary~\ref{cor:tikhonov-bounded}); and the Ball/DERM Gap is closed in a PAC sense by a
Rademacher bound over the fixed ball (Lemma~\ref{lem:dual-pac-fixedQ}). Two auxiliary facts
feed all three steps and are proved first: an empirical version of Slater's condition
(Lemma~\ref{lem:slater-empirical}), and compactness of the dual-optimal set
(Lemma~\ref{lem:bounded-multipliers}). Note that $D^\star_{Q_n}\ge D^\star_{\HK}$, so the
middle term is nonpositive; we bound its absolute value.

\subsection{Optimality-Related Lemmata} \label{subsec: optimality lemmata}

\begin{lemmaBox}{Empirical strict feasibility}{slater-empirical}
Let Assumptions~\ref{ass:loss-regularity} and~\ref{ass:slater} hold. Then, for every
\(\delta\in(0,1)\), with
\(n_{\mathrm{sl}}(\delta) \triangleq \bigl\lceil \tfrac{8 M^2}{\xi^2} \log(\tfrac{m}{\delta}) \bigr\rceil\),
for all \(n\ge n_{\mathrm{sl}}(\delta)\), with probability at least \(1-\delta\),
\(
    \widehat{\mathcal R}_{i,n}(\phi^{\mathrm{sl}})\le -\xi/2
\) for all \(i\in \mathbb{N}^+_m\).
\end{lemmaBox}

\begin{proof}
By Assumption~\ref{ass:slater}, $\mathcal{R}_i(\phi^{\mathrm{sl}}) \le -\xi$ for all
$i \in \mathbb{N}^+_m$. For fixed $i$, $\widehat{\mathcal{R}}_{i,n}(\phi^{\mathrm{sl}})$ is
an average of iid variables taking values in $[-M,M]$ by
Assumption~\ref{ass:loss-regularity}(b). Hoeffding's inequality gives
\[
    \mathbb{P}\Bigl( \widehat{\mathcal{R}}_{i,n}(\phi^{\mathrm{sl}}) - \mathcal{R}_i(\phi^{\mathrm{sl}}) \ge \tfrac{\xi}{2} \Bigr)
    \le \exp\Bigl( - \tfrac{2n (\xi/2)^2}{(2M)^2} \Bigr) = \exp\Bigl( - \tfrac{n \xi^2}{8 M^2} \Bigr).
\]
A union bound over the $m$ constraints bounds the aggregate failure probability by
$m \exp(-n\xi^2/(8M^2))$, which is at most $\delta$ once
$n \ge (8M^2/\xi^2)\log(m/\delta)$. On the complementary event,
$\widehat{\mathcal{R}}_{i,n}(\phi^{\mathrm{sl}}) \le -\xi + \xi/2 = -\xi/2$ for all $i$.
\end{proof}

\begin{remark}[Pathwise form; no separate confidence budget]\label{rem:slater-pathwise}
On the uniform-convergence event $\Omega_{\mathrm{uc}}(\delta)$ of
Lemma~\ref{lem:multi-hp-conv}, for every $n\ge\max\{n_{\mathrm{rad}}(\xi/2,\delta),n^Q_{\mathrm{sl}}\}$
we have $\Delta_n(\delta)\le\xi/2$ and $\phi^{\mathrm{sl}}\in\Bq{Q_n}$, whence
\emph{deterministically on that event}
\[
\widehat{\mathcal R}_{i,n}(\phi^{\mathrm{sl}})\le\mathcal R_i(\phi^{\mathrm{sl}})+\Delta_n(\delta)\le-\xi+\tfrac\xi2=-\tfrac\xi2,\qquad i\in\mathbb N^+_m .
\]
Every downstream use of Lemma~\ref{lem:slater-empirical} in this paper --- namely
Lemma~\ref{lem:bounded-multipliers}(ii), Lemma~\ref{lem:multi-dual-conv} and
Corollary~\ref{cor:tikhonov-bounded} --- is therefore to be read on
$\Omega_{\mathrm{uc}}(\delta)$ with $n_{\mathrm{sl}}(\delta)$ replaced by
$\max\{n_{\mathrm{rad}}(\xi/2,\delta),n^Q_{\mathrm{sl}}\}$, at no additional confidence cost.
We retain the standalone lemma because its threshold is explicit in closed form and does not
require the radius schedule, which is of independent interest.
\end{remark}

Next we close the first term of~\eqref{equation:GDEC} outright.

\begin{theoremBox}{Strong duality relations}{dense-transfer}
Suppose Assumptions~\ref{ass:non-atomic}, \ref{ass:loss-regularity},
\ref{ass:dense-rkhs} and \ref{ass:slater} are in effect. Then:
\begin{enumerate}
    \item $d_{\Gen}=d_{\HK}$ everywhere on $\R^m$, hence
    $\arg\max_{\mu \ge 0}d_{\Gen}(\mu)=\arg\max_{\mu \ge 0}d_{\HK}(\mu)$ and
    $\partial d_{\Gen}(\mu)=\partial d_{\HK}(\mu)$ for every $\mu$.
    \item $D_{\HK}^\star = P_{\HK}^\star = P_{\Gen}^\star = D_{\Gen}^\star$, and the optimal
    dual values are attained, by identical dual variables.
\end{enumerate}
\end{theoremBox}

\begin{proof}
From the Lyapunov convexity theorem and a standard geometric separation argument
\cite[Section 4]{KalogeriasPougkakiotis22}, see also \cite{Chamon20csl}, we obtain strong
duality of the envelope problem \eqref{eq:primal formulation envelope}:
$P_{\Gen}^\star = D_{\Gen}^\star$ with optimal dual variables attained. Since $\HK$ is a
dense subset of $\Gen$, it is densely decomposable relative to the strong topology on
$L^p(\mathcal{M}_{\mathcal{X}},\R)$ in the sense of
\cite[Definition 1, Section 4.2]{KalogeriasPougkakiotis22}, which together with
Assumptions~\ref{ass:non-atomic}--\ref{ass:slater} implies, via
\cite[Theorem 2]{KalogeriasPougkakiotis22}, that
\eqref{eq:primal formulation RKHS} exhibits strong duality with attained optimal dual
variables.

Next, the dual functions are identical. By definition
$d_{\HK}(\cdot)=\inf_{\phi\in\HK}L(\phi,\cdot)\ge\inf_{\phi\in\Gen}L(\phi,\cdot)=d_{\Gen}(\cdot)$.
For the reverse, fix $\mu$ and $\eps>0$ and choose $\phi_{\Gen}^{\eps/2}\in\Gen$ with
$L(\phi_{\Gen}^{\eps/2},\mu)\le\inf_{\phi\in\Gen}L(\phi,\mu)+\eps/2$. Since $L(\cdot,\mu)$ is
continuous on $\Gen$ (Section~\ref{subsec:envelope}) and $\HK$ is dense in $\Gen$, there is
$\phi_{\HK}^{\eps/2}\in\HK$ with
$L(\phi_{\HK}^{\eps/2},\mu)\le L(\phi_{\Gen}^{\eps/2},\mu)+\eps/2$, whence
$\inf_{\phi\in\HK}L(\phi,\mu)\le\inf_{\phi\in\Gen}L(\phi,\mu)+\eps$. As $\eps$ was arbitrary,
$d_{\Gen}=d_{\HK}$ on all of $\R^m$; note the argument nowhere uses $\mu\ge0$. Equality of
the functions gives equality of their argmax sets and of their superdifferentials at every
point.
\end{proof}

\begin{remark}
The quantitative content of Assumption~\ref{ass:loss-regularity} is stronger than what this
proof requires. The equality $d_{\Gen} = d_{\HK}$ uses only that $L(\cdot, \mu)$ is
continuous on $\Gen$ for each fixed $\mu$. Boundedness of the losses enters through
\cite{KalogeriasPougkakiotis22} as a sufficient condition for the Lagrangian to be well posed
and for the separation argument to apply, but is not minimal for strong duality per se. We
retain Assumption~\ref{ass:loss-regularity} in full because it is imposed throughout for
independent reasons --- it drives the Rademacher bounds and the uniform convergence analysis
--- so imposing it here costs nothing.
\end{remark}

\begin{lemmaBox}{Uniform boundedness of dual multipliers}{bounded-multipliers}
Let Assumptions~\ref{ass:loss-regularity} and~\ref{ass:slater} hold, and recall
$Q_{\mathrm{sl}} = \|\phi^{\mathrm{sl}}\|_{\HK}$ and $C_\mu = 4M/\xi$. Then, for every
$Q \ge Q_{\mathrm{sl}}$:
\begin{enumerate}[label=(\roman*)]
    \item $\arg\max_{\mu \ge 0} d_Q(\mu) \neq \emptyset$, and every
    $\mu_Q^\star \in \arg\max_{\mu \ge 0} d_Q(\mu)$ satisfies
    $\|\mu_Q^\star\|_1 \le 2M/\xi \le C_\mu$;
    \item the same conclusion holds for $d_{\HK}$ in place of $d_Q$: $\Sdual\ne\emptyset$ and
    $\Sdual\subseteq\Lambda$;
    \item on the event and for the sample sizes of Remark~\ref{rem:slater-pathwise},
    $\arg\max_{\mu \ge 0} \widehat{d}_{n}(\mu) \neq \emptyset$ and every
    $\widehat\mu_{n}^\star \in \arg\max_{\mu \ge 0} \widehat{d}_{n}(\mu)$ satisfies
    $\|\widehat\mu_{n}^\star\|_1 \le C_\mu$.
\end{enumerate}
Consequently both the population and (on that event) the empirical dual problems may be
restricted to the compact multiplier set $\Lambda = \{ \mu \in \R^m_+ : \|\mu\|_1 \le C_\mu \}$.
\end{lemmaBox}

\begin{proof}
(i) Fix $Q \ge Q_{\mathrm{sl}}$, so $\phi^{\mathrm{sl}} \in \Bq{Q}$. For any
$\mu \in \R^m_+$, evaluating the infimum at $\phi^{\mathrm{sl}}$ and using
Assumption~\ref{ass:slater},
\[
d_{Q}(\mu) \;\le\; L(\phi^{\mathrm{sl}}, \mu) \;=\; \mathcal{R}_0(\phi^{\mathrm{sl}}) + \sum_{i=1}^m \mu_i \mathcal{R}_i(\phi^{\mathrm{sl}})
\;\le\; \mathcal{R}_0(\phi^{\mathrm{sl}}) - \xi \|\mu\|_1 .
\]
Any maximizer $\mu^\star_Q$ satisfies $d_Q(\mu^\star_Q)\ge d_Q(0)$, and
$|\mathcal R_0|\le M$ gives $d_Q(0)\ge-M$ and $\mathcal R_0(\phi^{\mathrm{sl}})\le M$. Hence
\[
-M \;\le\; d_Q(0) \;\le\; d_Q(\mu^\star_Q) \;\le\; M - \xi\|\mu^\star_Q\|_1
\qquad\Longrightarrow\qquad
\|\mu^\star_Q\|_1 \;\le\; \frac{2M}{\xi}\;\le\;C_\mu .
\]
 The function $d_Q$ is an
infimum of affine functions of $\mu$, hence concave, and it is finite everywhere since
$|d_Q(\mu)|\le M(1+\|\mu\|_1)$, hence continuous; the displayed bound confines the search to
the compact set $\Lambda$, and the extreme value theorem gives nonemptiness.

(ii) The identical argument runs verbatim with $\HK$ in place of $\Bq{Q}$, since
$\phi^{\mathrm{sl}}\in\HK$. 

(iii) On the event of Remark~\ref{rem:slater-pathwise} we have
$\widehat{\mathcal{R}}_{i,n}(\phi^{\mathrm{sl}}) \le -\xi/2$ for all $i \in \mathbb{N}^+_m$
and $\phi^{\mathrm{sl}}\in\Bq{Q_n}$. Following the same logic,
$\widehat{d}_{n}(\mu) \le \widehat{\mathcal{R}}_{0,n}(\phi^{\mathrm{sl}}) - \tfrac{\xi}{2}\|\mu\|_1 \le M - \tfrac{\xi}{2}\|\mu\|_1$,
while $\widehat d_n(0)\ge-M$; hence
$-M \le \widehat{d}_{n}(\widehat\mu_{n}^\star) \le M - \tfrac{\xi}{2}\|\widehat\mu_{n}^\star\|_1$,
giving $\|\widehat\mu_{n}^\star\|_1 \le 4M/\xi = C_\mu$. It is this empirical case, with its
halved margin, that forces the constant $C_\mu = 4M/\xi$; the population bound in (i)--(ii)
is twice as sharp.
\end{proof}

\begin{lemmaBox}{PAC learnability of the dual on a fixed ball}{dual-pac-fixedQ}
Let Assumptions~\ref{ass:loss-regularity}--\ref{ass:slater} hold. Fix $Q\ge Q_{\mathrm{sl}}$
and let $\Lambda$ be as in Lemma~\ref{lem:bounded-multipliers}. For any $\delta \in (0,1)$,
define
\begin{equation}
\beta_n(Q,\delta)\triangleq(1+C_\mu)\left[\frac{cLQ}{\sqrt n}
    +cM\sqrt{\frac{\log((m+1)/\delta)}{n}}\right].
\end{equation}
Then, on the event and for the sample sizes of Remark~\ref{rem:slater-pathwise} intersected
with a probability-$(1-\delta)$ Rademacher event,
\(
    \sup_{\mu\in\Lambda}|\widehat d_{n}(\mu)-d_Q(\mu)| \le \beta_n(Q,\delta)
\)
and $|\widehat D_{n}^\star-D_Q^\star| \le \beta_n(Q,\delta)$.
\end{lemmaBox}

\begin{proof}
By Lemma~\ref{lem:bounded-multipliers} it suffices to control the empirical process on
$\Lambda$. For $i \in\mathbb N_m$ define
$\Delta_{i,n}(Q,\delta) \triangleq \sup_{\phi\in \Bq{Q}}|\widehat{\mathcal R}_{i,n}(\phi)-\mathcal R_i(\phi)|$.
Combining the standard Rademacher generalization bound with Talagrand's contraction lemma
and the Rademacher complexity of an RKHS ball
\cite[Theorems 3.3, 11.3, Lemma 5.7]{MohriRostamizadehTalwalkar18}, with probability at least
$1-\delta$, simultaneously for all $i$,
\[
    \Delta_{i,n}(Q,\delta) \;\le\; c\left(\frac{LQ}{\sqrt n} + M\sqrt{\frac{\log((m+1)/\delta)}{n}}\right),
\]
the $\log(m+1)$ arising from a union bound over the $m+1$ indices, and $c\ge1$ universal.
For $(\phi,\mu)\in \Bq{Q}\times\Lambda$, the triangle inequality gives
\[
    |\widehat L_{n}(\phi,\mu)-L_{\HK}(\phi,\mu)|
    \le \Delta_{0,n}+\|\mu\|_1\max_{1\le i\le m}\Delta_{i,n}
    \le (1+C_\mu)\max_{0\le i\le m}\Delta_{i,n}
    \;=\;\beta_n(Q,\delta).
\]
Since $\widehat d_{n}(\mu)$ and $d_Q(\mu)$ are infima of these two objectives over the
\emph{same} set $\Bq{Q}$, taking infima on both sides of
$\widehat L_n(\phi,\mu)\le L_{\HK}(\phi,\mu)+\beta_n$ and of the reverse gives
$|\widehat d_{n}(\mu)-d_Q(\mu)| \le \beta_n(Q,\delta)$ for each $\mu\in\Lambda$; taking
suprema over $\Lambda$, and then applying the same observation at the level of suprema,
yields both claims.
\end{proof}

\begin{lemmaBox}{Uniform boundedness of the Tikhonov complexity}{uniform-boundedness-epsilon-opt}
Let $\Lambda\subset\R^m_+$ be compact and fix $\eps>0$. Assume $d_{\HK}(\mu)$ is finite for
every $\mu\in\Lambda$, and that for every $\mu\in\Lambda$ the map $\phi\mapsto L(\phi,\mu)$
is sequentially weakly lower semicontinuous on $\HK$. Then:
\begin{enumerate}
    \item[\textnormal{(i)}] for every $\mu\in\Lambda$ the infimum defining $R_\eps(\mu)$ is
    attained;
    \item[\textnormal{(ii)}] $\mu\mapsto d_{\HK}(\mu)$ is continuous on $\Lambda$;
    \item[\textnormal{(iii)}] there exists $Q_\eps<\infty$ with
    $\sup_{\mu\in\Lambda}R_\eps(\mu)\le Q_\eps$; equivalently, for every $\mu\in\Lambda$ one
    may choose $\phi_\eps^\mu\in\HK$ with $L(\phi_\eps^\mu,\mu)\le d_{\HK}(\mu)+\eps$ and
    $\|\phi_\eps^\mu\|_{\HK}\le Q_\eps$.
\end{enumerate}
\end{lemmaBox}

\begin{proof}
(i) Fix $\mu\in\Lambda$. Since $d_{\HK}(\mu)$ is finite, $S_\eps(\mu)\ne\emptyset$ and
$R_\eps(\mu)<\infty$. Take $(f_j)\subset S_\eps(\mu)$ with
$\|f_j\|_{\HK}\to R_\eps(\mu)$. This sequence is bounded, hence by reflexivity admits a
weakly convergent subsequence $f_j\rightharpoonup f^\dagger$. Weak lower semicontinuity of
$L(\cdot,\mu)$ preserves $\eps$-optimality:
$L(f^\dagger,\mu)\le\liminf_j L(f_j,\mu)\le d_{\HK}(\mu)+\eps$, so
$f^\dagger\in S_\eps(\mu)$. Weak lower semicontinuity of the Hilbert norm gives
$\|f^\dagger\|_{\HK}\le\liminf_j\|f_j\|_{\HK}=R_\eps(\mu)$, while
$\|f^\dagger\|_{\HK}\ge R_\eps(\mu)$ by definition; so the infimum is attained at
$f^\dagger$.

(ii) For each $\phi$, $\mu\mapsto L(\phi,\mu)$ is affine, hence $d_{\HK}$ is concave on
$\R^m$ as a pointwise infimum of affine functions. Being also real-valued on $\R^m$, it is
continuous on $\R^m$ \cite[Theorem 10.1]{Rockafellar1970}, hence on $\Lambda$.

(iii) Suppose for contradiction that $\mu_j\in\Lambda$ with $R_\eps(\mu_j)\to\infty$. By
compactness pass to a subsequence $\mu_j\to\bar\mu\in\Lambda$. By (i) applied at level
$\eps/2$, choose $\bar f\in S_{\eps/2}(\bar\mu)$, so
$L(\bar f,\bar\mu)-d_{\HK}(\bar\mu)\le\eps/2$. Since $\bar f$ is fixed,
$\mu\mapsto L(\bar f,\mu)$ is continuous, and by (ii) so is $d_{\HK}$; hence
$L(\bar f,\mu_j)-d_{\HK}(\mu_j)\to L(\bar f,\bar\mu)-d_{\HK}(\bar\mu)\le\eps/2$. Therefore
for all large $j$, $L(\bar f,\mu_j)\le d_{\HK}(\mu_j)+\eps$, i.e.\
$\bar f\in S_\eps(\mu_j)$, giving $R_\eps(\mu_j)\le\|\bar f\|_{\HK}<\infty$ for all large
$j$ --- contradicting $R_\eps(\mu_j)\to\infty$. Hence
$Q_\eps\triangleq\sup_{\mu\in\Lambda}R_\eps(\mu)<\infty$, and (i) supplies the selector.
\end{proof}

\begin{corollary}[Verification of Lemma~\ref{lem:uniform-boundedness-epsilon-opt} and boundedness of $\mathfrak T^\eps_n$]
\label{cor:tikhonov-bounded}
Suppose Assumptions~\ref{ass:loss-regularity}--\ref{ass:slater} hold and $\HK$ satisfies
Assumption~\ref{ass:dense-rkhs}(a). Let $\Lambda$ be as in
Lemma~\ref{lem:bounded-multipliers}. Then $\Lambda$ satisfies the hypotheses of
Lemma~\ref{lem:uniform-boundedness-epsilon-opt}; consequently, for every $\eps>0$ there is
$Q_\eps<\infty$ independent of $n$ with $\sup_{\mu\in\Lambda}R_\eps(\mu)\le Q_\eps$,
$d_{\HK}$ is continuous on $\Lambda$, and the minimum-norm $\eps$-optimal selector is
attained for every $\mu\in\Lambda$. Moreover, on the event and for the sample sizes of
Remark~\ref{rem:slater-pathwise},
\[
    \mathfrak T^\eps_n \;\le\; Q_\eps \;<\;\infty,
\]
with the defining infimum attained at every empirical dual maximizer.
\end{corollary}

\begin{proof}
By Assumption~\ref{ass:loss-regularity}(b), $|\mathcal R_i(\phi)|\le M$ for all $\phi$ and
$i$, so $|L(\phi,\mu)|\le M(1+\|\mu\|_1)$ and hence
$|d_{\HK}(\mu)|\le M(1+\|\mu\|_1)<\infty$ for every $\mu\in\R^m$; $\Lambda$ is compact, so
the dual value is bounded on it.

For the second hypothesis we prove the stronger statement that each $\mathcal R_i$ is weakly
sequentially \emph{continuous} on $\HK$. Let $\phi_j\rightharpoonup\phi$ in $\HK$. Point
evaluation is a bounded linear functional (Assumption~\ref{ass:dense-rkhs}(a)), so
$\phi_j(x)=\langle\phi_j,K(x,\cdot)\rangle_{\HK}\to\phi(x)$ for every $x$. By
Assumption~\ref{ass:loss-regularity}(a) each $\ell_i(\cdot,y)$ is continuous, so
$\ell_i(\phi_j(x),y)\to\ell_i(\phi(x),y)$ pointwise; by
Assumption~\ref{ass:loss-regularity}(b) the integrands are uniformly bounded by $M$, so
dominated convergence gives $\mathcal R_i(\phi_j)\to\mathcal R_i(\phi)$. Weak sequential
continuity implies weak sequential lower semicontinuity, and since $\mu_i\ge0$ for
$i\in\mathbb N^+_m$, the same holds for
$L(\cdot,\mu)=\mathcal R_0(\cdot)+\sum_{i\ge1}\mu_i\mathcal R_i(\cdot)$.

Both hypotheses of Lemma~\ref{lem:uniform-boundedness-epsilon-opt} therefore hold, giving
$Q_\eps<\infty$, continuity of $d_{\HK}$ on $\Lambda$, and attainment of $R_\eps(\mu)$ for
every $\mu\in\Lambda$. On the event of Remark~\ref{rem:slater-pathwise} every empirical dual
maximizer lies in $\Lambda$ by Lemma~\ref{lem:bounded-multipliers}(iii), so
$\mathfrak T^\eps_n=\sup_{\mu\in\arg\max\widehat d_n}R_\eps(\mu)\le\sup_{\mu\in\Lambda}R_\eps(\mu)=Q_\eps$,
with the defining infimum attained at each such $\mu$.
\end{proof}

This completes the value arc. Set
\[
\mathfrak G_Q \;\triangleq\; \sup_{\mu\in\Lambda}\bigl[d_{Q}(\mu)-d_{\HK}(\mu)\bigr]\;\ge\;0 .
\]

Indeed, for every $\eps>0$ and every $Q\ge Q_\eps$: fix $\mu\in\Lambda$; by
Corollary~\ref{cor:tikhonov-bounded} the minimum-norm $\eps$-optimal selector
$\phi^\mu_\eps$ at $\mu$ is attained with $\|\phi^\mu_\eps\|_{\HK}=R_\eps(\mu)\le Q_\eps\le Q$,
so $\phi^\mu_\eps\in\Bq{Q}$ and $d_Q(\mu)\le L_{\HK}(\phi^\mu_\eps,\mu)\le d_{\HK}(\mu)+\eps$.
Taking the supremum over $\mu\in\Lambda$ gives $\mathfrak G_Q\le\eps$, with no attainment of
the primal required. Since both suprema defining $D^\star_Q$ and $D^\star_{\HK}$ are attained
in $\Lambda$ (Lemma~\ref{lem:bounded-multipliers}), $|D^\star_Q-D^\star_{\HK}|\le\eps$
whenever $Q\ge Q_\eps$.

In the growing-ball argument we therefore control this term as soon as $n$ is large enough
that $Q_n\ge Q_{\eps/2}$. Note that $Q_{\eps/2}$ depends on the $\mathcal D_i$ and on $\eps$,
but not on the realized sample path $\omega$; we remain in the universal regime, and
Corollary~\ref{cor:explicit-rate} quantifies that dependence under
Assumption~\ref{ass:source}. Assembling the three closed gaps through
\eqref{equation:GDEC} yields $|\widehat D^\star_n - P^\star_{\HK}| \le \eps$ with high
probability for $n$ large, which is part~(i) of the main theorem. We now turn to the harder
question --- feasibility --- which the value arc leaves entirely untouched.
\subsection{Feasibility-Related Lemmata}
\label{sec:feasibility-indices}

The core idea of the feasibility arc is to prove convergence of the relevant empirical
objects to their population counterparts, and then to read constraint satisfaction off the
population limit, up to the intrinsic gap $\epsinf$. The arc proceeds in five movements:
(0) a precise statement of the Levin--Valadier representation we use; (1) simultaneous
uniform convergence of empirical to population risks over the growing balls; (2) the
dual-side first-order structure and the subsequential convergence of empirical dual optima;
(3) the qualitative convergence result (Lemma~\ref{lem:multi-qual-conv}) bounding the
image-set deviation by $\epsinf$; and (4) the upgrade from subsequential to full-sequence
guarantees.

\subsubsection*{Movement 0: the representation of the superdifferential}

\begin{lemmaBox}{Levin--Valadier representation, empirical and population forms}{LV}
Let Assumption~\ref{ass:loss-regularity} hold and write
$\mathcal R(\phi)\triangleq(\mathcal R_1(\phi),\dots,\mathcal R_m(\phi))\in[-M,M]^m$,
$\widehat{\mathcal R}_n(\phi)$ analogously.
\begin{enumerate}[label=\textnormal{(\roman*)}, itemsep=2pt, topsep=2pt]
\item \textbf{(Empirical, exact form.)} For every $n$, every $\mu\in\R^m$ and every
$\omega$, the set
$K_n(\mu)\triangleq\widehat{\mathcal R}_n\bigl(\widehat\Phi_n(\mu)\bigr)\subseteq[-M,M]^m$
is nonempty and compact, $\conv K_n(\mu)$ is compact, and
\[
\partial\widehat d_n(\mu)\;=\;\conv K_n(\mu) .
\]
In particular the closure may be dropped from \eqref{eq:LV-empirical}, and every
$\nabla\in\partial\widehat d_n(\mu)$ is an \emph{exact} convex combination of at most $m+1$
realized empirical constraint-risk vectors.
\item \textbf{(Population, level-set form.)} For every $\mu\in\R^m$,
\[
\partial d_{\HK}(\mu)\;=\;\bigcap_{\eps>0}\ \clconv\bigl\{\mathcal R(\phi):\phi\in S_\eps(\mu)\bigr\}\;\subseteq\;[-M,M]^m ,
\]
where $S_\eps(\mu)$ is the $\eps$-optimal level set of Definition~\ref{def:tik}. If in
addition $\Phi_{\HK}(\mu)\ne\emptyset$ and $\mathcal R(\Phi_{\HK}(\mu))$ is closed, then
$\partial d_{\HK}(\mu)\supseteq\conv\mathcal R(\Phi_{\HK}(\mu))$.
\end{enumerate}
\end{lemmaBox}

\begin{proof}
(i) $\Bq{Q_n}$ is a closed ball in a Hilbert space, hence weakly compact, and by
Eberlein--\v Smulian weakly sequentially compact. Each $\phi\mapsto\widehat{\mathcal R}_{i,n}(\phi)$
is weakly sequentially continuous on $\Bq{Q_n}$: weak convergence implies pointwise
convergence by the reproducing property, and $\ell_i$ is continuous and bounded, so the
finite average converges. Consequently $\widehat L_n(\cdot,\mu)$, a finite linear combination
of these maps, is weakly sequentially continuous, hence attains its infimum on
$\Bq{Q_n}$ and has weakly sequentially closed argmin set; being a weakly sequentially closed
subset of a weakly sequentially compact set, $\widehat\Phi_n(\mu)$ is itself weakly
sequentially compact. The map $\Psi_n\triangleq\widehat{\mathcal R}_n$ is weakly
sequentially continuous into $\R^m$, so $K_n(\mu)=\Psi_n(\widehat\Phi_n(\mu))$ is
sequentially compact, hence compact, in $\R^m$. By Carath\'eodory's theorem
$\conv K_n(\mu)$ is the continuous image of the compact set
$K_n(\mu)^{m+1}\times\Delta_m$, hence compact and in particular closed.

For the identity, $\widehat d_n(\nu)=\inf_{\phi\in\Bq{Q_n}}\{\widehat{\mathcal R}_{0,n}(\phi)+\langle\nu,\widehat{\mathcal R}_n(\phi)\rangle\}$
is a pointwise infimum of a family of affine functions of $\nu$ indexed by
$\phi\in\Bq{Q_n}$, with the infimum attained. The classical representation of the
superdifferential of such an infimal function
\cite{Valadier,IoffeTihomirov1979,CastaingValadier1977} gives
$\partial\widehat d_n(\mu)=\clconv\{\widehat{\mathcal R}_n(\phi):\phi\in\widehat\Phi_n(\mu)\}$
when the infimum is attained and the index images are bounded, and we have just shown the
closure is redundant.

(ii) The same representation for an infimal function whose infimum need not be attained
requires passing to $\eps$-optimal level sets: this is the standard Ioffe--Tihomirov /
Valadier form \cite[Ch.~4]{IoffeTihomirov1979}, valid here because
$\{\mathcal R(\phi):\phi\in S_\eps(\mu)\}\subseteq[-M,M]^m$ is bounded and the family is
affine in $\nu$. The inclusion $\partial d_{\HK}(\mu)\subseteq[-M,M]^m$ is immediate. The
final inclusion holds because each $\phi\in\Phi_{\HK}(\mu)$ generates an affine majorant of
$d_{\HK}$ touching at $\mu$, so $\mathcal R(\phi)\in\partial d_{\HK}(\mu)$, and
$\partial d_{\HK}(\mu)$ is convex.
\end{proof}

\subsubsection*{Movement 1: uniform convergence}

\begin{lemmaBox}{High-probability simultaneous uniform convergence}{multi-hp-conv}
Under Assumptions~\ref{ass:loss-regularity} and~\ref{ass:dense-rkhs}, for every
$\delta \in (0,1)$ there exists a deterministic threshold
$n_{\mathrm{uc}}(\delta) \triangleq \lceil 1/\delta \rceil + 1 \in \mathbb{N}^+$
such that, with probability at least $1 - \delta$, simultaneously for all
$i \in \mathbb{N}_m$ and all $n \geq n_{\mathrm{uc}}(\delta)$,
\[
\Delta_{i,n}(Q_n, \delta) \triangleq \sup_{\phi \in \Bq{Q_n}}
\bigl|\widehat{\mathcal{R}}_{i,n}(\phi) - \mathcal{R}_i(\phi)\bigr|
\;\leq\; c\left(\frac{LQ_n}{\sqrt{n}} + M\sqrt{\frac{2\log(m+1) + 2\log n}{n}}\right) \;=\; \Delta_n(\delta),
\]
for a universal constant $c \geq 1$. Here $\delta$ enters \emph{only} through the starting
index $n_{\mathrm{uc}}(\delta)$, not through the constant inside the bracket. In particular
$\Delta_{i,n}(Q_n,\delta) = O(Q_n/\sqrt n + \sqrt{\log n/n}) \to 0$. All subsequent
convergence claims are made on this event, denoted $\Omega_{\mathrm{uc}}(\delta)$.
\end{lemmaBox}

\begin{proof}
Set $\delta_n \triangleq 1/((m+1)n^2)$ for $n \geq 1$. By the same Rademacher argument
underlying Lemma~\ref{lem:dual-pac-fixedQ}, applied at confidence level $\delta_n$, for each
$i\in\mathbb N_m$,
\[
\Pr\!\left(
\sup_{\phi \in \Bq{Q_n}}
\bigl|\widehat{\mathcal{R}}_{i,n}(\phi) - \mathcal{R}_i(\phi)\bigr| > \Delta_{i,n}(Q_n, \delta_n)
\right)
\leq \delta_n,
\]
where, by definition of $\Delta_{i,n}$ and using
$\log((m+1)/\delta_n) = \log((m+1)^2n^2) = 2\log(m+1) + 2\log n$,
\[
\Delta_{i,n}(Q_n,\delta_n)
= c\left(\frac{LQ_n}{\sqrt n} + M\sqrt{\frac{2\log(m+1) + 2\log n}{n}}\right).
\]
Let $E_n$ be the event that the supremum deviation exceeds $\Delta_{i,n}(Q_n,\delta_n)$ for
\emph{some} $i \in \mathbb{N}_m$. A union bound over the $m+1$ indices
$i \in \mathbb{N}_m = \{0,1,\dots,m\}$ gives
\[
\Pr(E_n) \;\leq\; \sum_{i=0}^m \delta_n \;=\; \frac{m+1}{(m+1)n^2} \;=\; \frac{1}{n^2}.
\]
Define $F_N \triangleq \bigcup_{n \geq N} E_n$, so that
$\Pr(F_N) \leq \sum_{n\ge N} n^{-2} \leq (N-1)^{-1}$ for $N\ge2$, by the integral comparison
$\sum_{n\ge N}n^{-2}\le\int_{N-1}^\infty t^{-2}\mathrm dt$. Setting
$N = n_{\mathrm{uc}}(\delta) = \lceil 1/\delta\rceil + 1$ gives
$N - 1 = \lceil 1/\delta\rceil \geq 1/\delta$, hence $\Pr(F_N) \leq \delta$. On
$\Omega_{\mathrm{uc}}(\delta) \triangleq F_{n_{\mathrm{uc}}(\delta)}^c$, which has
probability at least $1-\delta$, $E_n$ fails for every $n \geq n_{\mathrm{uc}}(\delta)$,
which is the claim. Note that, unlike a naive substitution of $\delta$ for $\delta_n$, the
confidence parameter enters only through the threshold, while the bound retains the
$2\log n$ term inherited from the shrinking per-step confidence levels.
\end{proof}

\subsubsection*{Movement 2: dual-side first-order structure}

\begin{lemmaBox}{Existence of sign- and complementarity-compatible subgradients}{sign-comp-subgrad}
Let Assumption~\ref{ass:loss-regularity}(b) hold, fix $n\in\mathbb{N}^+$ and $Q_n>0$, and
suppose $\widehat\mu_n \in \arg\max_{\mu\in\R^m_+}\widehat d_n(\mu)$ is nonempty. Then there
exists $\nabla_n \in \partial \widehat d_n(\widehat\mu_n)$ satisfying
\[
    \nabla_{n,i} \le 0
    \quad\text{and}\quad
    \widehat\mu_{n,i}\,\nabla_{n,i} = 0,
    \qquad i \in \mathbb{N}_m^+.
\]
The identical statement, with the identical proof, holds for $d_{\HK}$ and any
$\mu^\star\in\Sdual$.
\end{lemmaBox}

\begin{proof}
$\widehat d_n$ is concave on $\R^m$, being a pointwise infimum of the affine family
$\mu\mapsto\widehat L_n(\phi,\mu)$, and finite everywhere since
$|\widehat{\mathcal R}_{i,n}(\phi)|\le M$ gives
$|\widehat d_n(\mu)|\le M(1+\|\mu\|_1)<\infty$ for every $\mu\in\R^m$. A concave function
finite on all of $\R^m$ is continuous and superdifferentiable everywhere
\cite[Thms.\ 10.1, 23.4]{Rockafellar1970}, so
$\partial\widehat d_n(\widehat\mu_n)\ne\emptyset$.

Since $\widehat\mu_n$ maximizes $\widehat d_n$ over the closed convex cone $\R^m_+$, Fermat's
rule gives $0 \in \partial(-\widehat d_n)(\widehat\mu_n) + N_{\R^m_+}(\widehat\mu_n)$, the
sum rule applying exactly by continuity of $-\widehat d_n$ on $\R^m$
\cite[Thm.\ 23.8]{Rockafellar1970}. Unwinding signs, there exists
$\nabla_n\in\partial\widehat d_n(\widehat\mu_n)$ with $\nabla_n\in N_{\R^m_+}(\widehat\mu_n)$.
Finally, $v \in N_{\R^m_+}(\widehat\mu_n)$ iff $v\cdot(\mu-\widehat\mu_n)\le 0$ for every
$\mu\in\R^m_+$; since the constraint decouples coordinatewise, this holds iff $v_i\le0$
whenever $\widehat\mu_{n,i}=0$ and $v_i=0$ whenever $\widehat\mu_{n,i}>0$, i.e.\ iff
$v_i\le0$ and $\widehat\mu_{n,i}v_i=0$ for every $i\in\mathbb N_m^+$.
\end{proof}

\begin{remark}
The argument is entirely dual-side: it uses only that $\widehat d_n$ is concave and finite
everywhere (Assumption~\ref{ass:loss-regularity}(b) alone) and that $\widehat\mu_n$ is an
exact maximizer over $\R^m_+$. It invokes neither convexity of the primal problem nor any
zero-duality-gap property of the empirical constrained problem. Slater's condition enters one
level upstream, in guaranteeing via Lemma~\ref{lem:bounded-multipliers} that
$\arg\max_{\mu\ge0}\widehat d_n(\mu)$ is nonempty.
\end{remark}

\begin{corollary}[Uniform Lagrangian deviation bound]\label{cor:lagrangian-dev}
Under the hypotheses of Lemma~\ref{lem:multi-hp-conv}, on $\Omega_{\mathrm{uc}}(\delta)$: for
all $n \geq n_{\mathrm{uc}}(\delta)$, all $\mu \in \R^m_+$, and all $\phi \in \Bq{Q_n}$,
\[
\bigl|\widehat{L}_n(\phi,\mu) - L_{\HK}(\phi,\mu)\bigr|
\;\leq\;
(1+\|\mu\|_1)\,\Delta_n(\delta),
\qquad
\Delta_n(\delta) = \max_{0 \leq i \leq m} \Delta_{i,n}(Q_n,\delta),
\]
so that $\Delta_n(\delta) \to 0$ on $\Omega_{\mathrm{uc}}(\delta)$.
\end{corollary}

\begin{proof}
By the triangle inequality,
$|\widehat{L}_n(\phi,\mu) - L_{\HK}(\phi,\mu)| \leq \Delta_{0,n} + \|\mu\|_1\max_{1\le i\le m}\Delta_{i,n}$,
and bounding both by $\Delta_n(\delta)$ gives the claim.
\end{proof}

\begin{lemmaBox}{Convergent subsequence of empirical dual optima}{multi-dual-conv}
Under Assumptions~\ref{ass:non-atomic}--\ref{ass:slater}, on the event and for the sample
sizes of Remark~\ref{rem:slater-pathwise} intersected with $\Omega_{\mathrm{uc}}(\delta)$:
every cluster point of $(\widehat\mu_n)_n$ lies in $\Sdual$. In particular
$(\widehat\mu_n)_n$ admits a convergent subsequence
$\widehat\mu_{n_j} \to \mu_{\HK} \in \Sdual$.
\end{lemmaBox}

\begin{proof}
\textit{Boundedness.} By Remark~\ref{rem:slater-pathwise} the empirical Slater condition
holds pathwise with margin $\xi/2$ once $Q_n\ge Q_{\mathrm{sl}}$, so by
Lemma~\ref{lem:bounded-multipliers}(iii) $\|\widehat\mu_n\|_1 \leq C_\mu$ for all such $n$.
Hence $(\widehat\mu_n)_n$ is eventually contained in the compact set $\Lambda$ and admits a
convergent subsequence $\widehat\mu_{n_j} \to \mu_{\HK}$ with $\|\mu_{\HK}\|_1 \leq C_\mu$.

\textit{The limit is a population optimum.} Fix $\mu \geq 0$ with $\|\mu\|_1 \leq C_\mu$. By
optimality of $\widehat\mu_{n_j}$,
\begin{equation}\tag{$\ast$}\label{eq:opt-ineq}
\widehat{d}_{n_j}(\widehat\mu_{n_j}) \;\geq\; \widehat{d}_{n_j}(\mu).
\end{equation}
Lower bound on the right side: by Corollary~\ref{cor:lagrangian-dev}, for every
$\phi \in \Bq{Q_{n_j}}$,
$\widehat{L}_{n_j}(\phi,\mu) \geq L_{\HK}(\phi,\mu) - (1+\|\mu\|_1)\Delta_{n_j}(\delta)$;
taking the infimum over $\phi \in \Bq{Q_{n_j}} \subset \HK$ and using
$\inf_{\Bq{Q_{n_j}}} L_{\HK}(\cdot,\mu) \geq d_{\HK}(\mu)$ yields
$\widehat{d}_{n_j}(\mu) \geq d_{\HK}(\mu) - (1+\|\mu\|_1)\Delta_{n_j}(\delta)$, so
$\liminf_j \widehat{d}_{n_j}(\mu) \geq d_{\HK}(\mu)$.
Upper bound on the left side: fix $\eta > 0$ and choose $\phi^\dagger_\eta \in \HK$ with
$L_{\HK}(\phi^\dagger_\eta, \mu_{\HK}) \leq d_{\HK}(\mu_{\HK}) + \eta$. Since
$\phi^\dagger_\eta$ has finite RKHS norm and $Q_{n_j}\nearrow\infty$, we have
$\phi^\dagger_\eta \in \Bq{Q_{n_j}}$ for large $j$, and by
Corollary~\ref{cor:lagrangian-dev},
\[
\widehat{d}_{n_j}(\widehat\mu_{n_j})
\leq \widehat{L}_{n_j}(\phi^\dagger_\eta, \widehat\mu_{n_j})
\leq L_{\HK}(\phi^\dagger_\eta, \widehat\mu_{n_j}) + (1 + C_\mu)\Delta_{n_j}(\delta).
\]
Since $\mu \mapsto L_{\HK}(\phi^\dagger_\eta, \mu)$ is affine hence continuous and
$\widehat\mu_{n_j} \to \mu_{\HK}$, taking $\limsup_j$ gives
$\limsup_j \widehat{d}_{n_j}(\widehat\mu_{n_j}) \leq d_{\HK}(\mu_{\HK}) + \eta$; as $\eta$
was arbitrary, $\limsup_j \widehat{d}_{n_j}(\widehat\mu_{n_j}) \leq d_{\HK}(\mu_{\HK})$.
Taking $\limsup_j$ on the left of \eqref{eq:opt-ineq} and $\liminf_j$ on the right,
\[
d_{\HK}(\mu_{\HK}) \;\geq\; \limsup_j \widehat{d}_{n_j}(\widehat\mu_{n_j}) \;\geq\; \liminf_j \widehat{d}_{n_j}(\mu) \;\geq\; d_{\HK}(\mu)
\]
for every $\mu \geq 0$ with $\|\mu\|_1 \leq C_\mu$. By
Lemma~\ref{lem:bounded-multipliers}(ii) every population dual optimum lies in $\Lambda$, so
the restricted supremum coincides with $D_{\HK}^\star$, whence $\mu_{\HK} \in \Sdual$.
Since the argument applies to any convergent subsequence, every cluster point lies in
$\Sdual$.
\end{proof}

\begin{lemmaBox}{Pointwise and locally uniform convergence of the empirical dual}{multi-dual-pointwise}
Under Assumptions~\ref{ass:non-atomic}--\ref{ass:slater}, on $\Omega_{\mathrm{uc}}(\delta)$:
for every fixed $\mu \in \R^m$ (not merely $\mu\in\R^m_+$),
$\widehat{d}_n(\mu) \to d_{\HK}(\mu)$ as $n \to \infty$. This upgrades to uniform convergence
on every compact subset of $\R^m$.
\end{lemmaBox}

\begin{proof}
Work on $\Omega_{\mathrm{uc}}(\delta)$. Fix $\mu \in \R^m$. By
Corollary~\ref{cor:lagrangian-dev}, valid for every $\mu\in\R^m$ upon replacing $\|\mu\|_1$
by itself in the bound, for every $\phi \in \Bq{Q_n}$,
$\widehat{L}_n(\phi,\mu) \geq L_{\HK}(\phi,\mu) - (1+\|\mu\|_1)\Delta_n(\delta)$; taking the
infimum over $\Bq{Q_n} \subseteq \HK$ gives
$\widehat{d}_n(\mu) \geq d_{\HK}(\mu) - (1+\|\mu\|_1)\Delta_n(\delta)$, so
$\liminf_n \widehat{d}_n(\mu) \geq d_{\HK}(\mu)$. For the upper bound, fix $\eps > 0$ and
choose $\phi_\eps \in \HK$ with $L_{\HK}(\phi_\eps,\mu) \leq d_{\HK}(\mu) + \eps$; since
$\|\phi_\eps\|_{\HK} < \infty$ and $Q_n \nearrow \infty$, $\phi_\eps \in \Bq{Q_n}$ for all
large $n$, and then
$\widehat{d}_n(\mu) \leq \widehat L_n(\phi_\eps,\mu) \leq d_{\HK}(\mu) + \eps + (1+\|\mu\|_1)\Delta_n(\delta)$.
As $\eps>0$ is arbitrary and $\Delta_n(\delta) \to 0$,
$\limsup_n \widehat{d}_n(\mu) \leq d_{\HK}(\mu)$.

Each $\widehat{d}_n$ is concave and finite on all of $\R^m$, as is $d_{\HK}$, with
$|\widehat d_n(\mu)|,|d_{\HK}(\mu)| \leq M(1+\|\mu\|_1)$. By
\cite[Thm.~10.8]{Rockafellar1970}, pointwise convergence of finite concave functions on
$\R^m$ implies uniform convergence on every compact subset of the relative interior of the
common domain, which here is all of $\R^m$.
\end{proof}

\subsubsection*{Movement 3: the core convergence lemma}

Since the event of Assumption~\ref{ass:multi-argmin-recovery} has probability one,
intersecting it with $\Omega_{\mathrm{uc}}(\delta)$ costs nothing in confidence. We set
$\Omega^\star(\delta) \triangleq \Omega_{\mathrm{uc}}(\delta) \cap \Omega_{\mathrm{fbl}}$,
$\Pr(\Omega^\star(\delta)) \geq 1-\delta$, and work on $\Omega^\star(\delta)$ throughout.

\begin{lemmaBox}{Qualitative convergence up to the population closure--realization gap}{multi-qual-conv}
Under Assumptions~\ref{ass:non-atomic}--\ref{ass:multi-argmin-recovery}, on
$\Omega^\star(\delta)$: for every sequence $(\widehat\mu_n)_n$ of empirical dual optima and
every sequence $\nabla_n \in \partial \widehat{d}_n(\widehat\mu_n)$ with
$\nabla_{n,i} \leq 0$, there exists a subsequence $\{n_k\}$ along which
$\widehat\mu_{n_k} \to \mu_{\HK}\in\Sdual$ and
$\nabla_{n_k} \to \nabla^\infty \in \partial d_{\HK}(\mu_{\HK})$, and
\[
\limsup_{k \to \infty}\;
\Theta_{\nabla_{n_k}}(\widehat\mu_{n_k})
\;\leq\;
\Theta^\infty_{\nabla^\infty}(\mu_{\HK})
\;\leq\;
\epsinf(\mu_{\HK}).
\]
In particular, if $\epsinf(\mu_{\HK}) = 0$ then $\Theta_{\nabla_{n_k}}(\widehat\mu_{n_k}) \to 0$.
\end{lemmaBox}

\begin{proof}
Work on $\Omega^\star(\delta)$. By Lemma~\ref{lem:multi-dual-conv}, $(\widehat\mu_n)_n$
admits a subsequence $\{n_j\}$ with $\widehat\mu_{n_j} \to \mu_{\HK}\in\Sdual$. The proof of
that lemma in fact establishes the sharper fact that the dual \emph{values} converge:
\begin{equation}
\label{eq:value-conv}
\widehat{d}_{n_j}(\widehat\mu_{n_j}) \;\longrightarrow\; d_{\HK}(\mu_{\HK}),
\end{equation}
since it shows $\liminf_j \widehat{d}_{n_j}(\mu) \geq d_{\HK}(\mu)$ for every $\mu\in\Lambda$
--- take $\mu=\mu_{\HK}$ and combine with
$\widehat d_{n_j}(\widehat\mu_{n_j}) \ge \widehat d_{n_j}(\mu_{\HK})$ for the liminf --- and
$\limsup_j \widehat{d}_{n_j}(\widehat\mu_{n_j}) \leq d_{\HK}(\mu_{\HK})$ directly. Relabel
$\{n_j\}$ as the working sequence.

By Lemma~\ref{lem:LV}(i), $\nabla_n \in \conv K_n(\widehat\mu_n) \subseteq[-M,M]^m$;
combined with $\nabla_{n,i}\le 0$ this gives $\nabla_n \in [-M,0]^m$ for every $n$. By
Bolzano--Weierstrass, extract a further subsequence $\{n_k\} \subseteq \{n_j\}$ with
$\nabla_{n_k} \to \nabla^\infty \in [-M,0]^m$; the convergences
$\widehat\mu_{n_k} \to \mu_{\HK}$ and \eqref{eq:value-conv} persist.

\textit{Step 1: $\nabla^\infty \in \partial d_{\HK}(\mu_{\HK})$.} Since $\widehat d_{n_k}$ is
concave and finite on $\R^m$ and $\nabla_{n_k}\in\partial\widehat d_{n_k}(\widehat\mu_{n_k})$,
the supergradient inequality reads
\[
\widehat d_{n_k}(\nu)\;\le\;\widehat d_{n_k}(\widehat\mu_{n_k})+\bigl\langle\nabla_{n_k},\,\nu-\widehat\mu_{n_k}\bigr\rangle
\qquad\text{for every }\nu\in\R^m .
\]
Fix $\nu\in\R^m$ and let $k\to\infty$. The left side converges to $d_{\HK}(\nu)$ by
Lemma~\ref{lem:multi-dual-pointwise}; the first term on the right converges to
$d_{\HK}(\mu_{\HK})$ by \eqref{eq:value-conv}; and the inner product converges to
$\langle\nabla^\infty,\nu-\mu_{\HK}\rangle$ since $\nabla_{n_k}\to\nabla^\infty$ and
$\widehat\mu_{n_k}\to\mu_{\HK}$. Hence
$d_{\HK}(\nu)\le d_{\HK}(\mu_{\HK})+\langle\nabla^\infty,\nu-\mu_{\HK}\rangle$ for every
$\nu\in\R^m$, i.e.\ $\nabla^\infty\in\partial d_{\HK}(\mu_{\HK})$.

\textit{Step 2: transferring to the image sets.} Fix $\eta > 0$. By
Assumption~\ref{ass:multi-argmin-recovery}(a) the set $\Phi_{\HK}(\mu_{\HK})$
is nonempty, so we may choose $\phi^\dagger_\eta \in \Phi_{\HK}(\mu_{\HK})$ with
\begin{equation}
\label{eq:phi-dagger-choice}
\max_{i \in \mathbb{N}_m^+}
\bigl|\mathcal{R}_i(\phi^\dagger_\eta) - \nabla^\infty_i\bigr|
\;\leq\;
\Theta^\infty_{\nabla^\infty}(\mu_{\HK}) + \eta.
\end{equation}
Since $Q_n \nearrow \infty$, there is $n_0(\eta)$ with
$\phi^\dagger_\eta \in \Bq{Q_n}$ for all $n \geq n_0(\eta)$. Applying
Assumption~\ref{ass:multi-argmin-recovery}(b) on $\Omega_{\mathrm{fbl}}$ --- legitimately,
since that assumption is quantified over arbitrary subsequences and over every admissible
reference point, and $\phi^\dagger_\eta$ was chosen after the subsequence --- to the sequence
$k \mapsto \widehat\mu_{n_k}$ with reference point $\phi^\dagger_\eta$ yields
\begin{equation}
\label{eq:L1-recovery}
\lim_{k \to \infty}\;
\inf_{\phi \,\in\, \widehat{\Phi}_{n_k}(\widehat\mu_{n_k})}
\bigl\|\phi - \phi^\dagger_\eta\bigr\|_{L_1(\widehat{\mathcal{M}}^{\,n_k}_{\mathcal{X}})}
\;=\; 0 .
\end{equation}
We convert this into control of the empirical constraint risks. Fix
$\phi \in \widehat{\Phi}_{n_k}(\widehat\mu_{n_k})$ and $i \in \mathbb{N}_m^+$. Since
$\ell_i(\cdot,y)$ is $L$-Lipschitz uniformly in $y$,
\[
\bigl|\widehat{\mathcal{R}}_{i,n_k}(\phi) - \widehat{\mathcal{R}}_{i,n_k}(\phi^\dagger_\eta)\bigr|
\;\le\;
\frac{L}{n_k}\sum_{j=1}^{n_k}\bigl|\phi(X^{(i)}_j) - \phi^\dagger_\eta(X^{(i)}_j)\bigr|
\;=\;
L\,\bigl\|\phi - \phi^\dagger_\eta\bigr\|_{L_1(\widehat{\mathcal{D}}^{\,n_k}_{\mathcal{X},i})}.
\]
Because
$\widehat{\mathcal{M}}^{\,n_k}_{\mathcal{X}} \ge \tfrac{1}{m+1}\widehat{\mathcal{D}}^{\,n_k}_{\mathcal{X},i}$,
we have
$\|\cdot\|_{L_1(\widehat{\mathcal{D}}^{\,n_k}_{\mathcal{X},i})} \le (m+1)\|\cdot\|_{L_1(\widehat{\mathcal{M}}^{\,n_k}_{\mathcal{X}})}$,
so, maximizing over $i$ and taking the infimum over
$\phi \in \widehat{\Phi}_{n_k}(\widehat\mu_{n_k})$, \eqref{eq:L1-recovery} yields
$K_0(\eta) \in \mathbb{N}^+$ such that for all $k \geq K_0(\eta)$,
\begin{equation}
\label{eq:term1-bound}
\inf_{\phi \,\in\, \widehat{\Phi}_{n_k}(\widehat\mu_{n_k})}
\;\max_{i \in \mathbb{N}_m^+}
\bigl|\widehat{\mathcal{R}}_{i,n_k}(\phi) - \widehat{\mathcal{R}}_{i,n_k}(\phi^\dagger_\eta)\bigr|
\;<\; \eta
\end{equation}
(taking $L>0$ without loss of generality, the case $L=0$ being trivial).

\textit{Step 3: assembling.} For fixed $\phi \in \widehat{\Phi}_{n_k}(\widehat\mu_{n_k})$ and
$i \in \mathbb{N}_m^+$, write
\begin{multline*}
\widehat{\mathcal{R}}_{i,n_k}(\phi) - \nabla_{n_k,i}
=
\underbrace{\bigl[\widehat{\mathcal{R}}_{i,n_k}(\phi) - \widehat{\mathcal{R}}_{i,n_k}(\phi^\dagger_\eta)\bigr]}_{\text{(I)}}
+ \underbrace{\bigl[\widehat{\mathcal{R}}_{i,n_k}(\phi^\dagger_\eta) - \mathcal{R}_i(\phi^\dagger_\eta)\bigr]}_{\text{(II)}}
+ \underbrace{\bigl[\mathcal{R}_i(\phi^\dagger_\eta) - \nabla^\infty_i\bigr]}_{\text{(III)}}
+ \underbrace{\bigl[\nabla^\infty_i - \nabla_{n_k,i}\bigr]}_{\text{(IV)}}.
\end{multline*}
Terms (II)--(IV) do not depend on $\phi$. Taking $|\cdot|$, then
$\max_{i\in\mathbb{N}_m^+}$, then the infimum over $\phi$, and using subadditivity of the
maximum over the $\phi$-independent sum,
\[
\Theta_{\nabla_{n_k}}(\widehat\mu_{n_k})
\;\leq\;
\underbrace{\inf_{\phi}\max_{i} |\text{(I)}(\phi)|}_{\text{(a)}}
+ \underbrace{\max_i |\text{(II)}|}_{\text{(b)}}
+ \underbrace{\max_i |\text{(III)}|}_{\text{(c)}}
+ \underbrace{\|\nabla^\infty - \nabla_{n_k}\|_\infty}_{\text{(d)}}.
\]
For $n_k \geq \max\{n_0(\eta), n_{\mathrm{uc}}(\delta)\}$ and $k \geq K_0(\eta)$: (a)
$< \eta$ by \eqref{eq:term1-bound}; (b) $\leq \Delta_{n_k}(\delta) \to 0$ by
Lemma~\ref{lem:multi-hp-conv}, since $\phi^\dagger_\eta \in \Bq{Q_{n_k}}$; (c)
$\leq \Theta^\infty_{\nabla^\infty}(\mu_{\HK}) + \eta$ by \eqref{eq:phi-dagger-choice}; and
(d) $\to 0$. Taking $\limsup_{k}$ gives
$\limsup_k \Theta_{\nabla_{n_k}}(\widehat\mu_{n_k}) \leq \Theta^\infty_{\nabla^\infty}(\mu_{\HK}) + 2\eta$;
as $\eta>0$ was arbitrary and
$\Theta^\infty_{\nabla^\infty}(\mu_{\HK}) \le \sup_{\nabla \in \partial d_{\HK}(\mu_{\HK})}\Theta^\infty_\nabla(\mu_{\HK}) = \epsinf(\mu_{\HK})$
by Step~1 and Definition~\ref{def:multi-eps-star}, the claim follows.
\end{proof}

\begin{corollary}[Asymptotic feasibility up to the population gap]
\label{cor:multi-qual-feas}
Under the hypotheses of Lemma~\ref{lem:multi-qual-conv}, for every
$\omega \in \Omega^\star(\delta)$: along the subsequence $\{n_k\}$ furnished by that lemma,
the DERM selection $\widehat\phi_{n_k} \in \widehat{\Phi}_{n_k}(\widehat{\mu}_{n_k})$
satisfies, for every $i \in \mathbb{N}_m^+$,
\[
\limsup_{k \to \infty}\,
\widehat{\mathcal{R}}_{i,n_k}(\widehat{\phi}_{n_k})
\;\leq\;
\Theta^\infty_{\nabla^\infty}(\mu_{\HK})
\;\leq\;
\epsinf(\mu_{\HK}).
\]
When $\epsinf(\mu_{\HK}) = 0$, each empirical constraint is asymptotically satisfied.
\end{corollary}

\begin{proof}
By the DERM selection rule, $\widehat\phi_{n_k}$ attains the infimum defining
$\Theta_{\nabla_{n_k}}(\widehat\mu_{n_k})$ within $1/n_k$,\footnote{The selection exists
because the minimization defining $\widehat\Phi_n(\mu)$ is effectively finite-dimensional.
Let $V_n \triangleq \operatorname{span}\{K(\cdot,X^{(i)}_j)\}$ over the $(m+1)n$ observed
inputs and $P_n$ the orthogonal projection onto it. As
$(I-P_n)\phi \perp K(\cdot,X^{(i)}_j)$, the reproducing property gives
$\phi(X^{(i)}_j) = (P_n\phi)(X^{(i)}_j)$ at every observed input, so
$\widehat{L}_n(\phi,\mu) = \widehat{L}_n(P_n\phi,\mu)$ while
$\|P_n\phi\|_{\HK} \leq \|\phi\|_{\HK}$; hence
$\inf_{\Bq{Q_n}}\widehat{L}_n(\cdot,\mu) = \inf_{\Bq{Q_n} \cap V_n}\widehat{L}_n(\cdot,\mu)$.
The latter set is compact and $\widehat{L}_n(\cdot,\mu)$ is Lipschitz, so a minimizer exists
there and minimizes over all of $\Bq{Q_n}$. Only the projection is pinned down, so
$\widehat{\Phi}_n(\mu)$ is not norm compact, which is why
Lemma~\ref{lem:multi-qual-conv} and Lemma~\ref{lem:LV} argue in the weak topology.} so
$|\widehat{\mathcal{R}}_{i,n_k}(\widehat{\phi}_{n_k}) - \nabla_{n_k,i}| \leq \Theta_{\nabla_{n_k}}(\widehat{\mu}_{n_k}) + 1/n_k$.
Since $\nabla_{n_k,i} \leq 0$,
$\widehat{\mathcal{R}}_{i,n_k}(\widehat{\phi}_{n_k}) \leq \Theta_{\nabla_{n_k}}(\widehat{\mu}_{n_k}) + 1/n_k$,
and Lemma~\ref{lem:multi-qual-conv} gives the bound on taking $\limsup_k$.
\end{proof}

\begin{corollary}[Approximate complementary slackness]
\label{cor:multi-qual-comp-slack}
Under the hypotheses of Lemma~\ref{lem:multi-qual-conv}, for every
$\omega \in \Omega^\star(\delta)$, along $\{n_k\}$, with $\nabla_{n_k}$ satisfying
complementary slackness at $\widehat\mu_{n_k}$ and $\widehat\phi_{n_k}$ as in
Corollary~\ref{cor:multi-qual-feas},
\[
\limsup_{k \to \infty}
\sum_{i=1}^m
\widehat{\mu}_{n_k,i} \,
\bigl|\widehat{\mathcal{R}}_{i,n_k}(\widehat{\phi}_{n_k})\bigr|
\;\leq\;
C_\mu \cdot \epsinf(\mu_{\HK}).
\]
In particular, when $\epsinf(\mu_{\HK}) = 0$, exact complementary slackness is recovered in
the limit.
\end{corollary}

\begin{proof}
Complementary slackness means $\widehat{\mu}_{n_k,i} > 0 \Rightarrow \nabla_{n_k,i} = 0$; for
such $i$,
$|\widehat{\mathcal{R}}_{i,n_k}(\widehat{\phi}_{n_k})| = |\widehat{\mathcal{R}}_{i,n_k}(\widehat{\phi}_{n_k}) - \nabla_{n_k,i}| \leq \Theta_{\nabla_{n_k}}(\widehat{\mu}_{n_k}) + 1/n_k$,
while indices with $\widehat{\mu}_{n_k,i} = 0$ contribute nothing. Summing with weights
$\widehat\mu_{n_k,i}$ and using $\|\widehat{\mu}_{n_k}\|_1 \leq C_\mu$ gives the bound on
taking $\limsup_k$ and applying Lemma~\ref{lem:multi-qual-conv}.
\end{proof}

\begin{corollary}[Population feasibility and optimality of DERM outputs]
\label{cor:multi-primal-guarantee}
Under Assumptions~\ref{ass:non-atomic}--\ref{ass:multi-argmin-recovery}, along any DERM
selection $(\widehat\mu_n,\nabla_n,\widehat\phi_n)_n$, the following holds for every
$\omega \in \Omega^\star(\delta)$. Along the subsequence $\{n_k\}$ of
Lemma~\ref{lem:multi-qual-conv}:
\begin{enumerate}[label=\textnormal{(\alph*)},itemsep=2pt,topsep=2pt]
\item \textbf{Population subgradient tracking.} For every $i \in \mathbb{N}_m^+$,
\begin{equation}
\label{eq:pop-tracking}
    \limsup_{k \to \infty}\,
    \bigl|\mathcal{R}_i(\widehat{\phi}_{n_k}) - \nabla^\infty_i\bigr|
    \;\leq\;
    \epsinf(\mu_{\HK}).
\end{equation}
\item \textbf{Population feasibility.} For every $i \in \mathbb{N}_m^+$,
$\limsup_{k} \mathcal{R}_i(\widehat{\phi}_{n_k}) \leq \nabla^\infty_i + \epsinf(\mu_{\HK}) \leq \epsinf(\mu_{\HK})$.
\item \textbf{Population optimality.}
$\limsup_{k} \mathcal{R}_0(\widehat{\phi}_{n_k}) \leq P^\star_{\HK} + C_\mu\,\epsinf(\mu_{\HK})$.
\end{enumerate}
If additionally $\Sdual$ is a singleton, all three conclusions hold along the full original
sequence.
\end{corollary}

\begin{proof}
Fix $\omega \in \Omega^\star(\delta)$ and work along $\{n_k\}$.

\textit{(a).} By the DERM selection,
$|\widehat{\mathcal{R}}_{i,n_k}(\widehat\phi_{n_k}) - \nabla_{n_k,i}| \leq \Theta_{\nabla_{n_k}}(\widehat\mu_{n_k}) + 1/n_k$.
Since $\widehat\phi_{n_k} \in \Bq{Q_{n_k}}$, Lemma~\ref{lem:multi-hp-conv} gives
$|\widehat{\mathcal{R}}_{i,n_k}(\widehat\phi_{n_k}) - \mathcal{R}_i(\widehat\phi_{n_k})| \leq \Delta_{n_k}(\delta) \to 0$.
Two applications of the triangle inequality give
\[
    \bigl|\mathcal{R}_i(\widehat{\phi}_{n_k}) - \nabla^\infty_i\bigr|
    \leq
    \Theta_{\nabla_{n_k}}(\widehat\mu_{n_k}) + 1/n_k + \Delta_{n_k}(\delta) + \|\nabla^\infty - \nabla_{n_k}\|_\infty,
\]
and $\limsup_k$ with Lemma~\ref{lem:multi-qual-conv} yields \eqref{eq:pop-tracking}.

\textit{(b).} Immediate from (a) and $\nabla^\infty_i \leq 0$.

\textit{(c).} Since $\widehat\phi_{n_k}$ attains
$\widehat{d}_{n_k}(\widehat\mu_{n_k}) = \inf_{\Bq{Q_{n_k}}}\widehat L_{n_k}(\cdot,\widehat\mu_{n_k})$
within $1/n_k$,
\[
    \widehat{\mathcal{R}}_{0,n_k}(\widehat\phi_{n_k})
    + \sum_{i=1}^m \widehat\mu_{n_k,i}\,\widehat{\mathcal{R}}_{i,n_k}(\widehat\phi_{n_k})
    \;\leq\;
    \widehat{d}_{n_k}(\widehat\mu_{n_k}) + \tfrac{1}{n_k}.
\]
The DERM selection enforces $\widehat\mu_{n_k,i}\nabla_{n_k,i} = 0$ for every $i$, so
$\sum_i \widehat\mu_{n_k,i}\widehat{\mathcal{R}}_{i,n_k}(\widehat\phi_{n_k}) = \sum_i \widehat\mu_{n_k,i}(\widehat{\mathcal{R}}_{i,n_k}(\widehat\phi_{n_k}) - \nabla_{n_k,i})$,
a signed identity rather than a one-sided inequality. Rearranging and bounding the residual
in absolute value using the display in (a) and $\|\widehat\mu_{n_k}\|_1 \leq C_\mu$,
\[
    \widehat{\mathcal{R}}_{0,n_k}(\widehat\phi_{n_k})
    \;\leq\;
    \widehat{d}_{n_k}(\widehat\mu_{n_k}) + \tfrac{1}{n_k}
    + C_\mu\bigl(\Theta_{\nabla_{n_k}}(\widehat\mu_{n_k}) + \tfrac{1}{n_k}\bigr).
\]
By \eqref{eq:value-conv} and Theorem~\ref{thm:dense-transfer},
$\widehat{d}_{n_k}(\widehat\mu_{n_k}) \to d_{\HK}(\mu_{\HK}) = P^\star_{\HK}$; combining with
Lemma~\ref{lem:multi-qual-conv} and Lemma~\ref{lem:multi-hp-conv} at $i=0$ to transfer to the
population objective gives (c).

\textit{Full-sequence conclusion.} When $\Sdual$ is a singleton,
Lemma~\ref{lem:multi-dual-conv} gives $\widehat\mu_n \to \mu_{\HK}$ along the full sequence,
and the sub-subsequence argument in the proof of Lemma~\ref{lem:multi-qual-conv} gives
$\limsup_n \Theta_{\nabla_n}(\widehat\mu_n) \leq \epsinf(\mu_{\HK})$ along the full sequence,
since $\mu_{\HK}$ is then the only possible subsequential dual limit. Both derivations apply
verbatim with $n$ in place of $n_k$.
\end{proof}

\paragraph{Subsequential versus full-sequence guarantees.}
\label{subsec:three-regimes}
The guarantees above come in two structurally different forms, and only one is compatible
with the uniformization below. A \emph{subsequential} guarantee asserts that, for each
realization $\omega$ in some event of high probability, \emph{some} subsequence of sample
sizes exists along which a stated bound holds; it is silent about every sample size outside
that subsequence, at which the bound may fail arbitrarily badly.
Corollary~\ref{cor:multi-primal-guarantee} is of exactly this form. A \emph{full-sequence},
or \emph{eventual}, guarantee is stronger: for each such $\omega$ there is a threshold beyond
which the bound holds for \emph{every} sample size. We prove two guarantees of this stronger
form below --- Lemma~\ref{lem:full-seq-feas} and Corollary~\ref{cor:full-seq-conv} --- and it
is only these that admit uniformization.
\subsection{Uniformization of Pointwise Eventual Bounds}
\label{subsec:egorov}

The plan for the rest of the section is: (1) isolate the abstract uniformization step as a
self-contained lemma, phrased for an arbitrary bounded sequence with a pointwise $\limsup$
bound on a high-probability event; (2) prove that our feasibility/optimality bounds hold in
the required full-sequence form; (3) feed the latter into the former to obtain a single
$\omega$-independent threshold.

\begin{lemmaBox}{Uniformization of a pointwise eventual bound}{egorov-unif}
Let $(\Omega,\Sigma,\Pr)$ be a probability space, let $\Omega_0 \in \Sigma$ with
$\Pr(\Omega_0) \geq 1-\delta_0$ for some $\delta_0 \in [0,1)$, let $M > 0$, and let
$(f_n)_{n\in\mathbb N^+}$ be measurable functions $f_n : \Omega \to [-M,M]$ with
\[
\limsup_{n\to\infty} f_n(\omega) \;\leq\; \mathcal L \qquad \text{for every } \omega \in \Omega_0.
\]
Then, for every $\eta>0$ and every $\delta_1 \in (0,1)$, there exists
$N(\eta,\delta_1) \in \mathbb N^+$, depending on $\eta$, $\delta_1$ and the joint law of
$(f_n)_n$ only, such that
\[
\Pr\Bigl(\bigl\{\omega \in \Omega \,:\, f_n(\omega) < \mathcal L+\eta
\ \text{ for every } n \geq N(\eta,\delta_1)\bigr\}\Bigr)
\;\geq\; 1 - \delta_0 - \delta_1.
\]
\end{lemmaBox}

\begin{proof}
Define the tail-supremum envelope
$g_N(\omega) \triangleq \sup_{n \geq N} f_{n}(\omega) \in[-M,M]$, and put
\[
A_N \;\triangleq\; \bigl\{\omega\in\Omega \,:\, g_N(\omega) < \mathcal L+\eta\bigr\},\qquad N\in\mathbb N^+ .
\]
Each $g_N$ is measurable, being a countable supremum of measurable functions, so
$A_N\in\Sigma$. Since $(g_N)_N$ is \emph{nonincreasing} in $N$ --- a supremum over a
shrinking tail --- the sets $A_N$ are \emph{nondecreasing}: $A_N\subseteq A_{N+1}$.

Let $\omega\in\Omega_0$. Then $\lim_N g_N(\omega)=\limsup_n f_n(\omega)\le\mathcal L<\mathcal L+\eta$,
so $g_N(\omega)<\mathcal L+\eta$ for some finite $N$, i.e.\ $\omega\in\bigcup_N A_N$. Hence
$\Omega_0\subseteq\bigcup_N A_N$. By continuity of $\Pr$ from below along the nondecreasing
sequence $(A_N)_N$,
\[
\lim_{N\to\infty}\Pr(A_N)\;=\;\Pr\Bigl(\bigcup_N A_N\Bigr)\;\ge\;\Pr(\Omega_0)\;\ge\;1-\delta_0 .
\]
Choose $N(\eta,\delta_1)$ to be the least $N$ with $\Pr(A_N)\ge1-\delta_0-\delta_1$; it is
finite by the display, and depends only on $\eta$, $\delta_1$ and the joint law of
$(f_n)_n$, since $A_N$ does. On $A_{N(\eta,\delta_1)}$ we have, for every
$n\ge N(\eta,\delta_1)$,
$f_n\le g_{N(\eta,\delta_1)}<\mathcal L+\eta$, which is the claim.
\end{proof}

\begin{remark}[Subsequential hypotheses do not suffice]
\label{rem:regime-a-inadmissible}
The proof above uses its hypothesis in full: $\limsup_n f_n(\omega) \leq \mathcal L$
constrains the \emph{entire} tail $(f_n(\omega))_{n \geq N}$, for every $N$. A merely
subsequential hypothesis --- ``there is a subsequence $\{n_k(\omega)\}$ along which
$f_{n_k(\omega)}(\omega)$ is eventually below $\mathcal L+\eta$'' --- would not suffice: the
off-subsequence terms $f_n(\omega)$, $n \notin \{n_k(\omega)\}$, remain unconstrained, so the
envelope $g_N(\omega)$ could equal $M$ for every $N$, and then $\omega$ would belong to no
$A_N$, breaking the inclusion $\Omega_0\subseteq\bigcup_N A_N$ on which the whole argument
rests. This is exactly why Corollary~\ref{cor:multi-primal-guarantee} cannot be fed directly
into Lemma~\ref{lem:egorov-unif}, and why the full-sequence upgrade of
Lemma~\ref{lem:full-seq-feas} is necessary rather than cosmetic.
\end{remark}

\subsection{Proof of the Main Results}
\label{subsec:main-results}

\begin{lemmaBox}{Full-sequence eventual feasibility and near-optimality}{full-seq-feas}
Let Assumptions~\ref{ass:non-atomic}--\ref{ass:multi-argmin-recovery} be in effect. For every
DERM selection $(\widehat\mu_n,\nabla_n,\widehat\phi_n)_n$, every
$\omega \in \Omega^\star(\delta)$ and every $\eta>0$, there exists
$N(\omega,\eta) \in \mathbb N^+$ such that for \emph{every} $n \geq N(\omega,\eta)$ --- not
merely along some subsequence ---
\[
\mathcal R_i(\widehat\phi_n) < \epsinf(\Sdual) + \eta,
\quad i \in \mathbb N_m^+,
\qquad
\mathcal R_0(\widehat\phi_n) < P^\star_{\HK} + C_\mu\, \epsinf(\Sdual) + \eta.
\]
Equivalently,
$\limsup_{n} \mathcal R_i(\widehat\phi_n) \leq \epsinf(\Sdual)$ for
$i \in \mathbb N_m^+$ and
$\limsup_{n} \mathcal R_0(\widehat\phi_n) \leq P^\star_{\HK} + C_\mu\epsinf(\Sdual)$.
\end{lemmaBox}

\begin{proof}
Fix $\omega \in \Omega^\star(\delta)$ and $i \in \mathbb N_m^+$; suppose for contradiction
that $\limsup_n \mathcal R_i(\widehat\phi_n) > \epsinf(\Sdual)$. Then there exist $\eta > 0$
and a subsequence $\{n_j\}$ of the \emph{full} index set $\mathbb N^+$ along which
$\mathcal R_i(\widehat\phi_{n_j}) \geq \epsinf(\Sdual) + \eta$ for every $j$; this is a
real-analysis fact about the bounded real sequence $n \mapsto \mathcal R_i(\widehat\phi_n)$.

Re-index along $\{n_j\}$. The triple $(\widehat\mu_{n_j},\nabla_{n_j},\widehat\phi_{n_j})_j$
is again a valid DERM selection: $\widehat\mu_n$ is a valid empirical dual optimum,
$\nabla_n$ a valid supergradient with $\nabla_{n,i}\le0$ and
$\widehat\mu_{n,i}\nabla_{n,i}=0$, and $\widehat\phi_n$ a $1/n$-approximate attainer, at
\emph{every} sample size $n$ individually --- properties inherited term-by-term under any
re-indexing $j \mapsto n_j$, irrespective of how the indices $n_j$ were chosen. 

Corollary~\ref{cor:multi-primal-guarantee} is quantified over \emph{every} DERM selection,
and we apply it not to the original sequence but to this constructed one: it furnishes a
further subsequence $\{n_k\} \subseteq \{n_j\}$ along which
$\widehat\mu_{n_k} \to \mu$ for some $\mu \in \Sdual$, with
\[
\limsup_{k\to\infty} \mathcal R_i(\widehat\phi_{n_k}) \;\leq\; \epsinf(\mu)
\;\leq\; \epsinf(\Sdual).
\]
But $\{n_k\} \subseteq \{n_j\}$, so every term inherits the lower bound
$\epsinf(\Sdual) + \eta$ defining $\{n_j\}$, forcing
$\limsup_k \mathcal R_i(\widehat\phi_{n_k}) \geq \epsinf(\Sdual) + \eta$. Together these give
$\epsinf(\Sdual) + \eta \leq \epsinf(\Sdual)$, absurd. Hence
$\limsup_n \mathcal R_i(\widehat\phi_n) \leq \epsinf(\Sdual)$, the $\limsup$ now ranging over
the entire index set $\mathbb N^+$; and this is exactly restatable as a threshold
$N(\omega,\eta)$, since had no finite $N$ existed for some $\eta>0$, the full tail supremum
would remain at least $\epsinf(\Sdual)+\eta$ indefinitely. The identical argument at $i=0$,
invoking Corollary~\ref{cor:multi-primal-guarantee}(c), yields the objective bound.
\end{proof}

\begin{corollary}[Full-sequence convergence under a unique population dual optimum]
\label{cor:full-seq-conv}
In the setting of Lemma~\ref{lem:full-seq-feas}, suppose $\Sdual = \{\mu_{\HK}\}$ is a
singleton. Then $\epsinf(\Sdual) = \epsinf(\mu_{\HK})$, and for every
$\omega \in \Omega^\star(\delta)$, $\widehat\mu_n \to \mu_{\HK}$ along the full original
sequence. The bounds of Lemma~\ref{lem:full-seq-feas} hold verbatim with
$\epsinf(\mu_{\HK})$ in place of $\epsinf(\Sdual)$, and this is the tightest instance of that
lemma. We emphasize that the risk sequences $\mathcal R_i(\widehat\phi_n)$ themselves still
need not converge --- only the dual iterates $\widehat\mu_n$ do --- so this remains a
full-sequence, eventual guarantee rather than a pointwise-convergent one, and is an
admissible input to Lemma~\ref{lem:egorov-unif} on exactly the same footing as
Lemma~\ref{lem:full-seq-feas}.
\end{corollary}

\begin{proof}
Immediate from Lemma~\ref{lem:full-seq-feas}, the singleton case of the set extension of
$\epsinf$, and Lemma~\ref{lem:multi-dual-conv}, which places every subsequential limit of
$(\widehat\mu_n)_n$ in $\Sdual = \{\mu_{\HK}\}$; a bounded sequence all of whose
subsequential limits coincide converges to that common limit.
\end{proof}

\begin{corollary}[Uniform threshold]
\label{cor:uniform-threshold}
Under the hypotheses of Lemma~\ref{lem:full-seq-feas}, for every $\eta > 0$ and
$\delta_1 \in (0,1)$ there exists
$\widetilde n(\eta,\delta,\delta_1,\{\mathcal D_i\}_i) \in \mathbb N^+$, \emph{not depending
on $\omega$}, such that
\[
\Pr\Bigl(
\mathcal R_i(\widehat\phi_n) < \epsinf(\Sdual) + \eta \ (i\in\mathbb N_m^+)
\text{ and }
\mathcal R_0(\widehat\phi_n) < P^\star_{\HK} + C_\mu\epsinf(\Sdual) + \eta,
\ \forall n \geq \widetilde n
\Bigr)
\;\geq\; 1-\delta-\delta_1.
\]
If in addition $\Sdual$ is a singleton, the same holds with $\epsinf(\mu_{\HK})$ throughout,
via Corollary~\ref{cor:full-seq-conv}. No analogous uniformization is available for the
subsequential guarantee of Corollary~\ref{cor:multi-primal-guarantee}, by
Remark~\ref{rem:regime-a-inadmissible}.
\end{corollary}

\begin{proof}
For $i \in \mathbb{N}_m$ set $f_n^i(\omega) \triangleq \mathcal R_i(\widehat\phi_n(\omega)) \in [-M,M]$
(Assumption~\ref{ass:loss-regularity}(b)), $\mathcal L_i \triangleq \epsinf(\Sdual)$ for
$i \in \mathbb N_m^+$, and
$\mathcal L_0 \triangleq P^\star_{\HK} + C_\mu\epsinf(\Sdual)$. By
Lemma~\ref{lem:full-seq-feas}, $\limsup_n f_n^i(\omega) \leq \mathcal L_i$ for every
$\omega \in \Omega^\star(\delta)$, with $\Pr(\Omega^\star(\delta)) \geq 1-\delta$. Apply
Lemma~\ref{lem:egorov-unif} to each $f_n^i$ with $\Omega_0 = \Omega^\star(\delta)$,
$\delta_0 = \delta$, budget $\delta_1/(m+1)$, and the given $\eta$, obtaining
$N_i(\eta)$ with
$\Pr(\{f_n^i < \mathcal L_i+\eta \ \forall n \geq N_i(\eta)\}) \geq 1-\delta-\delta_1/(m+1)$.
Set $\widetilde n \triangleq \max_{0\le i\le m} N_i(\eta)$. The complement of the intersection
of the $m+1$ good events has probability at most
$\sum_{i=0}^m[\delta + \delta_1/(m+1)]$, which is wasteful; instead note each good event
contains $\Omega^\star(\delta)\cap A^{(i)}$ with $\Pr((A^{(i)})^c)\le\delta_1/(m+1)$, so the
intersection contains $\Omega^\star(\delta)\cap\bigcap_i A^{(i)}$, of probability at least
$1-\delta-\delta_1$ by a union bound over the $m+1$ sets $(A^{(i)})^c$ alone.
\end{proof}

\paragraph{Proof of Theorem~\ref{thm:informal-section3}.}
We assemble the pieces.

\emph{Part (i).} The four-way identity is Theorem~\ref{thm:dense-transfer} and is
deterministic. Work on $\Omega_{\mathrm{uc}}(\delta)$ and fix
$n \ge \max\{n_{\mathrm{rad}}(\xi/2,\delta), n^Q_{\mathrm{sl}}\}$. Then
$n \ge n_{\mathrm{uc}}(\delta)$, so Lemma~\ref{lem:multi-hp-conv} applies and
$\Delta_n(\delta) \le \xi/2$, while $Q_n \ge Q_{\mathrm{sl}}$ places $\phi^{\mathrm{sl}}$ in
$\Bq{Q_n}$; Remark~\ref{rem:slater-pathwise} then gives empirical strict feasibility pathwise
at no confidence cost, and Lemma~\ref{lem:bounded-multipliers} places
$\arg\max_{\mu\ge0} d_{Q_n}$, $\arg\max_{\mu\ge0}\widehat d_n$ and $\Sdual$ all inside
$\Lambda$.

We close the three terms of \eqref{equation:GDEC} in turn. The duality gap
$P^\star_{\HK} - D^\star_{\HK}$ vanishes by Theorem~\ref{thm:dense-transfer}. For the
ball/DERM gap, fix $\mu \in \Lambda$: Corollary~\ref{cor:lagrangian-dev} bounds
$|\widehat L_n(\phi,\mu) - L_{\HK}(\phi,\mu)|$ by $(1+\|\mu\|_1)\Delta_n(\delta)$ uniformly
over $\phi \in \Bq{Q_n}$, and since $\widehat d_n(\mu)$ and $d_{Q_n}(\mu)$ are infima over
the same set $\Bq{Q_n}$, taking infima on both sides gives
$|\widehat d_n(\mu) - d_{Q_n}(\mu)| \le (1+C_\mu) \Delta_n(\delta)$; as both suprema are
attained in $\Lambda$, $|\widehat D^\star_n - D^\star_{Q_n}| \le (1+C_\mu)\Delta_n(\delta)$.
For the dual/ball gap, $\Bq{Q_n} \subset \HK$ gives $d_{Q_n} \ge d_{\HK}$ pointwise, so
$|D^\star_{Q_n} - D^\star_{\HK}| \le \sup_{\mu\in\Lambda}[d_{Q_n}(\mu)-d_{\HK}(\mu)] = \mathfrak G_{Q_n}$
by attainment in $\Lambda$, and $\mathfrak G_{Q_n}\le\eps'$ whenever $Q_n\ge Q_{\eps'}$ as we have shown right after the definition of Tikhonov.

For the threshold, let $n \ge n^{\mathrm{val}}(\eps,\delta,\{\mathcal D_i\}_i)$. From
$n \ge n_{\mathrm{rad}}(\eps/(2(1+C_\mu)),\delta)$ we get
$\Delta_n(\delta) \le \eps/(2(1+C_\mu))$, so the third term of \eqref{equation:GDEC} is at
most $\eps/2$. From $n \ge n_{\mathrm{tik}}(\eps)$ we get $Q_n \ge Q_{\eps/2}$, so the
previous paragraph applied with $\eps'=\eps/2$ bounds the second term by $\eps/2$. Adding the
two halves gives (i). No attainment and no recovery condition were used.

\emph{Part (ii)} is Corollary~\ref{cor:multi-primal-guarantee}(a)--(c), read for the DERM
selection's own $(\widehat\mu_n,\nabla_n,\widehat\phi_n)$; it remains subsequential.

\emph{Part (iii)} is Lemma~\ref{lem:full-seq-feas}, which upgrades the subsequential content
of (ii) by applying it not to the original sequence but to an adversarially constructed
subsequence, the resulting contradiction constraining the entire tail.

\emph{Part (iv)} is Corollary~\ref{cor:uniform-threshold}. Lemma~\ref{lem:egorov-unif}
consumes the full-sequence statement (iii); the subsequential statement (ii) is not an
admissible input, by Remark~\ref{rem:regime-a-inadmissible}.

\emph{Part (v)} is Corollary~\ref{cor:full-seq-conv}, together with the same uniformization
applied to its sharpened constants.

\emph{Part (vi).} Suppose $\epsinf(\Sdual)=0$ and fix $\eps>0$ and $\delta\in(0,1)$. Apply
part~(iv) with confidence budgets $\delta/2$ in place of $\delta$ and $\delta_1=\delta/2$,
and with $\eta=\eps$. Since $\epsinf(\Sdual)=0$, the constants there are
$\mathcal L_i=0$ for $i\in\mathbb N^+_m$ and $\mathcal L_0=P^\star_{\HK}$, so part~(iv)
supplies $\widetilde n(\eps,\delta,\{\mathcal D_i\}_i)\triangleq
\widetilde n(\eps,\delta/2,\delta/2,\{\mathcal D_i\}_i)$, independent of $\omega$, with
\[
\Pr\Bigl(\mathcal R_i(\widehat\phi_n)<\eps\ (i\in\mathbb N_m^+)\ \text{ and }\ \mathcal R_0(\widehat\phi_n)<P^\star_{\HK}+\eps,\ \ \forall n\ge\widetilde n\Bigr)\;\ge\;1-\tfrac\delta2-\tfrac\delta2\;=\;1-\delta .
\]
This is verbatim the requirement of Definition~\ref{def:nonuniform-pacc} with
$\mathcal F=\HK$, the quantifier over $n$ being interior to the probability, and
$\widetilde n$ depending on the distributions only through the instance quantities
$\xi$, $Q_{\mathrm{sl}}$, $Q_{\eps/2}$ and the threshold produced by
Lemma~\ref{lem:egorov-unif}. Hence $\HK$ is agnostically universally PACC learnable. Since
$P^\star_{\HK}=P^\star_{\Gen}$ by Theorem~\ref{thm:dense-transfer} and the returned predictor
lies in $\Bq{Q_n}\subset\HK\subset\Gen$, the identical statement holds with $\Gen$ as target
class, and the learning is proper for both. \qed

\begin{remark}
For $\epsinf(\Sdual)>0$ the same argument degrades to a near-PACC guarantee, with an additive
slack of exactly $\epsinf(\Sdual)$ in the feasibility statement and
$C_\mu\epsinf(\Sdual)$ in the optimality statement. Unlike the residuals of
\cite{Chamon_CL2023,BarthakurChamon2025Equality}, this slack is an intrinsic population
constant rather than a modeling parameter. Its practical relevance is, however, subject to
the scope caveat of Remark~\ref{rem:scope}.
\end{remark}

\begin{figure}[p]
  \centering
  \IfFileExists{proof_flowchart.tex}
    {\resizebox{0.95\textwidth}{!}{\input{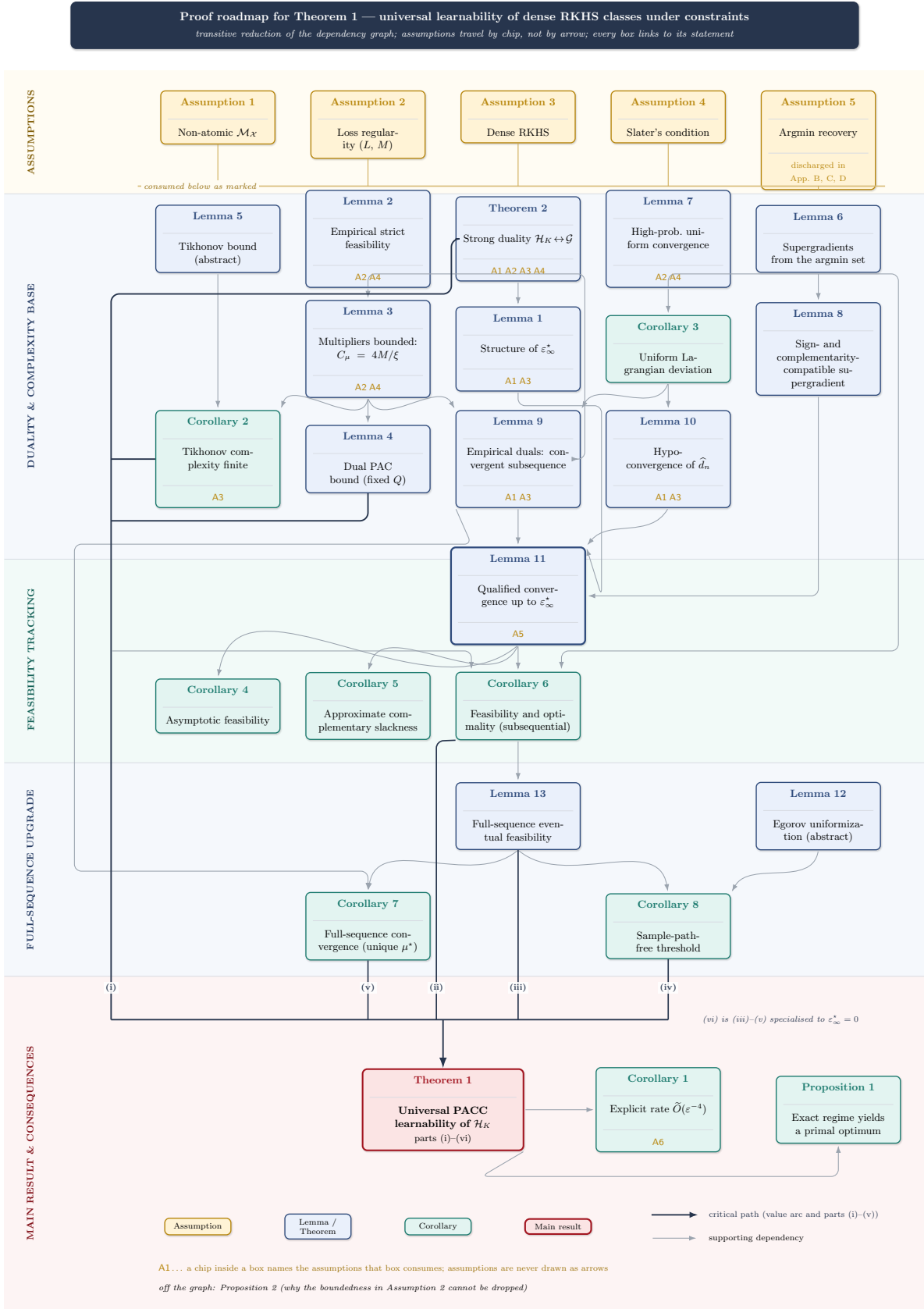}}}
    {\fbox{\parbox{0.9\textwidth}{\centering\vspace{2em}\textsf{[Proof dependency graph:
     \texttt{proof\_flowchart.tex} not found. See the note in the source for the required
     relabelling before recompiling.]}\vspace{2em}}}}
  \caption{Logical dependency graph of the proof of
         the main Theorem.}
  \label{fig:proof-flowchart}
\end{figure}

\section{Conclusion}
\label{sec:conclusion}

Statistical guarantees for constrained learning problems of the
form~\eqref{eq:primal formulation f} have so far been purchased either at the cost of
convexity or at the cost of a fixed, finite-complexity hypothesis class. The obstruction is
structural rather than technical because population strong duality requires classes that are
decomposable or densely decomposable, while generalization requires vanishing Rademacher
complexity, and no single fixed class has yet been proved to supply both 
(Section~\ref{subsec: the tension}). We resolved this by  posing the problem over a universal RKHS $\HK$, dense in a
decomposable envelope $\Gen$, so that density alone transfers exact strong duality to $\HK$
with no approximation error (Theorem~\ref{thm:dense-transfer} based on \cite{KalogeriasPougkakiotis22}); the empirical problem is
solved over norm balls $\Bq{Q_n}$ whose radius grows without bound but slower than $\sqrt n$,
so that uniform convergence holds at every $n$ while $\bigcup_n\Bq{Q_n}$ exhausts $\HK$.

The resulting $\mathcal H$-DERM procedure is universally PACC learnable
(Theorem~\ref{thm:informal-section3}). The optimal value is learned exactly with an explicit
high-probability threshold and fully explicit and polynomial in $1/\eps$ under a source
condition on the regularization path (Corollary~\ref{cor:explicit-rate}).  For the feasibility
and near-optimality of the solutions we introduced the population
closure--realization gap $\epsinf$ to control exact or near optimality of the problems. Unlike the residuals of
\cite{Chamon_CL2023,BarthakurChamon2025Equality}, which stem from a fixed parametric class
and cannot vanish, $\epsinf$ is an intrinsic population-geometric constant, provably zero
under benign dual geometry --- for instance whenever the dual is differentiable on its
optimal set (Lemma~\ref{lem:eps-star-properties}) --- in which case learnability is exact.
Because every threshold depends on the distributions through instance quantities, universal
rather than distribution-free learnability is the appropriate notion here; and this is not a
matter of convenience but of necessity, as Proposition~\ref{prop:no-distribution-free}
establishes. Throughout the appendices we have created examples and cases to show how one can reduce the assumption \ref{ass:multi-argmin-recovery} and also made examples to show both of the exact and near PAC regimes are applicable. 

\textbf{Limitations}
\begin{enumerate}[itemsep=2pt]
\item \emph{Recovery is verified by a general criterion only in the exact regime.}
Proposition~\ref{prop:sharpness}, the one checkable sufficient condition we give for
Assumption~\ref{ass:multi-argmin-recovery}, forces $\epsinf=0$ and so certifies only
branch~(vi) of Theorem~\ref{thm:informal-section3}. The near-PACC branch is inhabited ---
Appendix~\ref{sec:near-pacc-example} exhibits an instance with $\epsinf=1/4$, attained ---
but verified by hand, on a one-dimensional domain. A set-valued sharpness condition
tolerating ties in $\R^d$, certifying both regimes uniformly, is the principal open problem
left here; see Remark~\ref{rem:scope}.
\item \emph{The agnostic claim is confined to the value arc.} By
Proposition~\ref{prop:attainment}, in the exact regime
Assumption~\ref{ass:multi-argmin-recovery}(a) implies attainment of the constrained primal.
Part (i) of the main theorem is genuinely attainment-free; the feasibility arc is not.
\item \emph{Rates are explicit only conditionally.} Without
Assumption~\ref{ass:source} the Tikhonov radius $Q_\eps$ is finite by a compactness argument
supplying no modulus, and the uniformization threshold of Lemma~\ref{lem:egorov-unif} is
finite by a monotone-tail argument that likewise supplies none. The value arc is therefore
nonasymptotic up to a single scalar, and the feasibility arc is exact in its residual but
only eventual.

\item \emph{DERM is a specification, not an algorithm.} We analyze exact maximizers of the
empirical dual and exact minimizers of the empirical Lagrangian; turning this into an
implementable primal--dual scheme with its own convergence analysis is left open for future works.
\item \emph{Kernels are convenient, not essential.} The variable-complexity mechanism uses of
the RKHS only reflexivity, the reproducing property, density in an $L_p$ envelope, and a
controlled Rademacher complexity for balls. We expect it to transfer to other dense,
complexity-controllable families, but do not verify this.
\end{enumerate}

\noindent These are the natural directions forward. The variable-complexity mechanism
developed here is not intrinsically tied to kernels, and we expect it to be useful wherever
population-level structure and finite-sample control are required of the same hypothesis
class but pull in opposite directions.
\appendix

\section*{Roadmap of the appendices}
\addcontentsline{toc}{section}{Roadmap of the appendices}
The first appendix \ref{sec:universal-is-necessary} is about Universal learning and its necessity when we are posing it on a large class as RKHS.
The five appendices that follow it are a single sustained argument about
Assumption~\ref{ass:multi-argmin-recovery}, the empirical recovery condition,
which is the only hypothesis of Theorem~\ref{thm:informal-section3} that is not a
structural property of the losses, the kernel or the sampling scheme, and which is
therefore the only one a reader is entitled to be suspicious of. For the whole examples to be easier to be presented, we will consider examples where there exists a constant $C_\infty<\infty$ that the norm of all the solutions of type \ref{ass:multi-argmin-recovery} are bounded above by it. In our rates, and conditions we will refer to it as $C_\infty$ and one can verify, with this strengthening assumption, all the proofs would still follow.
Appendix~\ref{sec:on-last-assumption} reduces it, in the case of an almost-everywhere
unique pointwise minimizer, to a sharpness condition on a function of one real
variable, and then records the price of that reduction: the condition it verifies
forces $\epsinf=0$, so the general sufficient condition we can state is blind to
exactly the regime that parts~(ii)--(v) of the theorem were written for.
Appendix~\ref{sec:composite-predictors} asks what the framework becomes when the RKHS
function is a score read through a fixed link, as it is in classification; the answer
is that loss regularity becomes automatic on a much wider class, that attainment
becomes visibly a realizability requirement and nothing more, that sharpness can be
imported from a curvature statement in the prediction variable, and that margin
losses are excluded by a boundary obstruction --- with two fully verified instances,
both nonconvex in the RKHS variable, both landing in the exact regime.
Appendix~\ref{sec:near-pacc-example} settles the question the first appendix leaves
open by exhibiting an instance in which every assumption holds and
$\epsinf=\tfrac14$, so the near-PACC branch is inhabited; the mechanism is a tie in
the conditional integrand, which the envelope resolves by splicing and which
continuity forbids the RKHS to resolve at all, and the residual is shown to be
attained rather than merely bounded.
Appendix~\ref{sec:worked-example} carries one concrete fairness/coverage problem
through every assumption with explicit constants, in the exact regime, as a
counterpart to the previous appendix.
Appendix~\ref{sec:solving} records the computational shape of the ball-constrained
DERM problem and isolates the one caveat that the theory does not remove, namely that
the inner minimization is required to be solved globally.

\section{Universal Learning is the necessary condition}\label{sec:universal-is-necessary}
\begin{propBox}{Nonexistence of a distribution-free threshold}{no-distribution-free}
Let $m=1$, $\mathcal X=[0,1]$, $\mathcal Y=\{0,1\}$, and fix the losses
\[
\ell_0(u,y) \;\triangleq\; \tfrac12\bigl|\clip(u) - (2y-1)\bigr|,
\qquad
\ell_1(u,y) \;\triangleq\; -1 ,
\]
where $\clip$ denotes truncation to $[-1,1]$. These satisfy
Assumption~\ref{ass:loss-regularity} with $L=\tfrac12$, $M=1$, and
Assumption~\ref{ass:slater} holds with margin $\xi=1$ at every $\phi^{\mathrm{sl}}\in\HK$
and for \emph{every} choice of $\{\mathcal D_i\}$. Let $\HK$ be any RKHS satisfying
Assumption~\ref{ass:dense-rkhs}. Then, for $\eps_0 \triangleq 1/32$ and
$\delta_0 \triangleq 1/8$, there is \emph{no} threshold $\widetilde n(\eps_0,\delta_0)$
independent of $\{\mathcal D_i\}$ for which any learner $\mathcal A$ satisfies the
conclusion of Definition~\ref{def:nonuniform-pacc}: for every learner $\mathcal A$ and
every $n\in\mathbb N^+$ there exist distributions $\{\mathcal D_i\}$ obeying
Assumptions~\ref{ass:non-atomic}--\ref{ass:slater} for which the displayed guarantee fails
at sample size $n$.
\end{propBox}

\begin{proof}
Since $\ell_1\equiv-1$, we have $\mathcal R_1\equiv-1$, so the constraint is vacuous and
strictly satisfied with margin $\xi=1$ by every predictor; Assumption~\ref{ass:slater}
holds regardless of $\mathcal D_1$, which we may take equal to $\mathcal D_0$. The problem
therefore reduces to unconstrained agnostic learning of $\ell_0$.

Write $\eta(x)\triangleq \Pr(Y=1\mid X=x)$. For $u\in[-1,1]$,
\[
\E\bigl[\ell_0(u,Y)\mid X=x\bigr]
=\tfrac{\eta(x)(1-u)}{2}+\tfrac{(1-\eta(x))(1+u)}{2}
=\tfrac12-\tfrac{u\,(2\eta(x)-1)}{2},
\]
which is affine in $u$ and minimized at $u=\mathrm{sgn}(2\eta(x)-1)$ (with the convention
$\mathrm{sgn}(0)=+1$). Hence the pointwise minimum is $\tfrac12-\tfrac{|2\eta(x)-1|}{2}$,
and since $\ell_0$ depends on $\phi$ only through $\clip(\phi)$ and $\HK$ is dense in
$L_p(\mathcal M_{\mathcal X},\R)$ with $\mathcal R_0$ continuous there
(Section~\ref{subsec:envelope}), $P^\star_{\HK}=\tfrac12-\tfrac12\E|2\eta(X)-1|$.

Now specialize to distributions with $\eta(x)\in\{0,1\}$ for a.e.\ $x$, for which
$P^\star_{\HK}=0$. For any $\phi$, writing $v \triangleq \clip(\phi(x))\in[-1,1]$ and
$s(x)\triangleq \mathrm{sgn}(2\eta(x)-1)\in\{\pm1\}$,
\[
\mathcal R_0(\phi)
=\E\Bigl[\tfrac{1 - v\,s(X)}{2}\Bigr]
\;\ge\;
\tfrac12\,\E\bigl[\mathbf 1\{v\,s(X)\le 0\}\bigr]
\;=\;\tfrac12\,\mathcal L_{01}\bigl(\mathrm{sgn}\circ\phi\bigr),
\]
because $1-vs \ge 1$ whenever $vs\le0$ and $1-vs\ge0$ always. Here
$\mathcal L_{01}$ is the $0$--$1$ risk of the plug-in classifier $x\mapsto
\mathrm{sgn}(\phi(x))$, whose Bayes risk is $0$ under our specialization. Thus
\begin{equation}\label{eq:nfl-reduction}
\mathcal R_0(\phi)-P^\star_{\HK}\;\ge\;\tfrac12\,\mathcal L_{01}(\mathrm{sgn}\circ\phi).
\end{equation}

Fix a learner $\mathcal A$ and $n\in\mathbb N^+$, and consider the classification rule
$g_n \triangleq \mathrm{sgn}\circ\mathcal A(\{\widehat{\mathcal D}^n_i\}_i)$. By the
No-Free-Lunch theorem \cite[Thm.~5.1]{ShalevShwartzBenDavid2014}, applied on any domain of
cardinality at least $2n$, there is a distribution $\nu_n$ on that domain, realizable
(Bayes risk $0$), such that
$\Pr\bigl(\mathcal L_{01}(g_n)\ge 1/8\bigr)\ge 1/7$ over the draw of the $n$-sample.
The construction places $\nu_n$ on a finite set $C$ with $|C|=2n$; to meet
Assumption~\ref{ass:non-atomic} we \emph{smear} it: partition $[0,1]$ into $2n$ disjoint
intervals $\{I_c\}_{c\in C}$ of equal length, let $\mathcal D_{\mathcal X,0}$ be the
uniform (hence nonatomic) distribution on $[0,1]$ redistributing the mass $\nu_n(\{c\})$
uniformly over $I_c$, and set $\eta$ constant on each $I_c$, equal to the label of $c$.
Conditionally on the interval index, the within-interval coordinate is independent of $Y$
and therefore ancillary; consequently no rule using the full covariate can outperform the
best rule using only the interval index, and the lower bound transfers verbatim.

Combining with \eqref{eq:nfl-reduction}, with probability at least $1/7 > \delta_0$ we have
$\mathcal R_0(\widehat\phi_n)-P^\star_{\HK}\ge 1/16 > \eps_0$. Hence the guarantee of
Definition~\ref{def:nonuniform-pacc} fails at $(\eps_0,\delta_0)$ for this $n$. As $n$ was
arbitrary and the distribution depends on $n$, no $\{\mathcal D_i\}$-free threshold
$\widetilde n(\eps_0,\delta_0)$ can exist.
\end{proof}

\begin{remark}[What this does and does not say]\label{rem:nfl-scope}
Proposition~\ref{prop:no-distribution-free} is a statement about the \emph{class} $\HK$
being of unbounded capacity:
it shows the distribution dependence is intrinsic and would persist even if every
constrained-specific difficulty were removed. It does \emph{not} show that the Slater
margin $\xi$ creates an information-theoretic barrier; we are not aware of such a result. Whether the $\xi$-dependence of
these constants is necessary is an interesting open question that we do not resolve.
\end{remark}
\section{On the Empirical Recovery Condition}\label{sec:on-last-assumption}

Recall from \eqref{eq:lagrangian-integrand} that, for every $\mu\ge0$ and $\phi\in\HK$,
\[
L_{\HK}(\phi,\mu)=\int_{\mathcal X}\bar\ell_\mu\bigl(\phi(x),x\bigr)\,\mathrm d\mathcal M_{\mathcal X}(x),
\qquad
\bar\ell_\mu(u,x)=\sum_{i\in\mathbb N_m}\mu_i\,w_i(x)\,\bar\ell_i(u,x),
\quad \mu_0\triangleq1 .
\]
The Lagrangian is a Nemytskii functional of the scalar family $\bar\ell_\mu(\cdot,x)$, and
every question about its argmin is a question about that family. The following proposition
reduces Assumption~\ref{ass:multi-argmin-recovery} to a condition on a single real variable,
in the case where the pointwise minimizer is unique. The tie-bearing case, which the
proposition deliberately excludes, is treated by hand in
Appendix~\ref{sec:near-pacc-example}.

\begin{propBox}{Pointwise sharpness implies empirical recovery}{sharpness}
Let $\mathcal X$ be a compact metric space, let $K$ be continuous with
$\kappa^2 = \sup_x K(x,x)<\infty$ and with $\HK$ dense in $C(\mathcal X)$ for
$\|\cdot\|_\infty$ (a strengthening of Assumption~\ref{ass:dense-rkhs}(b), which requires
only density in $\Gen$), and let Assumption~\ref{ass:loss-regularity} hold with constants
$L$ (Lipschitz) and $M$ (bound). Let $Q_n\uparrow\infty$ with $Q_n=o(n^{1/4})$ ---
see Remark~\ref{rem:schedule} --- and $\widehat\Phi_n(\mu)=\argmin_{\phi\in\Bq{Q_n}}\widehat L_n(\phi,\mu)$.
Fix $\mu_{\HK}\in\Sdual$ and suppose there are an open neighbourhood $U\ni\mu_{\HK}$,
constants $\gamma_0>0$, $\Delta_0>0$, $p\ge1$, and, for each $\mu\in U$, a function
$\phi^\sharp_\mu\in\HK$ with $\|\phi^\sharp_\mu\|_{\HK}\le C_\infty$ such that, writing
$\varrho(t)\triangleq\min\{\gamma_0t^p,\Delta_0\}$,
\begin{equation}\label{eq:sharpness}
\bar\ell_\mu(u,x)-\bar\ell_\mu\bigl(\phi^\sharp_\mu(x),x\bigr)\;\ge\;\varrho\bigl(|u-\phi^\sharp_\mu(x)|\bigr)
\qquad\text{for }\mathcal M_{\mathcal X}\text{-a.e. }x\text{ and all }u\in\R .
\end{equation}
Then Assumption~\ref{ass:multi-argmin-recovery} holds at $\mu_{\HK}$ with witness
$\phi_{\HK}=\phi^\sharp_{\mu_{\HK}}$. Precisely, $\Phi_{\HK}(\mu)$ is exactly the set of
elements of $\HK$ agreeing with $\phi^\sharp_\mu$ $\mathcal M_{\mathcal X}$-a.e.\ for every
$\mu\in U$, and there is an event of probability one on which, for every strictly increasing
$\{n_k\}$ and every sequence $\mu_{n_k}\in\arg\max_{\mu\ge0}\widehat d_{n_k}(\mu)$ with
$\|\mu_{n_k}\|_1\le C_\mu$ and $\mu_{n_k}\to\mu_{\HK}$, all sufficiently large $k$ satisfy
\begin{multline}\label{eq:recovery-rate}
\sup_{\phi\in\widehat\Phi_{n_k}(\mu_{n_k})}\bigl\|\phi-\phi^\sharp_{\mu_{\HK}}\bigr\|_{L_1(\widehat{\mathcal M}^{n_k}_{\mathcal X})}
\;\le\;
\Bigl(\tfrac{2\eps_{n_k}}{\gamma_0}\Bigr)^{1/p} +\\  \tfrac{4\kappa Q_{n_k}\eps_{n_k}}{\Delta_0}
+\Bigl(\tfrac{2(m+1)M\|\mu_{n_k}-\mu_{\HK}\|_1}{\gamma_0}\Bigr)^{1/p}
+8\kappa Q_{n_k}\sqrt{\tfrac{\log\bigl((m+1)n_k\bigr)}{n_k}},
\end{multline}
where $\eps_n \triangleq (1+C_\mu)\max_{i}\sup_{\phi\in\Bq{Q_n}}|\widehat{\mathcal R}_{i,n}(\phi)-\mathcal R_i(\phi)|$.
In particular the right-hand side vanishes, so the limit in part (b) is zero.
\end{propBox}

\begin{proof}
Density of $\HK$ in $C(\mathcal X)$ gives density in $L_1(\mathcal M_{\mathcal X})$ by
regularity of $\mathcal M_{\mathcal X}$, while \eqref{eq:lagrangian-integrand} and the
Lipschitz property give
$|L(\phi,\mu)-L(\psi,\mu)|\le L(m+1)(1+\|\mu\|_1)\|\phi-\psi\|_{L_1(\mathcal M_{\mathcal X})}$,
using $w_i\le m+1$ (valid since $\mathcal D_{\mathcal X,i}\le(m+1)\mathcal M_{\mathcal X}$).
The infimum of $L(\cdot,\mu)$ over $\HK$ therefore equals its infimum over
$L_1(\mathcal M_{\mathcal X})$, which by measurable selection equals
$\int_{\mathcal X}\min_u\bar\ell_\mu(u,x)\,\mathrm d\mathcal M_{\mathcal X}(x)$. By
\eqref{eq:sharpness} this pointwise minimum is attained at the unique point
$\phi^\sharp_\mu(x)$, and $\phi^\sharp_\mu\in\HK$ by hypothesis, so $\phi^\sharp_\mu$
minimizes $L(\cdot,\mu)$ over $\HK$ and any other minimizer must realize the pointwise
minimum a.e., whence the stated description of $\Phi_{\HK}(\mu)$.

Part~(a) follows: $\phi^\sharp_\mu\in\Phi_{\HK}(\mu)\cap\Bq{C_\infty}$, and the requirement
that the infimum defining $\Theta^\infty_\nabla$ be unchanged upon restriction to
$\Phi_{\HK}(\mu)\cap\Bq{C_\infty}$ is automatic, since all elements of $\Phi_{\HK}(\mu)$
agree a.e.\ and hence have identical risk vectors.

Two remarks on the quantifiers before the estimates. First,
Assumption~\ref{ass:multi-argmin-recovery}(b) ranges over \emph{every} reference point
$\phi_{\HK}\in\Phi_{\HK}(\mu_{\HK})\cap\Bq{C_\infty}$, and the norm
$\|\cdot\|_{L_1(\widehat{\mathcal M}^n_{\mathcal X})}$ is not absolutely continuous with
respect to $\mathcal M_{\mathcal X}$; but $K$ is continuous, so elements of $\HK$ are
continuous, and two continuous functions agreeing $\mathcal M_{\mathcal X}$-a.e.\ agree on
$\supp\mathcal M_{\mathcal X}$. Since $\mathcal D_{\mathcal X,i}\ll\mathcal M_{\mathcal X}$,
every sample point lies in $\supp\mathcal M_{\mathcal X}$ almost surely, so all admissible
reference points induce the same empirical $L_1$ distance and it suffices to argue for
$\phi^\sharp_{\mu_{\HK}}$. Second, the events constructed below are defined without
reference to any particular subsequence, which is what licenses the
``along every $\{n_k\}$'' form of part~(b).

\emph{Perturbation of the witness in $\mu$.}
Comparing two dual variables in \eqref{eq:lagrangian-integrand} and using $w_i\le m+1$ and
$|\bar\ell_i|\le M$ gives
$\sup_{u,x}|\bar\ell_\mu(u,x)-\bar\ell_{\mu'}(u,x)|\le(m+1)M\|\mu-\mu'\|_1$. Evaluating
\eqref{eq:sharpness} for $\mu_{\HK}$ at $u=\phi^\sharp_\mu(x)$, using that
$\phi^\sharp_\mu$ minimizes $\bar\ell_\mu(\cdot,x)$ and inserting the previous bound twice,
yields $\varrho(|\phi^\sharp_\mu(x)-\phi^\sharp_{\mu_{\HK}}(x)|)\le2(m+1)M\|\mu-\mu_{\HK}\|_1$
for a.e.\ $x$, so once the right-hand side is below $\Delta_0$,
\begin{equation}\label{eq:witness-shift}
\bigl\|\phi^\sharp_\mu-\phi^\sharp_{\mu_{\HK}}\bigr\|_\infty
\;\le\;\Bigl(\tfrac{2(m+1)M\|\mu-\mu_{\HK}\|_1}{\gamma_0}\Bigr)^{1/p},
\end{equation}
uniformly in $x$ and with no dependence on $Q_n$.

\emph{The two almost-sure events.}
Let $\Omega_{\mathrm{uc}}$ be the event on which
$\eps_n = O\bigl((1+C_\mu)(L\kappa Q_nn^{-1/2}+M\sqrt{\log n/n})\bigr)$, and in particular
$\eps_n\to0$. It has probability one: symmetrization and Talagrand contraction give
$\mathfrak R_n(\ell_i\circ\Bq{Q_n})\le L\kappa Q_n/\sqrt n$, the bounded-difference constant
$2M/n$ is free of $Q_n$, and Borel--Cantelli at confidence $n^{-2}$ applies. Note that this
step requires only $Q_n=o(\sqrt n)$; the strictly smaller budget $Q_n=o(n^{1/4})$ is
consumed later, by the escape term.

Let $\Omega_{\mathrm{gc}}$ be the event on which, for all large $n$,
\[
\sup\Bigl\{\bigl|\,\|g\|_{L_1(\widehat{\mathcal M}^n_{\mathcal X})}-\|g\|_{L_1(\mathcal M_{\mathcal X})}\bigr|
\ :\ g=\phi-\psi,\ \phi\in\Bq{Q_n},\ \psi\in\Bq{C_\infty}\Bigr\}
\;\le\;8\kappa Q_n\sqrt{\tfrac{\log((m+1)n)}{n}} .
\]
It has probability one. Indeed $\widehat{\mathcal M}^n_{\mathcal X}$ is the average of the
$m+1$ empirical measures, so it suffices to control each stratum. On each, the class
$\{|\phi-\psi|\}$ has envelope $B_n\triangleq\kappa(Q_n+C_\infty)\le2\kappa Q_n$ (for
$Q_n\ge C_\infty$) and, by Talagrand contraction along $t\mapsto|t|$, Rademacher complexity
at most $\kappa(Q_n+C_\infty)/\sqrt n$. Symmetrization contributes
$2\kappa(Q_n+C_\infty)/\sqrt n\le4\kappa Q_n/\sqrt n$, and McDiarmid at confidence
$\delta_n=(m+1)^{-1}n^{-2}$ contributes
$B_n\sqrt{2\log(1/\delta_n)/n}\le 2\kappa Q_n\cdot2\sqrt{\log((m+1)n)/n}$; the two together
are at most $8\kappa Q_n\sqrt{\log((m+1)n)/n}$ for $n\ge3$. Borel--Cantelli over
$\sum_n(m+1)\delta_n<\infty$ closes the argument. Set
$\Omega_{\mathrm{fbl}}\triangleq\Omega_{\mathrm{uc}}\cap\Omega_{\mathrm{gc}}$.

Neither event uses \eqref{eq:sharpness}: both are consequences of
Assumption~\ref{ass:loss-regularity}, the kernel bound and the schedule alone. This is what
allows Appendix~\ref{sec:near-pacc-example}, where \eqref{eq:sharpness} fails, to reuse them
verbatim.

\emph{The estimate.}
Fix $\omega\in\Omega_{\mathrm{fbl}}$, a sequence $\mu_{n_k}\to\mu_{\HK}$ in $U$ with
$\|\mu_{n_k}\|_1\le C_\mu$, and $\widehat\phi_{n_k}\in\widehat\Phi_{n_k}(\mu_{n_k};\omega)$.
For $k$ large enough that $Q_{n_k}\ge C_\infty$ the comparator $\phi^\sharp_{\mu_{n_k}}$ is
admissible, so minimality of $\widehat\phi_{n_k}$ and two applications of the uniform bound
give
\[
L(\widehat\phi_{n_k},\mu_{n_k})\le\widehat L_{n_k}(\widehat\phi_{n_k},\mu_{n_k})+\eps_{n_k}
\le\widehat L_{n_k}(\phi^\sharp_{\mu_{n_k}},\mu_{n_k})+\eps_{n_k}
\le L(\phi^\sharp_{\mu_{n_k}},\mu_{n_k})+2\eps_{n_k}.
\]
Writing $\Delta_{n_k}\triangleq|\widehat\phi_{n_k}-\phi^\sharp_{\mu_{n_k}}|$ and subtracting
\eqref{eq:lagrangian-integrand} at $\phi^\sharp_{\mu_{n_k}}$, the sharpness inequality
\eqref{eq:sharpness} at $\mu_{n_k}$ integrates to
$\int\varrho(\Delta_{n_k})\,\mathrm d\mathcal M_{\mathcal X}\le2\eps_{n_k}$. Since
$\varrho(t)=\gamma_0\min\{t,t_0\}^p$ with $t_0\triangleq(\Delta_0/\gamma_0)^{1/p}$, Jensen
gives $\|\min\{\Delta_{n_k},t_0\}\|_{L_1(\mathcal M_{\mathcal X})}\le(2\eps_{n_k}/\gamma_0)^{1/p}$,
while $\varrho(t)=\Delta_0$ for $t\ge t_0$ gives
$\mathcal M_{\mathcal X}(\Delta_{n_k}>t_0)\le2\eps_{n_k}/\Delta_0$. As
$\Delta_{n_k}\le\kappa(Q_{n_k}+C_\infty)\le2\kappa Q_{n_k}$ pointwise, splitting the integral
yields
$\|\Delta_{n_k}\|_{L_1(\mathcal M_{\mathcal X})}\le(2\eps_{n_k}/\gamma_0)^{1/p}+4\kappa Q_{n_k}\eps_{n_k}/\Delta_0$.


It remains to change measure and comparator. On $\Omega_{\mathrm{gc}}$ the first two terms of
\eqref{eq:recovery-rate} bound $\|\Delta_{n_k}\|_{L_1(\widehat{\mathcal M}^{n_k}_{\mathcal X})}$
up to the additive fourth term, and \eqref{eq:witness-shift} together
with $\|\cdot\|_{L_1(\widehat{\mathcal M}^{n_k}_{\mathcal X})}\le\|\cdot\|_\infty$ replaces
$\phi^\sharp_{\mu_{n_k}}$ by $\phi^\sharp_{\mu_{\HK}}$ at the cost of the third term. This
proves \eqref{eq:recovery-rate} for every element of $\widehat\Phi_{n_k}(\mu_{n_k};\omega)$,
hence for the supremum. Finally
$Q_{n}\eps_{n}=O(Q_n^2n^{-1/2}+Q_n\sqrt{\log n/n})\to0$ and $Q_n\sqrt{\log n/n}\to0$
\emph{precisely under} $Q_n=o(n^{1/4})$ --- the first requirement being literally
$Q_n^2n^{-1/2}\to0$ --- and $\|\mu_{n_k}-\mu_{\HK}\|_1\to0$, so the
right-hand side vanishes.
\end{proof}

\begin{remark}[The quantifiers in part~(b), and what each result supplies]\label{rem:part-b-quantifiers}
Two readings of Assumption~\ref{ass:multi-argmin-recovery}(b) must be kept apart, because the
two results of this paper that verify it supply different ones. Write
$W\triangleq\Phi_{\HK}(\mu_{\HK})\cap\Bq{C_\infty}$ for the set of admissible witnesses. The
readings are
\begin{enumerate}[label=\textnormal{(\Alph*)},itemsep=2pt,topsep=2pt]
\item \emph{(sup form)} for every $\phi_{\HK}\in W$,
$\ \sup_{\phi\in\widehat\Phi_{n_k}(\mu_{n_k})}\|\phi-\phi_{\HK}\|_{L_1(\widehat{\mathcal M}^{n_k}_{\mathcal X})}\to0$;
\item \emph{(matching form)} for every $\phi_{\HK}\in W$ there are
$\widehat\phi_{n_k}\in\widehat\Phi_{n_k}(\mu_{n_k})$ with
$\|\widehat\phi_{n_k}-\phi_{\HK}\|_{L_1(\widehat{\mathcal M}^{n_k}_{\mathcal X})}\to0$, and
conversely $\sup_{\phi\in\widehat\Phi_{n_k}(\mu_{n_k})}\inf_{\phi_{\HK}\in W}\|\phi-\phi_{\HK}\|_{L_1(\widehat{\mathcal M}^{n_k}_{\mathcal X})}\to0$.
\end{enumerate}
Form~(B) is two-sided convergence of $\widehat\Phi_{n_k}(\mu_{n_k})$ to $W$ in the empirical
$L_1$ Hausdorff distance; (A) is (B) together with the requirement that $W$ be a single
$\mathcal M_{\mathcal X}$-equivalence class. Under the hypotheses of
Proposition~\ref{prop:sharpness} that requirement holds automatically, and the proposition
proves (A). The instance of Appendix~\ref{sec:near-pacc-example} \emph{cannot} satisfy (A):
there $W$ consists of two constants at mutual $L_1$ distance $1$, and
Proposition~\ref{prop:example-recovery}(iii) shows that both are approached by exact
empirical minimizers, so the supremum in (A) is at least $1-o(1)$ for either witness. It does
satisfy (B). Assumption~\ref{ass:multi-argmin-recovery}(b) is therefore to be read as (B)
throughout; Proposition~\ref{prop:sharpness} delivers strictly more than is needed, and the
reader should not infer from its statement that the assumption fails whenever the population
argmin fails to be a singleton.
\end{remark}

\begin{remark}[Reading of the sharpness condition]\label{rem:sharpness-reading}
Condition \eqref{eq:sharpness} constrains a function of one real variable and is blind to
convexity: $\ell_i(\cdot,y)$ may be bounded, saturating, multimodal or discontinuous in $y$,
since only the conditional mixture $\bar\ell_\mu(\cdot,x)$ is required to separate its
minimizer, and averaging over $y$ and over the $m+1$ components is precisely what produces
that separation. The exponent $p$ records the identifiability regime: $p=2$ when the
conditional integrand has nondegenerate curvature at an interior minimizer, and $p=1$ under a
Massart-type margin condition; it enters the rate \eqref{eq:recovery-rate} rather than the
qualitative conclusion. The cap $\Delta_0$ makes the condition compatible with bounded
losses, and is exactly the price of the growing balls: the escape term
$4\kappa Q_n\eps_n/\Delta_0$ is the contribution of the vanishing set on which the empirical
minimizer leaves the identifiable window, and it is what fixes the admissible growth budget
$Q_n=o(n^{1/4})$.

Condition \eqref{eq:sharpness} fails exactly in the presence of ties, that is when
$\bar\ell_{\mu_{\HK}}(\cdot,x)$ has two or more global minimizers on a set of positive
$\mathcal M_{\mathcal X}$-measure. What fails there is the \emph{sufficient condition}, not
necessarily Assumption~\ref{ass:multi-argmin-recovery} itself: with ties, the population
argmin over the \emph{envelope} $\Gen$ contains every measurable splice of the competing
branches, and single-witness $L_1$ recovery is hopeless relative to $\Gen$; but over $\HK$
the branches may be separated by the smoothness of the kernel, in which case
$\Phi_{\HK}(\mu_{\HK})$ collapses onto finitely many branch selections and part~(b), in the
form~(B) of Remark~\ref{rem:part-b-quantifiers}, can hold for each of them.
Appendix~\ref{sec:near-pacc-example} exhibits exactly such an instance, and
Lemma~\ref{lem:splice-exclusion} isolates the mechanism.
\end{remark}

\begin{remark}[What \eqref{eq:recovery-rate} actually gives under the admissible schedules]\label{rem:effective-rate}
The bound \eqref{eq:recovery-rate} is written as a rate, and it is worth saying plainly how
fast it is. Take $p=2$ and the schedule $Q_n=n^{1/4}/\log(e+n)$ of Remark~\ref{rem:schedule},
which is the one used in every instance below. Then $\eps_n\asymp Q_nn^{-1/2}\asymp
n^{-1/4}/\log n$ and, using $\|\mu_{n}-\mu_{\HK}\|_1=O(\eps_n)$ as in
Proposition~\ref{prop:example-recovery}(i), the four terms of \eqref{eq:recovery-rate} are, in
order,
\[
\asymp\ \gamma_0^{-1/2}n^{-1/8},\qquad
\asymp\ \frac{4\kappa}{\Delta_0\log^2 n},\qquad
\asymp\ \gamma_0^{-1/2}n^{-1/8},\qquad
\asymp\ \frac{n^{-1/4}}{\sqrt{\log n}} .
\]
The escape term dominates, and \eqref{eq:recovery-rate} therefore vanishes at rate
$\Theta(\log^{-2}n)$, not at any polynomial rate. This is a property of the growing radius
and not of the sharpness condition: at a fixed radius $Q$ one has $\eps_n\asymp Qn^{-1/2}$ and
the escape term is $O(Q^2n^{-1/2})$, so all four terms are polynomial; the same holds up to
logarithmic factors for any $Q_n$ growing polylogarithmically. The aggressive $n^{1/4}$
schedule is what buys the approximation error and what costs the rate, and the two cannot be
tuned independently. We do not know whether the escape term is an artifact of the proof; the
same phenomenon reappears, with the same origin, in
Proposition~\ref{prop:example-recovery}, where $\tau_n\asymp(\log n)^{-3/2}$.
\end{remark}

\begin{remark}[Scope of the two regimes of Theorem~\ref{thm:informal-section3}]\label{rem:scope}
We record plainly what is and is not covered by the two branches of
Theorem~\ref{thm:informal-section3}.

Suppose the sufficient condition \eqref{eq:sharpness} holds at $\mu_{\HK}$. Then the
pointwise minimizer of $\bar\ell_{\mu_{\HK}}(\cdot,x)$ is a.e.\ unique, so by
Lemma~\ref{lem:eps-star-properties}(v) the superdifferential $\partial d_{\HK}(\mu_{\HK})$ is
a singleton and hence $\epsinf(\mu_{\HK})=0$. In other words,
\emph{Proposition~\ref{prop:sharpness} can only ever certify the exact regime}, part~(vi):
the general sufficient condition we are able to state is blind to precisely the situation
that parts~(ii)--(v) are designed for. This is a limitation of the sufficient condition, and
for a while we took it to be a limitation of the theory.

It is not. Appendix~\ref{sec:near-pacc-example} constructs an instance satisfying
Assumptions~\ref{ass:non-atomic}--\ref{ass:multi-argmin-recovery} with
$\epsinf(\Sdual)=1/4>0$, so parts~(ii)--(v) are inhabited and the near-PACC branch is not
vacuous. The mechanism is the one anticipated above: ties make the envelope argmin
$\Phi_{\Gen}(\mu^\star)$ decomposable, hence its risk image a nondegenerate segment, while
$\Phi_{\HK}(\mu^\star)$ is confined by continuity to the two endpoints --- and the empirical
argmin set inherits the same dichotomy, because a function of bounded RKHS norm cannot
traverse between branches quickly enough to be an exact empirical minimizer. The
quantitative form of that last statement (Lemma~\ref{lem:splice-exclusion}) is a
compatibility condition between the derivative bound of the kernel and the radius schedule,
and it is satisfied by the schedules of Remark~\ref{rem:schedule}.

Two things remain genuinely open, and we state them as such.
\begin{enumerate}[label=(\alph*),itemsep=2pt]
\item \emph{A general tie-tolerant sufficient condition.} What
Appendix~\ref{sec:near-pacc-example} supplies is one instance, verified by hand, on a
one-dimensional domain with a kernel possessing a bounded derivative diagonal. The natural
general statement replaces \eqref{eq:sharpness} by its set-valued analogue,
$\bar\ell_\mu(u,x)-\min_v\bar\ell_\mu(v,x)\ge\varrho\bigl(\dist(u,S_\mu(x))\bigr)$ with
$S_\mu(x)$ the pointwise argmin correspondence, and adds a branch-separation hypothesis
playing the role of Lemma~\ref{lem:splice-exclusion}. We have not found the right general
form of the separation hypothesis in $\R^d$, where the level sets of a predictor need not be
finite unions of intervals and the interval-counting argument of
Lemma~\ref{lem:splice-exclusion} has no immediate analogue.
\item \emph{Whether $\epsinf$ can be positive without ties.} Over $\HK$ the realized set
$\mathcal R(\Phi_{\HK}(\mu))$ need be neither convex nor closed, whereas
$\partial d_{\HK}(\mu)=\partial d_{\Gen}(\mu)$ is convex and compact by
Lemma~\ref{lem:eps-star-properties}(iii), so strict inclusion should be generic rather than
exceptional; the tie construction is merely the most transparent way to realize it. We have
no example in which $\epsinf>0$ arises from non-closedness alone.
\end{enumerate}
\end{remark}

\section{Composite Predictors: A Link on Top of the RKHS}
\label{sec:composite-predictors}
In classification the RKHS function is seldom the predictor. It is a score, passed through a
fixed monotone squashing function before the loss sees it, so that the loss is a function of
a probability and not a real number. This appendix records what the framework of this
paper becomes under that reparameterization. No new theory is required: the composite problem
is an instance of \eqref{eq:primal formulation f} with a different loss family, so
Theorem~\ref{thm:informal-section3} applies verbatim. What changes, and what makes the
exercise worth carrying out, is the status of the hypotheses. Assumption~\ref{ass:loss-regularity}
becomes automatic for a substantially wider class of losses, including coercive ones, at no
cost to any constant in the analysis. The pointwise minimization that governs
Assumption~\ref{ass:multi-argmin-recovery} takes place on a compact interval rather than on
the line, so its attainment is free and its status becomes transparent. And the sharpness
condition of Proposition~\ref{prop:sharpness}, which must hold globally in the score
variable, can be transferred from a growth condition in the prediction variable, where it is
a statement about the curvature of a standard loss and is verifiable in closed form. We give
two instances in which every hypothesis is checked explicitly, both of them nonconvex in the
RKHS variable.

We are equally explicit about what the construction does not deliver. It does not evade the
scope limitation of Remark~\ref{rem:scope}: the sufficient condition below forces almost
everywhere unique pointwise minimizers and hence, by
Lemma~\ref{lem:eps-star-properties}(v), certifies only the exact regime $\epsinf=0$. Nor does
it weaken Assumption~\ref{ass:multi-argmin-recovery}(a), which
Lemma~\ref{lem:attain-realize} below identifies as exactly a realizability requirement,
neither more nor less. And the widened class of admissible losses is narrowed again from the
other side by Corollary~\ref{cor:link-boundary}, which excludes every outer loss whose
conditional minimizer sits on the boundary of the range of the link.

\paragraph{The composite problem.}
Fix a \emph{link} $\vartheta:\R\to[\alpha,\beta]$, continuous, strictly increasing and
$L_\vartheta$-Lipschitz, with $\vartheta(\R)$ dense in $[\alpha,\beta]$, and \emph{outer
losses} $\ell^{\mathrm o}_i:[\alpha,\beta]\times\mathcal Y\to\R$, $i\in\mathbb N_m$. The
composite setting is \eqref{eq:primal formulation RKHS} with
\[
\ell_i(u,y)\;\triangleq\;\ell^{\mathrm o}_i\bigl(\vartheta(u),y\bigr),
\qquad u\in\R,\ y\in\mathcal Y ,
\]
so that every object of Sections~\ref{sec:pac-dual-theorem} and~\ref{sec:proofs} is defined
without modification. We call $\Pi(\phi)\triangleq\vartheta\circ\phi$ the \emph{effective
prediction} of $\phi\in\HK$. Writing
$\bar\ell^{\mathrm o}_i(v,x)\triangleq\E_{\mathcal D_i}[\ell^{\mathrm o}_i(v,Y)\mid X=x]$ and
\[
\bar\ell^{\mathrm o}_\mu(v,x)\;\triangleq\;\sum_{i\in\mathbb N_m}\mu_i\,w_i(x)\,\bar\ell^{\mathrm o}_i(v,x),
\qquad \mu_0\triangleq1,\qquad v\in[\alpha,\beta],
\]
the conditional Lagrangian integrand of \eqref{eq:lagrangian-integrand} factors through the
link,
\begin{equation}\label{eq:composite-integrand}
\bar\ell_\mu(u,x)\;=\;\bar\ell^{\mathrm o}_\mu\bigl(\vartheta(u),x\bigr),
\qquad
L_{\HK}(\phi,\mu)\;=\;\int_{\mathcal X}\bar\ell^{\mathrm o}_\mu\bigl(\Pi(\phi)(x),x\bigr)\,d\mathcal M_{\mathcal X}(x).
\end{equation}
The logistic link $\sigma(u)=(1+e^{-u})^{-1}$ with $[\alpha,\beta]=[0,1]$ is the canonical
instance; the clipped variant $\vartheta_a\triangleq a+(1-2a)\sigma$, with range $(a,1-a)$, is
used below to accommodate outer losses that are finite only on the interior. Throughout, $u$
denotes the score and $v=\vartheta(u)$ the prediction, so that the manuscript's pointwise
minimizer $u^\sharp_\mu(x)$ of $\bar\ell_\mu(\cdot,x)$ and the pointwise minimizer
$v^\sharp_\mu(x)$ of $\bar\ell^{\mathrm o}_\mu(\cdot,x)$ over $[\alpha,\beta]$ are related by
$v^\sharp_\mu=\vartheta\circ u^\sharp_\mu$. In this appendix $\sigma$ always denotes the
logistic function; the Gaussian kernel is written $K_{\mathrm G}$ with bandwidth $b$, and the
clipped ramp of Appendices~\ref{sec:near-pacc-example} and~\ref{sec:worked-example} is written
$h$.

\paragraph{Regularity is inherited from the compact range.}
The first observation is elementary but is the reason the reparameterization is worth
recording, since Assumption~\ref{ass:loss-regularity} is what confines the body of the paper
to bounded Lipschitz losses and forces truncations such as the one in
Proposition~\ref{prop:no-distribution-free}.

\begin{lemmaBox}{Inheritance of loss regularity}{link-regularity}
Suppose each $\ell^{\mathrm o}_i$ is bounded and Lipschitz in $v$ on $[\alpha,\beta]$,
uniformly in $y$, with constants $L^{\mathrm o}_i$ and $M^{\mathrm o}_i$; this holds in
particular whenever $\ell^{\mathrm o}_i$ is continuous on $[\alpha,\beta]\times\mathcal Y$ with
$\mathcal Y$ compact. Then $\{\ell_i\}_{i\in\mathbb N_m}$ satisfies
Assumption~\ref{ass:loss-regularity} on all of $\R\times\mathcal Y$ with
\[
L\;=\;L_\vartheta\max_{i\in\mathbb N_m}L^{\mathrm o}_i,
\qquad
M\;=\;\max_{i\in\mathbb N_m}M^{\mathrm o}_i .
\]
\end{lemmaBox}

\begin{proof}
For $u,u'\in\R$, $|\ell_i(u,y)-\ell_i(u',y)|\le L^{\mathrm o}_i|\vartheta(u)-\vartheta(u')|\le
L^{\mathrm o}_iL_\vartheta|u-u'|$, and
$\ell_i(\R\times\mathcal Y)\subseteq\ell^{\mathrm o}_i([\alpha,\beta]\times\mathcal Y)$.
\end{proof}

Since continuity on a compact interval supplies both constants, the hypothesis of
Lemma~\ref{lem:link-regularity} asks only that the outer loss be finite and continuous on the
closed range of the link. Losses that are coercive, strictly convex or unbounded on $\R$ are
therefore admissible after composition, with no truncation and no modification of any constant
in the analysis: $L$ and $M$ enter exactly where they always do, through the multiplier bound
$C_\mu=4M/\xi$ of Lemma~\ref{lem:bounded-multipliers}, through the uniform-convergence envelope
$\Delta_n(\delta)$ of \eqref{eq:uc-envelope}, and through the Rademacher estimate
$\mathfrak R_n(\Bq{Q_n})\le\kappa Q_n/\sqrt n$, to which Talagrand contraction attaches
$\vartheta$ at no cost. The squared, logistic, exponential and cross-entropy outer losses all
satisfy the hypothesis, the last only when $[\alpha,\beta]\subset(0,1)$, which is precisely
what the clipped link supplies.

Two hypotheses require attention rather than inheritance. Assumption~\ref{ass:dense-rkhs} is a
statement about $\HK$ and $\mathcal M_{\mathcal X}$ and is untouched.
Assumption~\ref{ass:slater}, on the other hand, is now a condition on effective predictions:
since $\Pi(\HK)$ takes values in $\vartheta(\R)$, strict feasibility must be witnessed by a
predictor whose effective values lie in that range. When the constants belong to $\HK$, as for
$K=1+K_{\mathrm G}$, this reduces to the existence of a strictly feasible constant prediction
in $\vartheta(\R)$, and the Slater margin is computed with no approximation step and no
halving.

\paragraph{Attainment is realization of the pointwise optimum.}
The next result isolates a fact that is used inside the proof of
Proposition~\ref{prop:sharpness} but is not stated there, and that settles the status of
Assumption~\ref{ass:multi-argmin-recovery}(a). It holds in the raw setting as well; we state
it for the composite one because that is where the pointwise minimization is over a compact set
and the conclusion is therefore sharpest.

\begin{lemmaBox}{Attainment equals pointwise realization}{attain-realize}
Let Assumptions~\ref{ass:loss-regularity} and~\ref{ass:dense-rkhs} hold in the composite
setting. Then for every $\mu\ge0$,
\[
\inf_{\phi\in\HK}L_{\HK}(\phi,\mu)
\;=\;\int_{\mathcal X}\ \min_{v\in[\alpha,\beta]}\bar\ell^{\mathrm o}_\mu(v,x)\ d\mathcal M_{\mathcal X}(x),
\]
and $\phi^\dagger\in\Phi_{\HK}(\mu)$ if and only if
$\Pi(\phi^\dagger)(x)\in\argmin_{v\in[\alpha,\beta]}\bar\ell^{\mathrm o}_\mu(v,x)$ for
$\mathcal M_{\mathcal X}$-a.e.\ $x$. Consequently
Assumption~\ref{ass:multi-argmin-recovery}(a) holds at $\mu$ if and only if some
$\phi\in\Bq{C_\infty}$ has effective prediction equal to a pointwise-optimal selection
$\mathcal M_{\mathcal X}$-a.e.
\end{lemmaBox}

\begin{proof}
The pointwise minimum is attained because $\bar\ell^{\mathrm o}_\mu(\cdot,x)$ is continuous on
the compact $[\alpha,\beta]$, and is measurable in $x$ by measurable selection, so the
right-hand side is well defined and, by \eqref{eq:composite-integrand}, is a lower bound for the
left. For the reverse inequality write $\bar L\triangleq L(m+1)(1+\|\mu\|_1)$, so that
$|L_{\HK}(\phi,\mu)-L_{\HK}(\psi,\mu)|\le\bar L\|\Pi(\phi)-\Pi(\psi)\|_{L_1(\mathcal M_{\mathcal X})}$
and, more generally, the integral functional in \eqref{eq:composite-integrand} is
$\bar L$-Lipschitz in the prediction in $L_1(\mathcal M_{\mathcal X})$, using $w_i\le m+1$. Let
$v^\star$ be a measurable pointwise-optimal selection and, for small $\varepsilon>0$, let
$v^\star_\varepsilon$ be its truncation to $[\alpha+\varepsilon,\beta-\varepsilon]$, so that
$\|v^\star_\varepsilon-v^\star\|_{L_1(\mathcal M_{\mathcal X})}\le\varepsilon$. Since
$\vartheta$ is continuous and strictly increasing, $\vartheta(\R)$ is an interval, and being
dense in $[\alpha,\beta]$ it contains $(\alpha,\beta)$ and hence
$[\alpha+\varepsilon,\beta-\varepsilon]$; therefore
$\vartheta^{-1}\circ v^\star_\varepsilon$ is well defined, measurable, and bounded, with values
in the compact $[\vartheta^{-1}(\alpha+\varepsilon),\vartheta^{-1}(\beta-\varepsilon)]$, so it
lies in $\Gen=L_p(\mathcal M_{\mathcal X},\R)$. By Assumption~\ref{ass:dense-rkhs}(b) there is
$\phi\in\HK$ with $\|\phi-\vartheta^{-1}\circ v^\star_\varepsilon\|_{L_p(\mathcal M_{\mathcal X})}$
arbitrarily small, and since $\|\cdot\|_{L_1}\le\|\cdot\|_{L_p}$ on a probability space and
$\vartheta$ is $L_\vartheta$-Lipschitz, $\|\Pi(\phi)-v^\star_\varepsilon\|_{L_1(\mathcal M_{\mathcal X})}$
is arbitrarily small as well. Combining the two displays and letting $\varepsilon\downarrow0$
gives the claimed value. For the characterization, if $\phi^\dagger$ attains the infimum then
the nonnegative integrand
$\bar\ell^{\mathrm o}_\mu(\Pi(\phi^\dagger)(x),x)-\min_v\bar\ell^{\mathrm o}_\mu(v,x)$ has zero
integral, hence vanishes a.e.; the converse is immediate.
\end{proof}

Lemma~\ref{lem:attain-realize} says that Assumption~\ref{ass:multi-argmin-recovery}(a) is a
realizability hypothesis and is not weaker than one. The infimum over $\HK$ always equals the
pointwise-optimal integral, by density alone; what attainment adds is precisely that the
pointwise optimum is realized by an element of $\Bq{C_\infty}$. This sharpens the reading
offered in Remark~\ref{rem:recovery-status}, and it is worth stating plainly, because it is
tempting to hope that the composite structure removes the difficulty. It does not. What it
does is change the object whose realizability is at issue, from a possibly unbounded score
field to a bounded prediction field, and the following corollary shows that the change is not
cosmetic.

\begin{corollary}[Boundary obstruction]\label{cor:link-boundary}
Suppose $\vartheta(\R)=(\alpha,\beta)$ and, on a set of positive
$\mathcal M_{\mathcal X}$-measure, every minimizer of $\bar\ell^{\mathrm o}_\mu(\cdot,x)$ over
$[\alpha,\beta]$ lies in $\{\alpha,\beta\}$. Then $\Phi_{\HK}(\mu)=\emptyset$ and
Assumption~\ref{ass:multi-argmin-recovery}(a) fails at $\mu$.
\end{corollary}

Corollary~\ref{cor:link-boundary} is why the hypotheses below insist on \emph{uniform
interiority} of the pointwise optimum, and it delimits the reach of the construction concretely.
Margin losses are the natural casualty: for the hinge loss written in the prediction variable,
$\ell^{\mathrm o}(v,y)=(1-(2y-1)(2v-1))_+$, one computes on $[0,1]$ that
$\ell^{\mathrm o}(v,1)=2-2v$ and $\ell^{\mathrm o}(v,0)=2v$, so the conditional integrand at
$\mu=0$ is $2\eta_0(x)+2v\bigl(1-2\eta_0(x)\bigr)$, affine in $v$ with a minimizer at an
endpoint of $[0,1]$ whenever $\eta_0(x)\ne\tfrac12$; no continuous link with open range can
realize it. Outer losses whose conditional minimizer is a genuine interior probability, which
is to say proper losses under a nondegenerate noise condition, are the ones the construction
serves.

\paragraph{Sharpness transfers from the prediction variable.}
The sharpness condition~\eqref{eq:sharpness} of Proposition~\ref{prop:sharpness} is a
global statement in the score variable, and a saturating link makes the composite integrand flat
at infinity, so it is not obvious that a transfer is possible at all. It is, and the cap
$\Delta_0$ is exactly what absorbs the saturation, which is the role
Remark~\ref{rem:sharpness-reading} assigns to it. Write
\[
\underline{\vartheta}(T)\;\triangleq\;\inf\Bigl\{\tfrac{\vartheta(s')-\vartheta(s)}{s'-s}\;:\;-T\le s<s'\le T\Bigr\}
\]
for the increment modulus of the link on $[-T,T]$, which equals $\inf_{|s|\le T}\vartheta'(s)$
when $\vartheta$ is continuously differentiable and is strictly positive whenever $\vartheta$ is
strictly increasing on a window containing $[-T,T]$.

\begin{propBox}{Prediction-space growth implies score-space sharpness}{link-sharpness}
Fix $\mu_{\HK}\in\Sdual$ and suppose there are an open neighbourhood $U\ni\mu_{\HK}$, a
\emph{growth constant} $c_{\mathrm g}>0$, and for each $\mu\in U$ a measurable
$v^\sharp_\mu:\mathcal X\to[\alpha,\beta]$ such that, for $\mathcal M_{\mathcal X}$-a.e.\ $x$ and
every $v\in[\alpha,\beta]$,
\begin{equation}\label{eq:v-growth}
\bar\ell^{\mathrm o}_\mu(v,x)-\bar\ell^{\mathrm o}_\mu\bigl(v^\sharp_\mu(x),x\bigr)
\;\ge\;\tfrac{c_{\mathrm g}}{2}\bigl(v-v^\sharp_\mu(x)\bigr)^2 ,
\end{equation}
and such that $\phi^\sharp_\mu\triangleq\vartheta^{-1}\circ v^\sharp_\mu$ belongs to $\HK$ with
$\sup_{\mu\in U}\|\phi^\sharp_\mu\|_{\HK}\le C_\infty$. Then for every $\varsigma>0$ the
composite integrand satisfies \eqref{eq:sharpness} at every $\mu\in U$, with witness
$\phi^\sharp_\mu$, exponent $p=2$, and constants
\[
\gamma_0=\tfrac{c_{\mathrm g}}{2}\,\underline{\vartheta}(T)^2,
\qquad
\Delta_0=\gamma_0\varsigma^2,
\qquad
T=\kappa C_\infty+\varsigma .
\]
Consequently Proposition~\ref{prop:sharpness} applies and
Assumption~\ref{ass:multi-argmin-recovery} holds at $\mu_{\HK}$, with the rate
\eqref{eq:recovery-rate}.
\end{propBox}

\begin{proof}
Fix $\mu\in U$ and $x$ outside the null set, write $u^\sharp\triangleq\phi^\sharp_\mu(x)$ and
$v^\sharp\triangleq\vartheta(u^\sharp)=v^\sharp_\mu(x)$, and note
$|u^\sharp|\le\kappa\|\phi^\sharp_\mu\|_{\HK}\le\kappa C_\infty\le T$ by the reproducing
property and Assumption~\ref{ass:dense-rkhs}(a). Let $u\in\R$ and $d\triangleq|u-u^\sharp|$. If
$d\le\varsigma$ then $u$ and $u^\sharp$ both lie in $[-T,T]$, so
$|\vartheta(u)-v^\sharp|\ge\underline{\vartheta}(T)\,d$ by definition of the increment modulus,
and \eqref{eq:v-growth} evaluated at $v=\vartheta(u)$ gives, through
\eqref{eq:composite-integrand},
$\bar\ell_\mu(u,x)-\bar\ell_\mu(u^\sharp,x)\ge\gamma_0d^2$, which coincides with
$\varrho(d)=\min\{\gamma_0d^2,\Delta_0\}$ because $\gamma_0d^2\le\gamma_0\varsigma^2=\Delta_0$.
If instead $d>\varsigma$, suppose first $u>u^\sharp+\varsigma$; monotonicity of $\vartheta$ and
$u^\sharp+\varsigma\in[-T,T]$ give
$\vartheta(u)-v^\sharp\ge\vartheta(u^\sharp+\varsigma)-\vartheta(u^\sharp)\ge\underline{\vartheta}(T)\varsigma$,
so \eqref{eq:v-growth} yields
$\bar\ell_\mu(u,x)-\bar\ell_\mu(u^\sharp,x)\ge\gamma_0\varsigma^2=\Delta_0=\varrho(d)$, the last
equality because $\gamma_0d^2>\Delta_0$; the case $u<u^\sharp-\varsigma$ is symmetric, using
$u^\sharp-\varsigma\ge-T$. The two cases together are \eqref{eq:sharpness}.
\end{proof}

Three features of the transfer deserve comment. It uses no convexity, only growth of
$\bar\ell^{\mathrm o}_\mu(\cdot,x)$ away from its minimizer, and \eqref{eq:v-growth} is implied
by $c_{\mathrm g}$-strong convexity of $\bar\ell^{\mathrm o}_\mu(\cdot,x)$ on $[\alpha,\beta]$
but is strictly weaker. It is global in the score variable, the saturation of the link being
handled entirely by the cap. And it produces $p=2$, the nondegenerate curvature regime of
Remark~\ref{rem:sharpness-reading}, whenever the prediction-space growth is quadratic, which for
the standard proper losses it is. We also record that strict monotonicity of $\vartheta$
together with interiority of $v^\sharp_\mu$ makes the score-space minimizer
$u^\sharp_\mu=\phi^\sharp_\mu$ almost everywhere unique, which is the hypothesis of
Lemma~\ref{lem:eps-star-properties}(v).

\begin{remark}[Constants, and why link sharpness is not a free parameter]\label{rem:link-constants}
For the logistic link, $\underline{\sigma}(T)=\sigma(T)(1-\sigma(T))\ge\tfrac14e^{-T}$, so
$\gamma_0\ge\tfrac{c_{\mathrm g}}{32}e^{-2T}$ and, optimizing $\varsigma^2e^{-2\varsigma}$ at
$\varsigma=1$, $\Delta_0\ge\tfrac{c_{\mathrm g}}{32}e^{-2\kappa C_\infty-2}$. The constants
therefore degrade exponentially in $\kappa C_\infty$, and it is worth being precise about what
that costs, since $\gamma_0$ and $\Delta_0$ enter \eqref{eq:recovery-rate} through two
different terms and the answer depends on the radius schedule.

At a \emph{bounded} radius $Q$, where $\eps_n\asymp Qn^{-1/2}$, the binding constraint is
$\eps_n\lesssim\Delta_0$ and the sample size at which \eqref{eq:recovery-rate} becomes
informative scales like $Q^4\Delta_0^{-2}\asymp e^{4\kappa C_\infty}$. Under the growing
schedule $Q_n=n^{1/4}/\log(e+n)$ actually used in this paper, the accounting of
Remark~\ref{rem:effective-rate} applies instead and is less favourable in both terms: the first
term of \eqref{eq:recovery-rate} is $\asymp\gamma_0^{-1/2}n^{-1/8}$ and requires
$n\gtrsim e^{8\kappa C_\infty}$, while the escape term is $\asymp4\kappa/(\Delta_0\log^2n)$ and
requires $\log^2n\gtrsim\Delta_0^{-1}\asymp e^{2\kappa C_\infty}$, that is
$n\gtrsim\exp\bigl(ce^{\kappa C_\infty}\bigr)$. The exponential degradation in $\kappa C_\infty$
is therefore compounded by the schedule, and the second requirement is doubly exponential. We
record this rather than hide it; it is the quantitative face of the qualitative statement that
the growing balls are paid for by the cap.

The cost is intrinsic rather than an artifact of the logistic: for a link
affine with slope $s$ on $[-T_0,T_0]$ and constant outside, $\underline{\vartheta}(T)=s$ for all
$T\le T_0$ and the constants are absolute, but only under the design rule
$T_0\ge\kappa C_\infty+\varsigma$, that is, only when the active window of the link contains the
range of the witness. That rule cannot be circumvented, for the following reason. Since
$\|s\phi\|_{\HK}=s\|\phi\|_{\HK}$, the model $\vartheta(s\,\cdot)$ over $\Bq{Q}$ and the model
$\vartheta$ over $\Bq{sQ}$ coincide, and under this joint rescaling $\Delta_0$ is invariant.
Sharpening the link and inflating the radius schedule of \eqref{eq:schedule} are one knob and
not two, and the admissible growth budget $Q_n=o(n^{1/4})$ of Proposition~\ref{prop:sharpness}
depends on the link only through $\kappa C_\infty$.
\end{remark}

\paragraph{Three notions of convexity, and which one is used.}
The composite setting invites a conflation worth dispelling, since dispelling it also explains
why the instances below are genuine members of the nonconvex regime this paper addresses.
Distinguish convexity of $\ell^{\mathrm o}_i(\cdot,y)$ on $[\alpha,\beta]$, convexity of the
composite $\ell_i(\cdot,y)$ on $\R$, and convexity of $\phi\mapsto\mathcal R_i(\phi)$ on $\HK$.
The last two are equivalent, since $\phi\mapsto\phi(x)$ is linear and, for the kernels used
here, the constants belong to $\HK$. The first does not imply them.

\begin{remark}[A criterion for composite convexity]\label{rem:composite-convexity}
Let $\vartheta$ be continuously differentiable and strictly increasing and let
$\ell^{\mathrm o}(\cdot,y)$ be continuously differentiable on $\vartheta(\R)$. Since
$\tfrac{d}{du}\ell^{\mathrm o}(\vartheta(u),y)=\partial_v\ell^{\mathrm o}(v,y)\,\vartheta'(\vartheta^{-1}(v))$
at $v=\vartheta(u)$, the composite $\ell(\cdot,y)$ is convex on $\R$ if and only if
\[
v\;\longmapsto\;\partial_v\ell^{\mathrm o}(v,y)\,\vartheta'\bigl(\vartheta^{-1}(v)\bigr)
\qquad\text{is nondecreasing on }\vartheta(\R).
\]
For the logistic link $\vartheta'(\vartheta^{-1}(v))=v(1-v)$, so the criterion reads: $v\mapsto
v(1-v)\partial_v\ell^{\mathrm o}(v,y)$ nondecreasing on $(0,1)$. For the cross-entropy outer loss
$\ell^{\mathrm o}(v,y)=-y\log v-(1-y)\log(1-v)$ this quantity equals $v-y$, recovering the
classical convexity of logistic regression in the score, with second derivative
$\sigma(u)(1-\sigma(u))$. For the Brier outer loss $\ell^{\mathrm o}(v,y)=(v-y)^2$ it equals
$2v(1-v)(v-y)$, whose derivative at $y=1$ is $-2(1-v)(1-3v)<0$ for $v<\tfrac13$; the composite is
therefore not convex, and neither is $\phi\mapsto\mathcal R_0(\phi)$. Indeed the conditional
composite is $u\mapsto(\sigma(u)-\eta)^2$ up to a constant, with second derivative
$2\sigma(1-\sigma)\bigl[\sigma(1-\sigma)+(\sigma-\eta)(1-2\sigma)\bigr]$, whose bracket tends to
$-\eta<0$ as $u\to-\infty$; so nonconvexity in $\phi$ survives the average over $y$ for every
$\eta>0$, and in particular under the noise condition below. For the clipped link
$\vartheta_a$ with $a>0$ the cross-entropy criterion becomes $v\mapsto v-1+a(1-a)/v$ at $y=1$
(up to the positive factor $(1-2a)^{-1}$), decreasing for $v<\sqrt{a(1-a)}$, so the clipping
that restores Assumption~\ref{ass:loss-regularity}(b) for the cross-entropy loss destroys the
exact convexity it had. Only convexity in the prediction variable is used in
Proposition~\ref{prop:link-sharpness}, and it is the one the standard outer losses possess.
\end{remark}

\paragraph{Common setting for the two instances.}
Both take $\mathcal X\subset\R^d$ compact, $\mathcal Y=\{0,1\}$ and $m=1$, with $\mathcal D_0$ a
population and $\mathcal D_1$ a subgroup, both admitting Lebesgue densities, so that
Assumption~\ref{ass:non-atomic} holds. Both use $K=1+K_{\mathrm G}$ with $K_{\mathrm G}$
Gaussian of bandwidth $b$, so that
$\kappa^2=2$, Assumption~\ref{ass:dense-rkhs} holds for every $p\in[1,\infty)$ by
\cite[Theorem~4.63]{SteinwartChristmann2008}, and the constants belong to $\HK$ exactly with
$\|c\|_{\HK}=|c|$. Write $w_0,w_1$ for the mixture densities, so that $w_0+w_1=2$, and
$\eta_i(x)\triangleq\Pr_{\mathcal D_i}(Y=1\mid X=x)$. Both use the schedule
$Q_n=n^{1/4}/\log(e+n)$ of Remark~\ref{rem:schedule}, which satisfies \eqref{eq:schedule} and the
$o(n^{1/4})$ budget of Proposition~\ref{prop:sharpness} simultaneously. Both assume the
\emph{noise condition} that $\eta_0$ and $\eta_1$ take values in a fixed
$[\eta_-,\eta_+]\subset(0,1)$, and the \emph{overlap condition} that $\sup_x w_1(x)<2$,
equivalently that the two feature marginals have bounded density ratio; the latter is what makes
the curvature of $\bar\ell^{\mathrm o}_\mu$ uniform in $x$, and the former is what keeps the
pointwise optimum away from the boundary excluded by Corollary~\ref{cor:link-boundary}.

\paragraph{First instance: Brier objective under a subgroup Brier cap.}
Take the logistic link, $\ell^{\mathrm o}_0(v,y)=(v-y)^2$ and $\ell^{\mathrm o}_1(v,y)=(v-y)^2-c$
with $c\in(0,1)$, the latter in the normalized form of Remark~\ref{rem:constraint-levels}, so
that the constraint caps the subgroup Brier risk at $c$. (This is a level cap on one group, not
a parity constraint between two; a genuine parity requirement is the two-sided version, which
the framework accommodates with $m=2$ but which we do not treat here.)
Assumption~\ref{ass:loss-regularity} holds by Lemma~\ref{lem:link-regularity} with
$L=\tfrac12$ and $M=1$, since the squared loss is $2$-Lipschitz and bounded by $1$ on $[0,1]$,
$|(v-y)^2-c|\le\max\{c,1-c\}\le1$, and $L_\sigma=\tfrac14$. Because
$\E[(v-Y)^2\mid X=x]=(v-\eta_i(x))^2+\eta_i(x)(1-\eta_i(x))$ under $\mathcal D_i$, the conditional
integrand is, up to a term free of $v$,
\[
\bar\ell^{\mathrm o}_\mu(v,x)\;=\;w_0(x)\bigl(v-\eta_0(x)\bigr)^2+\mu_1w_1(x)\bigl(v-\eta_1(x)\bigr)^2 ,
\]
a quadratic with curvature $2(w_0+\mu_1w_1)\ge2\inf_xw_0(x)$, so \eqref{eq:v-growth} holds with
$c_{\mathrm g}=2\bigl(2-\sup_xw_1(x)\bigr)>0$ by the overlap condition. The pointwise minimizer
is the convex combination
\[
v^\sharp_\mu(x)\;=\;\frac{w_0(x)\eta_0(x)+\mu_1w_1(x)\eta_1(x)}{w_0(x)+\mu_1w_1(x)}\;\in\;[\eta_-,\eta_+] ,
\]
so it is uniformly interior and unconstrained by the endpoints of $[0,1]$. Assume further that
$\eta_0,\eta_1,w_0,w_1$ extend holomorphically to a fixed complex strip about $\mathcal X$, which
holds in particular when they are real-analytic. The denominator is bounded below by $\inf_xw_0>0$
uniformly over $\mu$ in a compact neighbourhood $U$ of $\mu_{\HK}$, and $\sigma^{-1}$ is
holomorphic on a neighbourhood of $[\eta_-,\eta_+]\subset(0,1)$, so the witness
$\phi^\sharp_\mu=\sigma^{-1}\circ v^\sharp_\mu$ extends holomorphically to a strip whose width does
not depend on $\mu\in U$; the resulting exponential Fourier decay places it in the Gaussian RKHS
with $\sup_{\mu\in U}\|\phi^\sharp_\mu\|_{\HK}\eqqcolon C_\infty<\infty$, by the same route used in
Appendix~\ref{sec:worked-example}. For Assumption~\ref{ass:slater}, write
$c_{\min}\triangleq\min_{v\in(0,1)}\E_{\mathcal D_1}[(v-Y)^2]$, attained at
$v=\E_{\mathcal D_1}[\eta_1(X)]\in(0,1)$ since
$\E_{\mathcal D_1}[(v-Y)^2]=v^2-2v\,\E[\eta_1]+\E[\eta_1]$; if $c>c_{\min}$ then the constant
predictor $\phi^{\mathrm{sl}}\equiv\sigma^{-1}(\E_{\mathcal D_1}[\eta_1(X)])$ is strictly feasible
with $\xi=c-c_{\min}$, $Q_{\mathrm{sl}}=|\sigma^{-1}(\E_{\mathcal D_1}[\eta_1(X)])|$ and
$C_\mu=4/\xi$. Proposition~\ref{prop:link-sharpness} now supplies \eqref{eq:sharpness} with
$p=2$, $\gamma_0=c_{\mathrm g}\,\underline{\sigma}(\kappa C_\infty+1)^2/2$ and
$\Delta_0=\gamma_0$ (taking $\varsigma=1$), so Proposition~\ref{prop:sharpness} delivers
Assumption~\ref{ass:multi-argmin-recovery}. The pointwise minimizer is a.e.\ unique, so
$\epsinf=0$ by Lemma~\ref{lem:eps-star-properties}(v) and
Theorem~\ref{thm:informal-section3}(vi) applies: the guarantee is exact PACC. By
Remark~\ref{rem:composite-convexity} the objective $\phi\mapsto\mathcal R_0(\phi)$ is \emph{not}
convex on $\HK$, so the nonconvex machinery of the paper is genuinely doing work here.

\paragraph{Second instance: clipped cross-entropy under a positive-rate cap.}
Take the clipped link $\vartheta_a=a+(1-2a)\sigma$, the cross-entropy outer loss
$\ell^{\mathrm o}_0(v,y)=-y\log v-(1-y)\log(1-v)$, and the linear constraint
$\ell^{\mathrm o}_1(v,y)=v-\gamma$, so that
$\mathcal R_1(\phi)=\E_{\mathcal D_1}[\Pi(\phi)(X)]-\gamma$ caps the subgroup positive rate at
$\gamma$. (Appendix~\ref{sec:worked-example} imposes the opposite orientation, a \emph{floor} on
the same quantity; both are within the framework and we distinguish them by name.) The
unclipped cross-entropy violates Assumption~\ref{ass:loss-regularity}(b), and the clip is exactly
what restores it: Lemma~\ref{lem:link-regularity} applies on $[a,1-a]$ with
$M=\max\{\log(1/a),1\}$ and $L=(1-2a)/(4a)$, since
$|\partial_v\ell^{\mathrm o}_0|=|v-y|/(v(1-v))\le1/a$ there. The conditional integrand is
\[
\bar\ell^{\mathrm o}_\mu(v,x)\;=\;w_0(x)\bigl[-\eta_0(x)\log v-(1-\eta_0(x))\log(1-v)\bigr]+\mu_1w_1(x)(v-\gamma),
\]
with $\partial^2_v\bar\ell^{\mathrm o}_\mu=w_0[\eta_0v^{-2}+(1-\eta_0)(1-v)^{-2}]\ge w_0$ on
$(0,1)$ --- the bracket is minimized at $v=(1+t)^{-1}$, $t=((1-\eta_0)/\eta_0)^{1/3}$, with value
$(\eta_0^{1/3}+(1-\eta_0)^{1/3})^3\ge1$ --- so \eqref{eq:v-growth} holds with
$c_{\mathrm g}=2-\sup_xw_1(x)>0$. Writing
$r_\mu(x)\triangleq\mu_1w_1(x)/w_0(x)$, the first-order condition
$(v-\eta_0)/(v(1-v))=-r_\mu$ is the quadratic $r_\mu v^2-(1+r_\mu)v+\eta_0=0$, whence the closed
form
\[
v^\sharp_\mu\;=\;\frac{(1+r_\mu)-\sqrt{(1+r_\mu)^2-4r_\mu\eta_0}}{2r_\mu}\quad(r_\mu>0),
\qquad
v^\sharp_\mu=\eta_0\quad(r_\mu=0),
\]
together with the elementary bounds $\eta_0/(1+r_\mu)\le v^\sharp_\mu\le\eta_0$, which follow from
$v(1+r_\mu-r_\mu v)=\eta_0$ and $1\le1+r_\mu-r_\mu v\le1+r_\mu$. Since $\|\mu\|_1\le C_\mu$ on the
relevant multiplier set (Lemma~\ref{lem:bounded-multipliers}), we have
$r_\mu\le r_{\max}\triangleq C_\mu\sup_xw_1/\inf_xw_0$, so $v^\sharp_\mu$ lies in
$[\eta_-/(1+r_{\max}),\eta_+]$; choosing $a<\min\{\eta_-/(1+r_{\max}),\,1-\eta_+\}$ makes the clip
inactive at the optimum and leaves $v^\sharp_\mu$ uniformly interior to $\vartheta_a(\R)=(a,1-a)$.
The discriminant obeys $(1+r)^2-4r\eta_0=(1-r)^2+4r(1-\eta_0)\ge(1-r)^2+4r(1-\eta_+)>0$, so
$v^\sharp_\mu$ is an analytic function of $(\eta_0,r_\mu)$ and the holomorphic-extension argument
of the first instance again places $\phi^\sharp_\mu=\vartheta_a^{-1}\circ v^\sharp_\mu$ in the
Gaussian RKHS with a uniform bound $C_\infty$. Strict feasibility is witnessed, whenever
$a<\gamma<1-a$, by the constant $\phi^{\mathrm{sl}}\equiv\vartheta_a^{-1}((a+\gamma)/2)$ with
margin $\xi=(\gamma-a)/2$; the upper restriction on $\gamma$ is what puts the midpoint
$(a+\gamma)/2$ inside the range of the link, and is harmless since $a$ is chosen small.
Propositions~\ref{prop:link-sharpness} and~\ref{prop:sharpness} then deliver
Assumption~\ref{ass:multi-argmin-recovery}, and as in the first instance the guarantee is exact
PACC. By Remark~\ref{rem:composite-convexity} the unclipped composite would be convex in $\phi$
but inadmissible under Assumption~\ref{ass:loss-regularity}(b), whereas the clipped composite is
admissible and nonconvex; the trade is exact and is visible in the criterion.

\paragraph{Scope.}
Under the hypotheses of Proposition~\ref{prop:link-sharpness} the pointwise minimizer is a.e.\
unique, so Lemma~\ref{lem:eps-star-properties}(v) forces $\epsinf(\mu_{\HK})=0$. The composite
construction therefore inherits the limitation recorded in Remark~\ref{rem:scope} without
alleviating it: it certifies branch~(vi) of Theorem~\ref{thm:informal-section3} and is blind to
the near-PACC branch that parts~(ii) to~(v) are designed for, of which
Appendix~\ref{sec:near-pacc-example} supplies the only instance we have. The reason is structural
rather than incidental. Growth of $\bar\ell^{\mathrm o}_\mu(\cdot,x)$ around a single point is what
Proposition~\ref{prop:link-sharpness} converts into \eqref{eq:sharpness}, and ties in the
prediction variable are precisely what such growth excludes. A tie-tolerant analogue would replace
\eqref{eq:v-growth} by
$\bar\ell^{\mathrm o}_\mu(v,x)-\min_{v'}\bar\ell^{\mathrm o}_\mu(v',x)\ge\tfrac{c_{\mathrm g}}{2}\operatorname{dist}(v,S_\mu(x))^2$
for the pointwise argmin correspondence $S_\mu$, and would then need the branch-separation
mechanism of Appendix~\ref{sec:near-pacc-example} to survive composition with the link.

Here the composite setting is mildly encouraging, but only on a window, and the qualification
matters. Because $\vartheta$ is monotone with $\underline{\vartheta}(T)>0$, two prediction
branches separated by $\delta$ in $v$ are separated by at least $\delta/L_\vartheta$ in $u$
\emph{provided both preimages lie in $[-T,T]$}, by the same two-case argument as in
Proposition~\ref{prop:link-sharpness}. The empirical minimizers, however, range over
$\Bq{Q_n}$ and hence over scores of modulus up to $\kappa Q_n\to\infty$, where the link is flat
and $\underline{\vartheta}$ is useless; on that region it is the cap $\Delta_0$, not the
increment modulus, that must charge a predictor for being far from the argmin set, exactly as in
Proposition~\ref{prop:sharpness}. A tie-tolerant transfer would therefore have the same
two-regime structure as the sharpness transfer, with the window carrying the separation and the
cap carrying the escape. What does not transfer at all is the interval-counting step of
Lemma~\ref{lem:splice-exclusion}, which remains one-dimensional. We regard the composite setting
as a natural place to look for the general condition sought in Remark~\ref{rem:scope}, and we do
not have it.

\section{An Instance with a Strictly Positive Closure--Realization Gap}
\label{sec:near-pacc-example}
Remark~\ref{rem:scope} explained why Proposition~\ref{prop:sharpness} can certify only the
exact regime $\epsinf=0$. We now exhibit an explicit instance on which
Assumptions~\ref{ass:non-atomic}--\ref{ass:multi-argmin-recovery} all hold and
$\epsinf(\Sdual)=1/4$, so that parts~(ii)--(v) of Theorem~\ref{thm:informal-section3} are
inhabited. The mechanism is the one Lemma~\ref{lem:eps-star-properties}(iv) predicts: the
population Lagrangian has \emph{two} pointwise minimizers with \emph{distinct} constraint
risks; the envelope realizes the whole segment between them by splicing; and $\HK$ --- a
space of continuous functions on a connected domain --- realizes only the two endpoints.

\subsection{The instance}
Let $\mathcal X=[0,1]$, $\mathcal Y=\{-1,+1\}$, $m=1$, and let
$\mathcal D_{\mathcal X,0}=\mathcal D_{\mathcal X,1}$ be Lebesgue measure on $[0,1]$, so
$\mathcal M_{\mathcal X}$ is uniform and $w_0=w_1\equiv1$; the conditional label laws are
arbitrary, since the losses below do not depend on $y$. Put
\[
h(u)\;\triangleq\;\mathrm{cl}_{[0,1]}\Bigl(\tfrac{u+1}{2}\Bigr),
\qquad
\gamma\triangleq\tfrac12,
\qquad
\ell_0(u,y)\triangleq g(u),
\qquad
\ell_1(u,y)\triangleq h(u)-\gamma ,
\]
where $\mathrm{cl}_{[0,1]}$ denotes truncation to $[0,1]$ and $g$ is the explicit piecewise
linear function
\begin{equation}\label{eq:g-explicit}
g(u)\;\triangleq\;
\begin{cases}
\ \ \tfrac14, & u\le-1,\\[2pt]
-\tfrac u2-\tfrac14, & -1\le u\le-\tfrac12,\\[2pt]
\ \ 0, & -\tfrac12\le u\le0,\\[2pt]
-\tfrac u2, & 0\le u\le\tfrac12,\\[2pt]
-\tfrac14, & u\ge\tfrac12 .
\end{cases}
\end{equation}
Then $g$ is continuous, $\tfrac12$-Lipschitz and bounded by $\tfrac14$, with
$g(-\tfrac12)=0$ and $g(\tfrac12)=-\tfrac14$. Take
\[
K \;=\; 1 + K_{\mathrm G} ,
\qquad K_{\mathrm G}(x,x')=e^{-|x-x'|^2/2b^2},
\]
which is bounded with $\kappa^2=2$, universal on the compact $[0,1]$, and --- unlike
$K_{\mathrm G}$ alone --- contains the constants, with $\|c\|_{\HK}=|c|$. It also satisfies the
derivative bound used below: every $f\in\HK$ is differentiable with
\begin{equation}\label{eq:deriv-bound}
\|f'\|_\infty\;\le\;\varkappa\,\|f\|_{\HK},
\qquad
\varkappa\;\triangleq\;\sup_x\bigl(\partial_1\partial_2K(x,x)\bigr)^{1/2}=\tfrac1b .
\end{equation}
Finally take the radius schedule $Q_n=n^{1/4}/\log(e+n)$ of Remark~\ref{rem:schedule}, which
satisfies \eqref{eq:schedule}, the $o(n^{1/4})$ budget, and the sharper requirement
\begin{equation}\label{eq:schedule-splice}
Q_n^{2}\sqrt{\tfrac{\log n}{n}}\;\longrightarrow\;0
\end{equation}
used in Lemma~\ref{lem:splice-exclusion} and Proposition~\ref{prop:example-recovery}. For this
schedule the left-hand side of \eqref{eq:schedule-splice} is $\asymp(\log n)^{-3/2}$.

\begin{remark}[Why $g$ is not mollified]\label{rem:no-mollify}
The kinks of \eqref{eq:g-explicit} are load-bearing twice over. They pin the argmin of
$\psi_\mu$ below to a breakpoint under small perturbations of $\mu$, which is what makes the
superdifferential a singleton off $\mu^\star$; and they produce the linear ($p=1$) sharpness
of \eqref{eq:set-sharpness}. Smoothing $g$ would also perturb the two well values and destroy
the exact tie on which $\mu^\star=\tfrac12$ depends. No smoothing is needed: $g$ as written
satisfies Assumption~\ref{ass:loss-regularity} outright.
\end{remark}

\subsection{The population problem}

\emph{Assumptions~\ref{ass:non-atomic}--\ref{ass:dense-rkhs}.} The marginals are Lebesgue,
hence nonatomic. $g$ is $\tfrac12$-Lipschitz and bounded by $\tfrac14$, $h-\gamma$ is
$\tfrac12$-Lipschitz and bounded by $\tfrac12$, so Assumption~\ref{ass:loss-regularity} holds
with $L=\tfrac12$, $M=\tfrac12$. $K$ is universal on the compact $[0,1]$, so
Assumption~\ref{ass:dense-rkhs} holds with any $p\in[1,\infty)$.

\emph{Assumption~\ref{ass:slater}.} $\phi^{\mathrm{sl}}\equiv-\tfrac12\in\HK$ has
$\mathcal R_1=h(-\tfrac12)-\gamma=\tfrac14-\tfrac12=-\tfrac14$, so $\xi=\tfrac14$,
$Q_{\mathrm{sl}}=\tfrac12$, and $C_\mu=4M/\xi=8$; accordingly $\Lambda=[0,8]$.

\emph{The dual function.} Since the losses are $x$-free, the conditional integrand of
\eqref{eq:lagrangian-integrand} is
$\bar\ell_\mu(u,x)=\psi_\mu(u)\triangleq g(u)+\mu\bigl(h(u)-\tfrac12\bigr)$, constants lie in
$\HK$, and $\HK$ is dense in $\Gen$; hence
$d_{\HK}(\mu)=\min_{u\in\R}\psi_\mu(u)$, attained. The function $g+\mu h$ is piecewise linear
with breakpoints $-1,-\tfrac12,0,\tfrac12,1$ and slopes
$0,\ \tfrac{\mu-1}{2},\ \tfrac\mu2,\ \tfrac{\mu-1}{2},\ \tfrac\mu2,\ 0$ on the successive
pieces, so for $\mu\in(0,1)$ its only local minima are at $u_1\triangleq-\tfrac12$ and
$u_2\triangleq+\tfrac12$, where $g+\mu h$ takes the values $\tfrac\mu4$ and
$\tfrac{3\mu}{4}-\tfrac14$; equivalently, subtracting the constant $\mu/2$,
\[
\psi_\mu(u_1)=-\tfrac\mu4,\qquad \psi_\mu(u_2)=\tfrac\mu4-\tfrac14 .
\]
For $\mu\ge1$ all slopes are nonnegative and the minimum of $g+\mu h$ is the constant value
$\tfrac14$ attained on $(-\infty,-1]$. Comparing the two wells, $\psi_\mu(u_1)\le\psi_\mu(u_2)$
exactly when $\mu\ge\tfrac12$, so
\begin{equation}\label{eq:d-explicit}
d_{\HK}(\mu)\;=\;
\begin{cases}
\tfrac\mu4-\tfrac14, & 0\le\mu\le\tfrac12,\\[2pt]
-\tfrac\mu4, & \tfrac12\le\mu\le1,\\[2pt]
\tfrac14-\tfrac\mu2, & \mu\ge1,
\end{cases}
\qquad
\Sdual=\{\mu^\star\},\quad \mu^\star=\tfrac12,\quad D^\star_{\HK}=P^\star_{\HK}=-\tfrac18 ,
\end{equation}
and $d_{\HK}$ grows sharply away from its maximizer:
\begin{equation}\label{eq:dual-sharp}
d_{\HK}(\mu^\star)-d_{\HK}(\mu)\;\ge\;\tfrac14\bigl|\mu-\mu^\star\bigr|
\qquad\text{for every }\mu\in\Lambda=[0,8] ,
\end{equation}
with equality on $[0,1]$ and slack $\tfrac14(\mu-1)$ beyond. In particular $\Sdual$ is a
singleton, so the sharpened part~(v) of Theorem~\ref{thm:informal-section3} applies. At
$\mu^\star$ the two wells are exactly level,
$\psi_{\mu^\star}(u_1)=\psi_{\mu^\star}(u_2)=-\tfrac18$, with \emph{distinct} constraint
risks
\[
h(u_1)-\gamma=-\tfrac14,\qquad h(u_2)-\gamma=+\tfrac14 .
\]

\emph{Sharpness relative to the argmin set.} A direct computation from
\eqref{eq:g-explicit} gives the identity
\[
\psi_{\mu^\star}(u)+\tfrac18\;=\;\varrho\bigl(\dist(u,\{u_1,u_2\})\bigr)
\qquad\text{for every }u\in\R,
\qquad
\varrho(t)\triangleq\min\{t/4,\ 1/8\} ,
\]
with \emph{equality}, not merely inequality: on each of the six pieces of
\eqref{eq:g-explicit} both sides are affine with matching slopes $\pm\tfrac14$ or $0$ and
matching values at the breakpoints. Since
$\|\psi_\mu-\psi_{\mu^\star}\|_\infty\le|\mu-\mu^\star|\,\|h-\tfrac12\|_\infty
=|\mu-\mu^\star|/2$ and the same bound therefore governs
$|d_{\HK}(\mu)-d_{\HK}(\mu^\star)|$, we obtain
\begin{equation}\label{eq:set-sharpness}
\psi_{\mu}(u)-d_{\HK}(\mu)\;\ge\;\varrho\bigl(\dist(u,\{u_1,u_2\})\bigr)-\bigl|\mu-\mu^\star\bigr|
\qquad\text{for all }u\in\R,\ \mu\ge0 .
\end{equation}
This is the set-valued analogue of \eqref{eq:sharpness} with $\gamma_0=\tfrac14$, $p=1$,
$\Delta_0=\tfrac18$.

\subsection{The argmin sets, and the gap}

\emph{Over $\HK$ the argmin collapses to two points.} Let $\phi\in\HK$ satisfy
$L_{\HK}(\phi,\mu^\star)=d_{\HK}(\mu^\star)$, i.e.\ $\psi_{\mu^\star}(\phi(x))=-\tfrac18$ for
a.e.\ $x$, i.e.\ $\phi(x)\in\{u_1,u_2\}$ a.e. The sets $A=\phi^{-1}(u_1)$ and
$B=\phi^{-1}(u_2)$ are closed by continuity, disjoint, and $A\cup B$ has full Lebesgue
measure hence is dense, hence equals $[0,1]$; since $[0,1]$ is connected, one of them is
empty. Therefore
\begin{equation}\label{eq:argmin-two-points}
\Phi_{\HK}(\mu^\star)=\{\,\phi\equiv u_1,\ \phi\equiv u_2\,\},
\qquad
\mathcal R_1\bigl(\Phi_{\HK}(\mu^\star)\bigr)=\bigl\{-\tfrac14,\ +\tfrac14\bigr\},
\end{equation}
both constants lying in $\HK$ with norm $\tfrac12$. Thus
Assumption~\ref{ass:multi-argmin-recovery}(a) holds with $C_\infty=\tfrac12$, and its
restriction clause is vacuous because $\Phi_{\HK}(\mu^\star)\subset\Bq{C_\infty}$. The same
$C_\infty$ works on a neighbourhood: for $\mu\in(0,1)\setminus\{\tfrac12\}$ the minimizer of
$\psi_\mu$ is the single point $u_1$ or $u_2$ according as $\mu>\tfrac12$ or $\mu<\tfrac12$,
so $\Phi_{\HK}(\mu)$ is the corresponding constant, again of norm $\tfrac12$.

\emph{Over $\Gen$ the argmin is decomposable.} $\Phi_{\Gen}(\mu^\star)$ consists of all
measurable $\{u_1,u_2\}$-valued functions, a decomposable set on which the pointwise minimum
is attained; by Lemma~\ref{lem:eps-star-properties}(iii) and
Theorem~\ref{thm:dense-transfer},
\[
\partial d_{\HK}(\mu^\star)=\partial d_{\Gen}(\mu^\star)=\bigl[-\tfrac14,+\tfrac14\bigr],
\]
the segment swept out by measurable splices of $u_1$ and $u_2$, and consistent with the
one-sided slopes $+\tfrac14$ and $-\tfrac14$ of $d_{\HK}$ at $\mu^\star$ read off from
\eqref{eq:d-explicit}. Hence
\begin{equation}\label{eq:example-eps}
\epsinf(\mu^\star)
=\sup_{\nabla\in[-\frac14,\frac14]}\ \min\Bigl\{\bigl|\nabla+\tfrac14\bigr|,\ \bigl|\nabla-\tfrac14\bigr|\Bigr\}
=\tfrac14 \;>\;0,
\end{equation}
the supremum attained at the midpoint $\nabla=0$ --- which is precisely the sign- and
complementarity-compatible supergradient selected by Lemma~\ref{lem:sign-comp-subgrad}, since
$\mu^\star>0$ forces $\mu^\star\nabla=0$. The worst case in
Definition~\ref{def:multi-eps-star} is therefore not an artifact of the supremum: it is
realized at the supergradient the learner actually tracks.

Note also that $\Phi_{\Gen}(\mu^\star)$ contains the splice $\phi=u_1$ on $[0,\tfrac12]$,
$u_2$ on $(\tfrac12,1]$, which is exactly feasible ($\mathcal R_1=0$) and optimal
($\mathcal R_0=-\tfrac18=P^\star_{\HK}$), whereas over $\HK$ neither constant is both: $u_1$
is feasible with $\mathcal R_0=0>P^\star_{\HK}$, and $u_2$ is optimal but infeasible. So the
primal is attained over the envelope and \emph{not} over $\HK$ --- over $\HK$ the value
$-\tfrac18$ is approached by predictors that spend vanishing measure transiting between the
two wells --- and the entire gap is the failure of continuity to accommodate splicing, the
nondecomposability obstruction of Lemma~\ref{lem:eps-star-properties}(iv) in its purest form.

\subsection{Verification of Assumption~\ref{ass:multi-argmin-recovery}(b)}

It remains to verify part~(b), in the form~(B) of Remark~\ref{rem:part-b-quantifiers}; form~(A)
is unavailable here and must be, since the two witnesses are at $L_1$ distance $1$. We do this
carefully, because the naive argument fails.

\begin{remark}[The constants are \emph{not} exact empirical minimizers]\label{rem:not-exact}
One is tempted to argue that $u_1$ and $u_2$, being pointwise minimizers of
$\psi_{\mu^\star}$, belong to $\widehat\Phi_n(\mu)$ for $\mu$ near $\mu^\star$ once
$Q_n\ge\tfrac12$. They do not, for any $\mu>0$ and any $Q_n>\tfrac12$. The empirical
Lagrangian evaluates $g$ on the sample $\{X^{(0)}_j\}$ and $h$ on the sample
$\{X^{(1)}_j\}$, and these are a.s.\ disjoint, so the two terms decouple: choosing
$f\in\HK$ with $f(X^{(0)}_j)=0$ for all $j$ and $f(X^{(1)}_j)=-1$ for all $j$ --- possible
since $K_{\mathrm G}$ is strictly positive definite and the $2n$ points are distinct --- gives,
for small $t>0$, $\widehat{\mathcal R}_{0,n}(u_r+tf)=\widehat{\mathcal R}_{0,n}(u_r)$
exactly, while $h$ has slope $\tfrac12$ near $u_r$, so
$\widehat{\mathcal R}_{1,n}(u_r+tf)=\widehat{\mathcal R}_{1,n}(u_r)-t/2$ and
$\widehat L_n(u_r+tf,\mu)<\widehat L_n(u_r,\mu)$. Part~(b) must therefore be established by
showing that the exact empirical minimizers are \emph{close} to the constants, not that they
\emph{are} the constants. The improvement above costs RKHS norm --- $f$ must interpolate
$2n$ prescribed values --- and the argument below quantifies exactly this.
\end{remark}

The first ingredient rules out an empirical minimizer straddling both wells.

\begin{lemmaBox}{Splice exclusion under a derivative bound}{splice-exclusion}
Let $\mathcal X=[0,1]$, let $K$ satisfy $\sup_xK(x,x)\le\kappa^2$ and the derivative bound
\eqref{eq:deriv-bound} with constant $\varkappa$, let $u_1<u_2$, and set
$\rho\triangleq(u_2-u_1)/4$. For $\phi:\mathcal X\to\R$ write
$D_\phi(x)\triangleq\min\{\dist(\phi(x),\{u_1,u_2\}),\rho\}$. Assume
\eqref{eq:schedule-splice}. Then there is an event of probability one and an index $n_0$ such
that, for every $n\ge n_0$ and every $\phi\in\Bq{Q_n}$ attaining both a value
$\le u_1+\rho$ and a value $\ge u_2-\rho$ on $\mathcal X$,
\[
\bigl\|D_\phi\bigr\|_{L_1(\widehat{\mathcal M}^n_{\mathcal X})}
\;\ge\;\frac{\rho^{2}}{2\varkappa\,Q_n}.
\]
\end{lemmaBox}

\begin{proof}
Write $\bar u\triangleq(u_1+u_2)/2$ and note $u_1+\rho<\bar u<u_2-\rho$. By the intermediate
value theorem there is $c\in[0,1]$ with $\phi(c)=\bar u$. Put $r\triangleq\rho/(\varkappa Q_n)$
and $I\triangleq[c-r,c+r]\cap[0,1]$, so $|I|\ge r$ (for $r\le1$). By \eqref{eq:deriv-bound},
$|\phi(x)-\bar u|\le\varkappa Q_n\,r=\rho$ for $x\in I$; since $|\bar u-u_j|=2\rho$ for
$j=1,2$, this gives $\dist(\phi(x),\{u_1,u_2\})\ge2\rho-\rho=\rho$ and therefore
$D_\phi\equiv\rho$ on $I$.

It remains to bound $\widehat{\mathcal M}^n_{\mathcal X}(I)$ from below. Each of the $m+1=2$
empirical marginals is that of $n$ i.i.d.\ uniforms, so by the Dvoretzky--Kiefer--Wolfowitz
inequality at confidence $2n^{-2}$ and Borel--Cantelli, on an event of probability one the
Kolmogorov distance of each to the uniform law is at most $\sqrt{\log n/n}$ for all large $n$;
hence $\sup_{J\text{ interval}}|\widehat{\mathcal M}^n_{\mathcal X}(J)-|J||\le2\sqrt{\log n/n}$.
Since \eqref{eq:schedule-splice} implies $Q_n\sqrt{\log n/n}\to0$, we have
$r=\rho/(\varkappa Q_n)\ge 4\sqrt{\log n/n}$ for large $n$, so
$\widehat{\mathcal M}^n_{\mathcal X}(I)\ge r-2\sqrt{\log n/n}\ge r/2$. Therefore
$\|D_\phi\|_{L_1(\widehat{\mathcal M}^n_{\mathcal X})}\ge\rho\cdot r/2=\rho^2/(2\varkappa Q_n)$.
\end{proof}

For the instance at hand $u_2-u_1=1$, so $\rho=\tfrac14$ and, by \eqref{eq:deriv-bound},
$\rho^2/(2\varkappa Q_n)=b/(32Q_n)$: the kernel bandwidth enters the threshold linearly.

\begin{propBox}{The instance satisfies Assumption~\ref{ass:multi-argmin-recovery}(b)}{example-recovery}
For the instance of this appendix there is an event $\Omega_{\mathrm{fbl}}$ of probability one
and a sequence $\tau_n\downarrow0$, with $\tau_n=O\bigl(Q_n^2\sqrt{\log n/n}\bigr)
\asymp(\log n)^{-3/2}$ for the schedule fixed above, such that on $\Omega_{\mathrm{fbl}}$, for
all large $n$ and every $\widehat\mu_n\in\arg\max_{\mu\ge0}\widehat d_n(\mu)$:
\begin{enumerate}[label=\textnormal{(\roman*)},itemsep=2pt,topsep=2pt]
\item $|\widehat\mu_n-\mu^\star|\le8\eps_n$, so in particular $\widehat\mu_n\to\mu^\star$
along the full sequence;
\item every $\phi\in\widehat\Phi_n(\widehat\mu_n)$ satisfies
$\min_{r\in\{1,2\}}\|\phi-u_r\|_{L_1(\widehat{\mathcal M}^n_{\mathcal X})}\le\tau_n$;
\item for \emph{each} $r\in\{1,2\}$ there exists
$\phi\in\widehat\Phi_n(\widehat\mu_n)$ with
$\|\phi-u_r\|_{L_1(\widehat{\mathcal M}^n_{\mathcal X})}\le\tau_n$.
\end{enumerate}
Items (ii) and (iii) are exactly the two halves of form~(B) of
Remark~\ref{rem:part-b-quantifiers}, so Assumption~\ref{ass:multi-argmin-recovery}(b) holds at
$\mu^\star$, for every strictly increasing $\{n_k\}$ and for both admissible witnesses.
\end{propBox}

\begin{proof}
Let $\Omega_{\mathrm{fbl}}$ be the intersection of the two probability-one events
$\Omega_{\mathrm{uc}}$ and $\Omega_{\mathrm{gc}}$ constructed in the proof of
Proposition~\ref{prop:sharpness} --- whose construction, as noted there, uses only
Assumption~\ref{ass:loss-regularity}, the kernel bound and the schedule, and not the sharpness
hypothesis, which fails here --- with the probability-one event of
Lemma~\ref{lem:splice-exclusion}. On $\Omega_{\mathrm{uc}}$ one has
$\eps_n=O(Q_n\sqrt{\log n/n})$, and on $\Omega_{\mathrm{gc}}$ the empirical and population
$L_1$ norms of differences $\phi-\psi$, $\phi\in\Bq{Q_n}$, $\psi\in\Bq{C_\infty}$, agree up to
$8\kappa Q_n\sqrt{\log(2n)/n}$. Throughout, $\eps_n$ is as in
Proposition~\ref{prop:sharpness}, and by Corollary~\ref{cor:lagrangian-dev}
\begin{equation}\label{eq:dual-unif}
\sup_{\mu\in\Lambda}\bigl|\widehat d_n(\mu)-d_{\HK}(\mu)\bigr|\;\le\;\eps_n
\qquad\text{once }Q_n\ge1 ,
\end{equation}
the upper bound using that for every $\mu\in\Lambda$ the population minimum of $\psi_\mu$ is
attained at a constant of modulus at most $1$, hence in $\Bq{Q_n}$.

\emph{(i).} By Lemma~\ref{lem:bounded-multipliers}(iii), $\widehat\mu_n\in\Lambda$. From
$\widehat d_n(\widehat\mu_n)\ge\widehat d_n(\mu^\star)$ and \eqref{eq:dual-unif},
$d_{\HK}(\widehat\mu_n)\ge d_{\HK}(\mu^\star)-2\eps_n$, and \eqref{eq:dual-sharp} gives
$|\widehat\mu_n-\mu^\star|\le8\eps_n$.

\emph{(ii).} Fix $\phi\in\widehat\Phi_n(\widehat\mu_n)$. Since the population minimizer
$u_1$ lies in $\Bq{Q_n}$, minimality of $\phi$ and two applications of the uniform bound give
\[
L_{\HK}(\phi,\widehat\mu_n)\le L_{\HK}(u_1,\widehat\mu_n)+2\eps_n
= d_{\HK}(\widehat\mu_n)+2\eps_n+\bigl(\psi_{\widehat\mu_n}(u_1)-d_{\HK}(\widehat\mu_n)\bigr)
\le d_{\HK}(\widehat\mu_n)+2\eps_n+\tfrac{|\widehat\mu_n-\mu^\star|}{2},
\]
the last step because $\psi_{\mu^\star}(u_1)=d_{\HK}(\mu^\star)\ge d_{\HK}(\widehat\mu_n)$ and
$\|\psi_{\widehat\mu_n}-\psi_{\mu^\star}\|_\infty\le|\widehat\mu_n-\mu^\star|/2$. Writing
\[
\beta_n\;\triangleq\;2\eps_n+\tfrac{|\widehat\mu_n-\mu^\star|}{2}+|\widehat\mu_n-\mu^\star|
\;\le\;2\eps_n+12\eps_n\;=\;14\eps_n
\]
by (i), and integrating \eqref{eq:set-sharpness} at $\widehat\mu_n$,
\[
\int_{\mathcal X}\varrho\bigl(\dist(\phi(x),\{u_1,u_2\})\bigr)\,\mathrm d\mathcal M_{\mathcal X}(x)\;\le\;\beta_n .
\]
Since $D_\phi\le\rho=\tfrac14$ we have $\varrho(\dist(\phi,\{u_1,u_2\}))\ge\varrho(D_\phi)=D_\phi/4$,
so $\|D_\phi\|_{L_1(\mathcal M_{\mathcal X})}\le4\beta_n$; the class
$\{D_\phi:\phi\in\Bq{Q_n}\}$ has envelope $\rho$ and, by contraction along the $1$-Lipschitz
map $v\mapsto\min\{\dist(v,\{u_1,u_2\}),\rho\}$, Rademacher complexity at most
$\kappa Q_n/\sqrt n$, so on $\Omega_{\mathrm{fbl}}$
\begin{equation}\label{eq:D-emp}
\|D_\phi\|_{L_1(\widehat{\mathcal M}^n_{\mathcal X})}\;\le\;4\beta_n+8\kappa Q_n\sqrt{\tfrac{\log(2n)}{n}}\;\eqqcolon\;\zeta_n
\;=\;O\Bigl(Q_n\sqrt{\tfrac{\log n}{n}}\Bigr).
\end{equation}
By \eqref{eq:schedule-splice}, $Q_n\zeta_n\to0$; in particular
$\zeta_n<\rho^2/(2\varkappa Q_n)=b/(32Q_n)$ for all large $n$.
Lemma~\ref{lem:splice-exclusion} therefore forbids $\phi$ from attaining both a value
$\le u_1+\rho$ and a value $\ge u_2-\rho$: there is $r=r(\phi)\in\{1,2\}$ with $\phi$ never
entering the opposite band. Say $r=1$, i.e.\ $\phi<u_2-\rho$ on $\mathcal X$. Then on
$\{D_\phi<\rho\}$ the nearer of the two wells is necessarily $u_1$, so
$|\phi-u_1|=D_\phi$ there, while $\widehat{\mathcal M}^n_{\mathcal X}(D_\phi=\rho)\le\zeta_n/\rho$
by Markov and $|\phi-u_1|\le\kappa Q_n+\tfrac12$ pointwise. Hence
\[
\|\phi-u_1\|_{L_1(\widehat{\mathcal M}^n_{\mathcal X})}
\;\le\;\zeta_n+\bigl(\kappa Q_n+\tfrac12\bigr)\tfrac{\zeta_n}{\rho}
\;=\;\zeta_n\bigl(1+4\kappa Q_n+2\bigr)\;=\;O\bigl(Q_n\zeta_n\bigr)\;\eqqcolon\;\tau_n\;\longrightarrow\;0 ,
\]
again by \eqref{eq:schedule-splice}; symmetrically for $r=2$. Since
$\zeta_n=O(Q_n\sqrt{\log n/n})$ we get $\tau_n=O(Q_n^2\sqrt{\log n/n})$, which for the schedule
of this appendix is $\asymp(\log n)^{-3/2}$. This is (ii).

\emph{(iii).} By (i), $\widehat\mu_n>0$ for large $n$, so Lemma~\ref{lem:sign-comp-subgrad}
supplies $\nabla_n\in\partial\widehat d_n(\widehat\mu_n)$ with $\widehat\mu_n\nabla_n=0$,
i.e.\ $\nabla_n=0$. By Lemma~\ref{lem:LV}(i),
$0\in\conv\bigl\{\widehat{\mathcal R}_{1,n}(\phi):\phi\in\widehat\Phi_n(\widehat\mu_n)\bigr\}$.
Now $\widehat{\mathcal R}_{1,n}(u_r)=h(u_r)-\gamma=\mp\tfrac14$ \emph{exactly}, because
$\ell_1$ does not depend on $y$ and constants therefore make the empirical and population risks
coincide; and $\widehat{\mathcal R}_{1,n}$ is $L$-Lipschitz with respect to
$\|\cdot\|_{L_1(\widehat{\mathcal D}^n_{\mathcal X,1})}\le2\|\cdot\|_{L_1(\widehat{\mathcal M}^n_{\mathcal X})}$.
So by (ii) the set $\widehat{\mathcal R}_{1,n}(\widehat\Phi_n(\widehat\mu_n))$ is contained in
$[-\tfrac14-2L\tau_n,-\tfrac14+2L\tau_n]\cup[\tfrac14-2L\tau_n,\tfrac14+2L\tau_n]$. For $n$
large enough that $2L\tau_n<\tfrac14$ --- with $L=\tfrac12$, that is $\tau_n<\tfrac14$ --- the
point $0$ lies in neither interval, so its membership in the convex hull forces \emph{both}
intervals to be met, one being entirely negative and the other entirely positive. Each is met
by some $\phi\in\widehat\Phi_n(\widehat\mu_n)$, which by (ii) is within $\tau_n$ of the
corresponding constant. This is (iii).

Finally, the three statements hold for all large $n$ simultaneously and the event
$\Omega_{\mathrm{fbl}}$ does not depend on any subsequence, so
Assumption~\ref{ass:multi-argmin-recovery}(b) holds along every $\{n_k\}$ and for both
reference points $u_1,u_2$, as required.
\end{proof}

\subsection{Consequences, and the tightness of the closure gap}

All of Assumptions~\ref{ass:non-atomic}--\ref{ass:multi-argmin-recovery} hold, and
$\epsinf(\Sdual)=\epsinf(\mu^\star)=\tfrac14>0$. Theorem~\ref{thm:informal-section3}(ii)--(v)
therefore has a model, and the near-PACC branch is not vacuous. Because $\Sdual$ is a
singleton, part~(v) applies and the bound is $\epsinf(\mu^\star)$ rather than a supremum over
a set.

The bound is moreover \emph{attained} by the worst DERM selection, which is the sense in
which the theorem is sharp. Indeed the DERM selection rule requires $\widehat\phi_n$ to
attain, within $1/n$, the infimum defining $\Theta_{\nabla_n}(\widehat\mu_n)$ with
$\nabla_n=0$, i.e.\ to minimize $|\widehat{\mathcal R}_{1,n}(\phi)|$ over
$\widehat\Phi_n(\widehat\mu_n)$. By Proposition~\ref{prop:example-recovery}(ii)--(iii) that
infimum equals $\tfrac14+O(\tau_n)$ and is attained, to within $O(\tau_n)$, at
\emph{both} branches. The rule therefore does not discriminate between them, and a selection
returning the branch near $u_2$ yields
\[
\mathcal R_1(\widehat\phi_n)\;\longrightarrow\;+\tfrac14\;=\;\epsinf(\mu^\star),
\qquad
\mathcal R_0(\widehat\phi_n)\;\longrightarrow\;-\tfrac14\;<\;P^\star_{\HK} ,
\]
so the feasibility slack of Theorem~\ref{thm:informal-section3}(iii)--(v) is incurred in
full. Since those statements are quantified over arbitrary DERM selections, the constant
$\epsinf$ cannot be improved. (A selection returning the branch near $u_1$ is feasible with
$\mathcal R_1\to-\tfrac14$ and suboptimal with $\mathcal R_0\to0=P^\star_{\HK}+\tfrac18$;
consistently, $C_\mu\epsinf=2$ is a valid but very loose bound on that suboptimality here.)

\begin{remark}[What the splice-exclusion mechanism costs]\label{rem:splice-cost}
Lemma~\ref{lem:splice-exclusion} is where the kernel and the radius schedule interact, and it
is worth isolating the tradeoff. A function of RKHS norm $Q_n$ cannot change by $u_2-u_1$
faster than the derivative bound \eqref{eq:deriv-bound} permits, so a branch transition
occupies an interval of length at least $\asymp1/(\varkappa Q_n)$, on which the predictor is
far from both wells; sharpness meanwhile caps the total $L_1$ deviation from the argmin set
at $O(\eps_n)=O(Q_n\sqrt{\log n/n})$. Transitions are excluded precisely when the second
quantity is smaller than the first, i.e.\ when $Q_n^2\sqrt{\log n/n}\to0$, which is
\eqref{eq:schedule-splice}. Faster radius growth would let the empirical problem behave like
the envelope, restoring splices, driving the realized violation to zero and making $\epsinf$
a non-attained upper bound. The instance is thus a genuine model of the near-PACC branch for
the schedules this paper actually uses, and it shows that the branch is a real phenomenon
rather than an artifact of loose analysis; but the residual is a joint property of the
problem \emph{and} of the admissible radius growth, and we have not characterized the
transition in general. The same margin is what makes $\tau_n$ only polylogarithmic here, in
the manner of Remark~\ref{rem:effective-rate}: the identification of the empirical argmin set
is squeezed between two quantities that differ by a factor $Q_n^2\sqrt{\log n/n}$, and under
\eqref{eq:schedule-splice} that factor closes slowly.
\end{remark}

\section{A Worked Instance}\label{sec:worked-example}

We verify Assumptions~\ref{ass:non-atomic}--\ref{ass:multi-argmin-recovery} end to end on one
concrete problem in the exact regime $\epsinf=0$, complementing
Appendix~\ref{sec:near-pacc-example}. The example is of the fairness/coverage type that
motivates the paper, has a genuinely nonconvex objective, and evaluates objective and
constraint under \emph{different} distributions.

\paragraph{Instance.}
Let $\mathcal X=[0,1]^d$ and $\mathcal Y=\{-1,+1\}$, $m=1$.
\begin{itemize}[itemsep=2pt]
\item \emph{Objective.} $\mathcal D_0$ has a marginal $\mathcal D_{\mathcal X,0}$ with a
density bounded above and below on $\mathcal X$, and conditional
$\eta_0(x)\triangleq\Pr_{\mathcal D_0}(Y=+1\mid X=x)$. Take the bounded, nonconvex
\emph{Welsch} (correntropy-induced) loss
\[
\ell_0(u,y)\;\triangleq\;1-\exp\Bigl(-\tfrac{(u-y)^2}{2}\Bigr),
\]
which takes values in $[0,1)$ and is $e^{-1/2}$-Lipschitz in $u$, the Lipschitz constant being
$\max_t|t|e^{-t^2/2}=e^{-1/2}$. We use it in place of the sigmoid margin loss
$(1+e^{yu})^{-1}$: the conditional integrand of the latter is
$\eta_0(1+e^{u})^{-1}+(1-\eta_0)(1+e^{-u})^{-1}$, whose derivative equals
$(1-2\eta_0)\,e^{u}(1+e^{u})^{-2}$ and hence has the constant sign of $1-2\eta_0(x)$, so the
integrand is strictly monotone and has no minimizer whenever $\eta_0(x)\ne\tfrac12$, and
\eqref{eq:sharpness} cannot hold. A bounded loss with an interior minimizer is what the
sharpness condition requires, and the Welsch loss is the standard such choice.
\item \emph{Constraint.} $\mathcal D_1$ is a second, possibly unrelated distribution on
$\mathcal X\times\mathcal Y$ with a density bounded above and below, encoding a subpopulation
on which a minimum positive-prediction rate $\gamma\in(0,1)$ is required. Take
$\ell_1(u,y)=\gamma-h(u)$ with $h(u)\triangleq\mathrm{cl}_{[0,1]}\bigl(\tfrac{u+1}{2}\bigr)$
the clipped ramp of Appendix~\ref{sec:near-pacc-example}, so that $\ell_1$ is bounded by
$\max\{\gamma,1-\gamma\}\le1$ and $\tfrac12$-Lipschitz. The constraint
$\mathcal R_1(\phi)\le0$ reads $\mathbb E_{\mathcal D_1}[h(\phi(X))]\ge\gamma$: a floor, where
the second instance of Appendix~\ref{sec:composite-predictors} imposed a cap.
\item \emph{Kernel.} $K=1+K_{\mathrm G}$ with $K_{\mathrm G}$ Gaussian on $\mathcal X$ of
bandwidth $b$, so $\kappa=\sqrt2$, Assumption~\ref{ass:dense-rkhs} holds with any
$p\in[1,\infty)$ \cite[Thm.~4.63]{SteinwartChristmann2008}, and --- as in
Appendix~\ref{sec:near-pacc-example} --- the constants belong to $\HK$ with
$\|c\|_{\HK}=|c|$. (Using $K_{\mathrm G}$ alone would exclude the constants and force an
approximation argument in the Slater step below.)
\item \emph{Schedule.} $Q_n=n^{1/4}/\log(e+n)$, satisfying \eqref{eq:schedule} and the
$o(n^{1/4})$ budget of Proposition~\ref{prop:sharpness} simultaneously
(Remark~\ref{rem:schedule}).
\end{itemize}

\paragraph{Verification.}
\emph{Assumption~\ref{ass:non-atomic}} holds since both marginals have densities, hence so
does the mixture $\mathcal M_{\mathcal X}$, which is therefore nonatomic.

\emph{Assumption~\ref{ass:loss-regularity}} holds with $L=e^{-1/2}$ and $M=1$: the two losses
are $e^{-1/2}$- and $\tfrac12$-Lipschitz respectively, and bounded by $1$ and by
$\max\{\gamma,1-\gamma\}\le1$. (One may of course take $L=1$ for readability; the constants
below use $M$ only.)

\emph{Assumption~\ref{ass:dense-rkhs}} holds as noted.

\emph{Assumption~\ref{ass:slater}}: take $\phi^{\mathrm{sl}}\equiv\gamma$, so that
$h(\gamma)=\tfrac{1+\gamma}{2}$ and
$\mathcal R_1(\phi^{\mathrm{sl}})=\gamma-\tfrac{1+\gamma}{2}=-\tfrac{1-\gamma}{2}$. Since the
constants lie in $\HK$ exactly, the Slater margin is $\xi=\tfrac{1-\gamma}{2}>0$ with
$Q_{\mathrm{sl}}=\gamma$ and $C_\mu=4M/\xi=8/(1-\gamma)$, with no approximation step and no
halving of the margin.

\emph{Assumption~\ref{ass:multi-argmin-recovery}} via Proposition~\ref{prop:sharpness}.
Writing $w_0,w_1$ for the mixture densities, the conditional integrand at multiplier $\mu\ge0$
is
\[
\bar\ell_\mu(u,x)
=w_0(x)\Bigl[1-\eta_0(x)e^{-(u-1)^2/2}-\bigl(1-\eta_0(x)\bigr)e^{-(u+1)^2/2}\Bigr]
+\mu\,w_1(x)\bigl[\gamma-h(u)\bigr].
\]
Three facts are needed.

\emph{(1) Every pointwise minimizer lies in $[-1,1]$, for every $x$ and every $\mu\ge0$.}
Outside $[-1,1]$ the ramp is flat, $h'=0$, so the constraint term is constant there and the
sign of $\partial_u\bar\ell_\mu$ is that of $\partial_u\bar\ell_0$, where
\[
\partial_u\bar\ell_0(u,x)=\eta_0(x)(u-1)e^{-(u-1)^2/2}+\bigl(1-\eta_0(x)\bigr)(u+1)e^{-(u+1)^2/2} .
\]
For $u\ge1$ both summands are $\ge0$ and the second is $>0$, so $\bar\ell_\mu(\cdot,x)$ is
strictly increasing on $[1,\infty)$; for $u\le-1$ both are $\le0$ with the first $<0$, so it is
strictly decreasing on $(-\infty,-1]$. Hence the infimum over $\R$ equals the minimum over the
compact $[-1,1]$ and is attained there, with no dependence on $\gamma$, on $C_\mu$, or on
$\sup w_1/\inf w_0$. (This replaces the cruder argument comparing values at $u=1$ with the
limit at $+\infty$, which gave only a $\mu$-dependent radius.)

\emph{(2) First-order condition.} On $(-1,1)$ one has $h'=\tfrac12$, so
$\partial_u\bar\ell_\mu(u,x)=w_0(x)\partial_u\bar\ell_0(u,x)-\tfrac{\mu w_1(x)}{2}$ and
interior minimizers solve
\[
w_0(x)\,\partial_u\bar\ell_0(u,x)=\tfrac{\mu w_1(x)}{2} .
\]
Note that the right-hand side is nonnegative, so the constraint term shifts the minimizer
\emph{upward} relative to the unconstrained one, consistently with $h$ increasing and the
constraint rewarding large predictions.

\emph{(3) Sharpness.} Assume, in addition to the density bounds, that $\eta_0$ is continuous
and that there are $c_0>0$, $\varsigma>0$, $\Delta_0>0$ such that for every $x$ and every
$\mu$ in a neighbourhood $U$ of $\mu_{\HK}$ the function $\bar\ell_\mu(\cdot,x)$ has a
\emph{unique} global minimizer $\phi^\sharp_\mu(x)$ and
\begin{align*}
&\partial_u^2\bar\ell_\mu(u,x)\ \ge\ c_0
&&\text{for all }u\text{ with }|u-\phi^\sharp_\mu(x)|<\varsigma,\\
&\bar\ell_\mu(u,x)-\bar\ell_\mu\bigl(\phi^\sharp_\mu(x),x\bigr)\ \ge\ \Delta_0
&&\text{for all }u\text{ with }|u-\phi^\sharp_\mu(x)|\ge\varsigma,
\end{align*}
both uniformly in $x$ and $\mu\in U$. Then \eqref{eq:sharpness} holds with $p=2$,
$\gamma_0=c_0/2$ and this $\Delta_0$: on the window, Taylor with the second-order bound
\emph{on the whole window} gives
$\bar\ell_\mu(u,x)-\bar\ell_\mu(\phi^\sharp_\mu(x),x)\ge\tfrac{c_0}{2}|u-\phi^\sharp_\mu(x)|^2
\ge\varrho(|u-\phi^\sharp_\mu(x)|)$, and beyond it the second hypothesis gives
$\ge\Delta_0\ge\varrho(\cdot)$. We emphasise that the curvature bound is required on the
window and not merely at the minimizer: a second-order condition at a single point carries no
quadratic growth without a remainder estimate, and the constant would then not be $c_0/2$.

We state these as hypotheses rather than derive them, because they are exactly the
identifiability content of the instance: they hold, for example, when $\eta_0$ is bounded away
from $\tfrac12$ and $\mu\sup w_1/\inf w_0$ is small enough that the constraint term does not
create a second competing well, and they are checkable numerically for any given
$(\eta_0,w_0,w_1)$ --- the minimization is over the fixed compact $[-1,1]$ by (1), which makes
the check a finite one. They fail precisely when the two wells of $\bar\ell_0(\cdot,x)$ become
level on a set of positive measure --- the tie situation of
Appendix~\ref{sec:near-pacc-example}.

Under (1)--(3), $x\mapsto\phi^\sharp_\mu(x)$ is continuous (by the implicit function theorem
at interior roots, and by uniqueness plus compactness in general), hence lies in
$C(\mathcal X)$. Whether it lies in $\HK$ with norm at most $C_\infty$, uniformly over
$\mu\in U$, is the one genuinely restrictive requirement of
Proposition~\ref{prop:sharpness}; it holds, e.g., when $\eta_0$ and the densities are
analytic, since then $\phi^\sharp_\mu$ is analytic and Gaussian-RKHS membership follows from
its Fourier decay. Proposition~\ref{prop:sharpness} then delivers
Assumption~\ref{ass:multi-argmin-recovery}, in the form~(A) of
Remark~\ref{rem:part-b-quantifiers} and hence a fortiori in the form~(B) that the theorem uses.

\paragraph{What the instance shows, and what it does not.}
Theorem~\ref{thm:informal-section3}(i) applies with the explicit constants $L=e^{-1/2}$,
$M=1$, $\kappa=\sqrt2$, $\xi=\tfrac{1-\gamma}{2}$, $C_\mu=8/(1-\gamma)$,
$Q_{\mathrm{sl}}=\gamma$. Assumption~\ref{ass:source} holds with $\beta=1$ and
$A=\max\{1,C_\infty^2\}$, since the Lagrangian is attained at $\phi^\sharp_\mu\in\Bq{C_\infty}$;
Corollary~\ref{cor:explicit-rate} then gives $n^{\mathrm{val}}=\widetilde O(\eps^{-4})$ with all
constants explicit. (The corollary is stated for $Q_n=n^{1/4}$; the schedule used here differs
by a logarithmic factor, which changes only the radius index
$\min\{n:Q_n/\sqrt n\le\theta/2cL\}$ by a polylogarithmic factor and is absorbed by
$\widetilde O$.) The recovery bound \eqref{eq:recovery-rate} itself is subject to
Remark~\ref{rem:effective-rate} and vanishes only at rate $\Theta(\log^{-2}n)$ under this
schedule. Because the pointwise minimizer is a.e.\ unique, part~(vi) applies: the guarantee is
exact PACC.

This instance therefore illustrates part~(vi) only. By
Lemma~\ref{lem:eps-star-properties}(v) it necessarily has $\epsinf=0$, which is the general
phenomenon recorded in Remark~\ref{rem:scope}: any instance certified by
Proposition~\ref{prop:sharpness} lies in the exact regime. For the complementary regime
$\epsinf>0$ --- obtained here by replacing $\ell_0$ with a symmetric double-well loss so that
$\bar\ell_\mu(\cdot,x)$ has two global minimizers on a positive-measure region --- see
Appendix~\ref{sec:near-pacc-example}, where the recovery condition is verified directly rather
than through Proposition~\ref{prop:sharpness}.

\section{On Solving the Ball-Constrained DERM Problem}\label{sec:solving}

We record, without convergence analysis, the computational shape of
\eqref{eq:dual formulation DERM}. By the representer argument in the footnote to
Corollary~\ref{cor:multi-qual-feas}, the inner minimization at fixed $\mu$ may be carried out
over $\Bq{Q_n}\cap V_n$, where $V_n$ is the span of the $N\triangleq(m+1)n$ kernel sections at
the observed inputs. Index the observed inputs by pairs $(i,j)$, write
$\phi_\alpha\triangleq\sum_{i,j}\alpha_{ij}K(\cdot,X^{(i)}_j)$, let $G\in\R^{N\times N}$ be the
Gram matrix, so that $\|\phi_\alpha\|^2_{\HK}=\alpha^\top G\alpha$ and the vector of
evaluations at the observed inputs is $G\alpha$, and let $[\,\cdot\,]_i$ extract the block of
coordinates belonging to dataset $i$. The inner problem is then
\[
\min_{\alpha\,:\,\alpha^\top G\alpha\le Q_n^2}\ \
\frac1n\sum_{j=1}^n\ell_0\bigl([G\alpha]_{0,j},Y^{(0)}_j\bigr)
+\sum_{i=1}^m\mu_i\,\frac1n\sum_{j=1}^n\ell_i\bigl([G\alpha]_{i,j},Y^{(i)}_j\bigr),
\]
a (generally nonconvex) minimization over an ellipsoid, for which projected gradient or
Frank--Wolfe steps are available at cost $O(N^2)=O((m+1)^2n^2)$ per iteration.

The outer problem is the maximization of a finite concave function of $\mu$ over the compact
set $\Lambda$ of Lemma~\ref{lem:bounded-multipliers}, for which a supergradient is furnished
\emph{exactly} by Lemma~\ref{lem:LV}(i): any $\widehat{\mathcal R}_n(\phi)$ with
$\phi\in\widehat\Phi_n(\mu)$ is a supergradient, and by
Lemma~\ref{lem:sign-comp-subgrad} a sign- and complementarity-compatible one exists at the
optimum. Standard projected supergradient ascent on $\Lambda$ therefore applies.

The essential caveat is that the inner problem must be solved to \emph{global} optimality.
This is not a limitation confined to the outer iterate: the exact argmin set
$\widehat\Phi_n(\mu)$ is what Lemma~\ref{lem:LV}(i) represents, what
Lemma~\ref{lem:sign-comp-subgrad} and Lemma~\ref{lem:bounded-multipliers}(iii) presuppose,
and what the index $\Theta_\nabla$ and the DERM selection rule are defined against.
Appendix~\ref{sec:near-pacc-example} shows that the requirement has teeth rather than being a
formality: there the exact argmin set has two well-separated components, both of which the
selection rule is entitled to return, and an inner solver that reliably found only one of them
would silently change the guarantee. Relaxing exactness to a computable tolerance, and
propagating that tolerance through Lemma~\ref{lem:multi-qual-conv} and
Lemma~\ref{lem:full-seq-feas}, is the natural next step and is left to future work.

\bibliographystyle{unsrt}
\bibliography{refrences}   
\appendix

\end{document}